\documentclass[11pt]{article}
\usepackage{latexsym}

\usepackage{macros}
\usepackage{natbib}

\usepackage{graphicx,epsfig}
\usepackage{overpic}
\usepackage{tikz}

\usepackage{xifthen}
\usepackage{authblk}
\usepackage{algorithm,algpseudocode}
\usepackage{accents}
\usepackage{hyperref}
\usepackage{cleveref}
\usepackage{booktabs}
\crefformat{footnote}{#2\footnotemark[#1]#3}
\newcommand{\Cov}{{\rm Cov}}
\newcommand{\Var}{{\rm Var}}

\newcommand{\wtilde}{\widetilde}

\renewcommand{\H}{\mathcal{H}}

\newcommand{\wbar}{\overline}
\newcommand{\F}{\mathcal{F}}

\newcommand{\eps}{\epsilon}

\newcommand{\grad}{\nabla}

\newcommand{\err}{\mathcal{E}}

\newcommand{\KRR}{{\rm KRR}}

\newcommand{\FNG}{F^{\rm NG}}
\newcommand{\FML}{F^{\rm ML}}
\newcommand{\FKRR}{F^{\rm KRR}}
\newcommand{\FOLS}{F^{\rm OLS}}
\newcommand{\proj}{\Pi}
\newcommand{\subgrad}{\mathsf{g}}
\newcommand{\opt}{^*}

\newcommand{\circup}[1]{\accentset{\circ}{#1}}

\newcommand{\X}{\mathcal{X}}

\begin{document}

\title{On the Self-Penalization Phenomenon in Feature Selection}


\author{Michael I. Jordan$^\dag$, Keli Liu$^*$ and Feng Ruan$^\dag$}

\affil{University of California, Berkeley$^\dag$ \\
The Voleon Group$^*$}

\maketitle

\begin{abstract}
We describe an implicit sparsity-inducing mechanism based on minimization over a 
family of kernels:
\begin{equation*}
	\min_{\beta, f}~\what{\E}[L(Y, f(\beta^{1/q} \odot X)] + \lambda_n \norm{f}_{\H_q}^2~~\text{subject to}~~\beta \ge 0,
\end{equation*}
where $L$ is the loss, $\odot$ is coordinate-wise multiplication and $\H_q$ is the 
reproducing kernel Hilbert space based on the kernel $k_q(x, x') = h(\norm{x-x'}_q^q)$,
where $\norm{\cdot}_q$ is the $\ell_q$ norm.  Using gradient descent to optimize 
this objective with respect to $\beta$ leads to exactly sparse stationary points with 
high probability. The sparsity is achieved without using any of the well-known 
explicit sparsification techniques such as penalization (e.g., $\ell_1$), 
early stopping or post-processing (e.g., clipping). 

As an application, we use this sparsity-inducing mechanism to build algorithms consistent for
feature selection. 

\end{abstract}

\vspace{.1in}
Invariant Set.   
\vspace{.1in}

\section{Introduction}
In supervised learning, we are given labeled data $(Y,X)$ and want to learn a 
function $f(X)$ for predicting $Y$. Whether motivated by domain knowledge or 
practical constraints, we may want the function $f$ to depend on only a few 
of the features in $X$. $\ell_1$ regularization is the most well known mechanism 
for inducing sparsity in the components of $X$. When $f$ is assumed 
linear---$f(x) = x^\top \beta$---the Lasso objective uses squared loss 
and an $\ell_1$ penalty on $\beta$ to learn $f$:
\begin{equation}
\label{eqn:lasso_obj}
\min_{\beta \in \mathbb{R}^p}~\what{\E} \Big[(Y - X^\top \beta)^2\Big] + \gamma \norm{\beta}_1,
\end{equation}
where $\what{\E}$ is the empirical expectation over the data. The $\ell_1$ penalty 
in the Lasso objective forces components of $\beta$ to be exactly zero, operationalizing
the notion of automatic selection of those features of $X$ that are most relevant 
for predicting $Y$.

This paper introduces a new sparsity-inducing mechanism which does not use any 
explicit $\ell_1$ penalty but still obtains sparsity in $f$.  The mechanism relies 
on the fact that the stationary points of certain kernel-based empirical risk 
functions are \emph{naturally sparse} with respect to the features---a fact that
has not been previously noted. The usual kernel approach to learning $f$ solves 
the following optimization problem:
\begin{equation}
\label{eqn:basic_kernel_obj}
	\min_{f\in \H} \what{\E} \left[L(Y, f(X))\right] + \lambda_n \norm{f}_{\H}^2,
\end{equation}
where $L$ is a loss function and $\H$ is a reproducing kernel Hilbert space (RKHS). 
The solution $f$ is generally not sparse in $X$. To induce sparsity, we consider 
a parametrized function, $f(\beta^{1/q} \odot X)$, where $\odot$ denotes coordinate-wise 
multiplication, and we simultaneously minimize the usual kernel objective over $f$ 
and $\beta$:
\begin{equation}
\label{eqn:general-template-empirical}
\begin{split}
	\min_{\beta}~&F_n(\beta)~~\text{subject to}~~\beta \ge 0,~ \norm{\beta}_\infty\le M\\
	~~\text{where}~~~&F_n(\beta) = \min_{f} 
		\what{\E} \left[L(Y, f(\beta^{1/q} \odot X))\right] + \lambda_n \norm{f}_{\H_q}^2.
\end{split}
\end{equation}
We call $F_n(\beta)$ the kernel feature selection objective; in general, it is nonconvex 
in $\beta$. The ridge parameter $\lambda_n$ and the box constraint parameter $M$ are needed 
so that the solutions $(f, \beta)$ are bounded.  Similar to the case of the Lasso, a feature 
$X_j$ is active in the learned prediction function if and only if $\beta_j \neq 0$; however, 
the kernel formulation lacks an explicit $\ell_1$ penalty for $\beta$ that aims to separate
active from inactive features. Conventional wisdom suggests that we should \emph{not} expect 
to obtain an \emph{exactly sparse} solution $\beta$ by simply minimizing $F_n(\beta)$, at 
least not without some form of post-processing of the solution (e.g., clipping). Surprisingly, 
we observe many cases where the stationary points of the kernel feature selection objective 
are exactly sparse. We call this phenomenon \emph{self-penalization}. The purpose of the 
paper is to (i) clarify the nature of self-penalization, and (ii) discuss its statistical 
consequences for (nonparametric) feature selection.
 
The phenomenon is present even in the restricted setting of kernel functions of the form 
$k(x, x') = h(\norm{x-x'}_q^q)$, for $q\in \{1, 2\}$, and for simplicity our theoretical 
analysis will restrict attention to this setting.  We further limit ourselves to two choices 
of the loss function $L$:
\paragraph{Metric learning}
Consider the classification setting where $X \in \R^p$ and $Y \in \{\pm 1\}$.  Let $L$ 
be the margin-based loss $L(y, \what{y}) = -y\hat{y}$. Then the objective has the following 
explicit form:
\begin{equation}
\label{eqn:obj-metric-learning-emp}
F_n(\beta) = - \frac{1}{4\lambda_n}\what{\E}\left[YY' h(\norm{X-X'}_{q, \beta}^q)\right],
\end{equation}
where $(X, Y)$ and $(X', Y')$ denote independent copies from the empirical distribution 
and $\norm{z}_{q, \beta}$ denotes the weighted $\ell_q$ norm: $\norm{z}_{q, \beta}^q 
= \sum_i \beta_i |z_i|^q$.  We supply a short derivation of this result in 
Appendix~\ref{section:basic-derivation}; see \cite{LiuRu20} for a fuller presentation.
\paragraph{Kernel ridge regression} 
Consider the regression setting where $X \in \R^p$ and $Y \in \R$. Let $L$ be the squared 
error loss $L(y, \what{y}) = \half (y - \hat{y})^2$. Define the objective function as
follows:
\begin{equation}
\label{eqn:obj-kernel-ridge-regression-emp}
F_n(\beta) = \min_{f}~ \half \what{\E}[(Y- f(\beta^{1/q} \odot X))^2] + \lambda_n \norm{f}_{\H_q}^2.
\end{equation}

Both the metric learning and kernel ridge regression (KRR) objectives turn out to exhibit 
naturally sparse stationary points. To begin to study this phenomenon, let us consider a 
concrete example involving the metric learning objective.
Suppose $X$ is uniformly distributed on the square $[-1,1]^2$ and the response $Y\in \{\pm 1\}$ 
is balanced; i.e., $\E[Y] = 0$.  Let us assume that $\E[Y|X] = X_1^3$.  
Figure \ref{fig:l1-vs-l2} provides a graphical illustration of the gradient 
field of the empirical metric learning objective, equation \eqref{eqn:obj-metric-learning-emp}, 
when $n=20$. The figure clearly shows that despite the multiplicity of stationary points, 
the gradient flow, which moves along the negative gradient, converges to points that only 
select $X_1$. 
\newcommand{\Unif}{{\rm Unif}}
\begin{figure}[!tph]
\begin{centering}
\includegraphics[width=.55\linewidth]{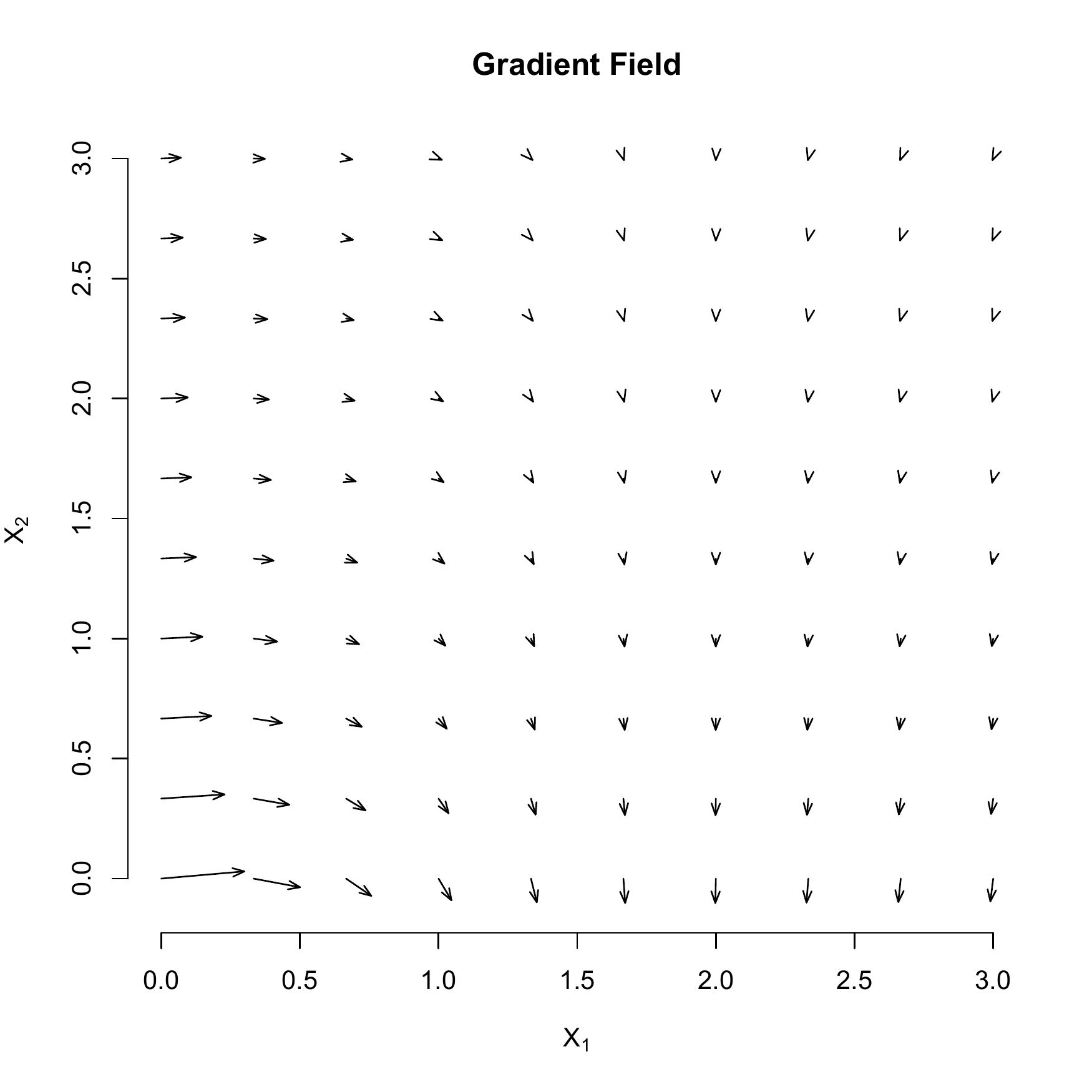}
\par\end{centering}
\caption{The gradient field for the empirical metric learning 
objective~\eqref{eqn:obj-metric-learning-emp}.  The sample size is $n = 20$ and
the dimension is $p = 2$. The data as distributed according to $X_1, X_2 \sim \Unif[-1,1]$ 
and $\E[Y|X] = X_1^3$. The arrows represent the \emph{negative} gradients of the 
empirical metric learning objective. 
}
\label{fig:l1-vs-l2}
\end{figure}

\setcounter{example}{0}
\renewcommand\theexample{\arabic{example}} 

The rest of the paper is organized as follows. Section~\ref{sec:overview} sets up the problem and introduces assumptions. 
Section~\ref{sec:main-results} sketches the process by which self-penalizing objectives exclude noise variables. 
Section~\ref{sec:notion-of-self-penalization} supplies the formal definition of a self-penalizing objective and shows that the metric learning and kernel ridge regression objectives are self-penalizing while
classical objectives, such as least squares, are not. Section~\ref{sec:gradient-dynamics} characterizes the behavior of gradient descent when applied to the metric learning and kernel ridge regression objectives: under appropriate conditions, gradient descent reaches sparse stationary points in finite time.
Section~\ref{sec:statistical-implications} discusses how to use self-penalizing objectives to build consistent feature selection algorithms. Section~\ref{sec:experiments} presents experiments to supplement our theory.

\subsection{Notation}
\label{sec:notation}
\begin{table}[!ht]
\begin{center}
\begin{tabular}{r c p{10cm} }
\toprule
$\R$      & $\triangleq$ &the set of real numbers\\
$\R_+$ & $\triangleq$ & the set of nonnegative reals \\
$[p]$ & $\triangleq$ &the set $\{1, 2, \ldots, p\}$ \\
$v_S$ & $\triangleq$ & restriction of a vector $v$ to the index set $S$\\
$\supp(\mu)$ & $\triangleq$ & the support for a measure $\mu$ \\
$\norm{v}_q$ & $\triangleq$ & the $\ell_q$ norm of the vector $v$: $\norm{v}_q = (\sum_i |v_i|^q)^{1/q}$\\  
$\norm{v}_{q, \beta}$ & $\triangleq$ & the weighted $\ell_q$ norm of $v$: $\norm{v}_{q, \beta} = (\sum_i \beta_i |v_i|^q)^{1/q}$\\  
$\mathcal{C}^\infty(\R_+)$  & $\triangleq$ & the set of functions $f$ infinitely differentiable 
				 on $(0, \infty)$ whose derivatives are right continuous at $0$ \\
$\dot{\beta}(t)$ & $\triangleq$ & the derivative of $t \mapsto \beta(t)$ at $t$ \\
$\norm{f}_\infty$ & $\triangleq$ & $\sup_{x \in \R^p} |f(x)|$ the supremum of $|f|$\\
\bottomrule
\end{tabular}
\label{table:notation}
\end{center}
\end{table}

The notation $\Q$ is reserved for the population distribution of the data $(X, Y)$, and 
$\Q_n$ denotes the empirical distribution. 
The notation $\H \equiv \H_q$ stands for the $\ell_q$-type RKHS associated with the 
kernel function $k(x,x') = h(\norm{x-x'}_q^q)$.

\section{Problem Setup}
\label{sec:overview}


\newcommand{\W}{\mathcal{P}}
Consider the following general kernel feature selection problem:
\begin{equation}
\label{eqn:general-template-population}
\begin{split}
	\min_{\beta}~~ &F_\lambda(\beta; \Q)~~\text{subject to}~~\beta \ge 0, ~\norm{\beta}_\infty \le M. \\
	&F_\lambda(\beta; \Q) = \min_{f\in \H} \E_Q[L(Y, f(\beta^{1/q} \odot X)] + \frac{\lambda}{2} \norm{f}_{\H}^2.
\end{split}
\end{equation}
When $\Q = \Q_n$ is the empirical distribution over data $(X_i, Y_i)$, 
$F_\lambda(\beta; \Q_n)$ is the empirical kernel feature selection objective; 
setting $\Q$ as the population distribution of $(X, Y)$ yields the population 
objective. $\beta$ enters the objective in the form $\beta^{1/q}$, which 
takes on one of two values depending on the choice of $q\in \{1, 2\}$ in
the kernel $k(x, x') = h(\norm{x-x'}_q^q)$.

This paper considers using projected gradient descent to solve the nonconvex 
problem~\eqref{eqn:general-template-population} and focuses on describing the 
sparsity pattern of the solutions. Let $\mathcal{B} = \{\beta: \beta \ge 0, 
\norm{\beta}_\infty \le M\}$ be the feasible set (which is convex). 
We consider the projected gradient descent algorithm: 
\begin{equation}
	\beta(k+1) = \proj_{\mathcal{B}}  (\beta(k) - \alpha \grad F_\lambda(\beta(k); \Q))~~\text{for $k=0, 1, 2, \ldots $},
\end{equation}
where $\proj_\mathcal{B}$ denotes the $\ell_2$ projection onto the set $\mathcal{B}$; 
i.e., $\proj_\mathcal{B}(z) = \argmin_{\beta\in \mathcal{B}} \ltwo{\beta-z}$.  
To simplify the analysis, we also consider the projected gradient inclusion, 
which is, intuitively speaking, the limit of the projected gradient iterates 
after taking the stepsize $\alpha \mapsto 0^+$, and is defined as follows:
\begin{equation}
\label{eqn:projected-gradient-inclusion}
\dot{\beta}(t)  \in  -\grad F_\lambda(\beta(t); \Q) - \normal_{\mathcal{B}}(\beta(t)),~~\text{and}~~\beta(0) \in \mathcal{B},
\end{equation}
where $\normal_{\mathcal{B}}$ denotes the normal cone with respect to the convex 
set $\mathcal{B}$. 
Appendix~\ref{sec:preliminaries-on-differential-inclusions} presents a formal treatment 
of the projected gradient inclusion~\eqref{eqn:projected-gradient-inclusion}, showing 
that the solution $t \mapsto \beta(t)$ exists whenever $\beta \mapsto \grad 
F_\lambda(\beta; \P)$ is Lipschitz in the feasible set $\mathcal{B}$. 

Two special cases of the general kernel feature selection objective are of 
interest to us: the metric learning objective in~\eqref{eqn:obj-metric-learning-emp} 
and the kernel ridge regression objective in~\eqref{eqn:obj-kernel-ridge-regression-emp}). 
We reproduce them here using our formal notation:
\begin{equation}
\label{eqn:obj-metric-learning}
\FML(\beta; \Q) = -\E_Q\left[YY' h(\norm{X-X'}_{q, \beta}^q)\right]
\end{equation}
\begin{equation}
\label{eqn:obj-kernel-ridge-regression}
\FKRR_\lambda(\beta; \Q) = \min_{f \in \H}~ \half \E_Q[(Y- f(\beta^{1/q} \odot X))^2] + \frac{\lambda}{2} \norm{f}_\H^2.
\end{equation}
Since both the metric learning and kernel ridge regression objectives are smooth in $\beta$, the projected gradient descent inclusion is well defined for any probability measure $\Q$. 


%

Our results are based on two technical assumptions. The first assumption is
concerned with the regularity of the RKHS $\H$. Recall that $\H$ is the 
RKHS associated with a kernel of the form $h(\norm{x-x'}_q^q)$ where $q\in \{1, 2\}$.
According to Bernstein's theorem~\citep[see, e.g.,][Proposition 1]{JordanLiRu21}, 
for either $q=1$ or $q=2$, there exists an RKHS $\H$ corresponding to a function 
$h \in \mathcal{C}^\infty(\R_+)$ if and only if the function $h$ satisfies the
following:
\begin{equation}
\label{eqn:f-kernels}
	h(z) =  \int_0^\infty e^{-tz} \mu(dt),
\end{equation}
for some measure $\mu$ on $\R_+$ with $\mu((0, \infty)) > 0$.
Let $m_\mu = \inf\{x: x\in \supp(\mu)\}$ and $M_\mu = \sup\{x: x\in \supp(\mu)\}$.

\begin{assumption}
\label{assumption:mu-assumtion-kernels}
Assume that $\supp(\mu)$ is bounded and bounded away from zero: $0 < m_\mu < M_\mu < \infty$.
\end{assumption}

\begin{remark}
Assumption~\ref{assumption:mu-assumtion-kernels} can be easily verified using the following equivalent characterization: 
(i) the function $h \in \mathcal{C}^\infty(\R_+)$ is completely monotone, 
i.e., $(-1)^k h^{(k)}(x) > 0$ for all $k \in \N$, and 
(ii) the function $h^\prime(z)$ has exponential decay: 
$C_1 e^{-c_1z} \le h^\prime(z) \le C_2 e^{-c_2 z}$ holds for all $z \in \R_+$,
where $C_1, C_2, c_1, c_2$ are independent constants.  
Assumption~\ref{assumption:mu-assumtion-kernels} holds for a range of
functions including the important case of $h(z) = \exp(-z)$.
\end{remark}

Our second assumption concerns the regularity of the population distribution $\Q$. 
\begin{assumption}
\label{assumption:distribution-P}
Assume the distribution $\Q$ has the following properties: 
\begin{enumerate}
\item[(i)] The response $Y$ is centered: $\E_Q[Y] = 0$. 
\item[(ii)] The support of $(X, Y)$ is bounded under $\Q$:  
$\norm{X}_\infty \le M_X < \infty$, $|Y| \le M_Y < \infty$. 
\item[(iii)] The coordinates of $X$ have non-vanishing variance: $\Var_Q(X_l) > 0$ for all $l \in [p]$.
\end{enumerate}
\end{assumption}

\begin{remark}
Assumption (i) is simply for convenience; we can dispense with it by adding an intercept 
to the objective---the new objective with the intercept would inherit all of the properties 
stated in the paper. Assumption (ii) is also imposed to simplify the exposition and can be 
replaced by a weaker assumption; e.g., a sub-gaussian assumption. In our proofs, it is 
mainly used to guarantee that population and empirical quantities (objective values and 
gradients) are uniformly close to each other.
\end{remark}

\section{Main Results}
\label{sec:main-results}
In this section we present our main result: that projected gradient descent 
(or gradient inclusion) applied to the metric learning or kernel ridge regression 
objectives induces \emph{exactly sparse} solutions in finite samples with 
high probability. This result holds without any use of explicit regularization 
techniques such as $\ell_1$ regularization or early stopping. 

A necessary condition for a sparsity-inducing mechanism to arise is that there
is some form of constraint on the dependence structure of the features. 
Indeed, consider the extreme case in which all of the covariates are all 
equal: $X_1 = X_2 = \ldots = X_p$.  If the initialization has equal coordinates, 
$\beta_1(0) = \ldots = \beta_p(0)$, then the population gradient flow, by 
symmetry, must have equal coordinates along the entire trajectory. Consequently,
the gradient flow can't converge to a sparse solution (with the exception of 
$\beta= 0$). With this counterexample in mind, we divide the total set of 
indices, $\{1, \ldots, p\}$, into complementary subsets, $S$ and $S^c$, 
such that a variable is in $S$ if it is either indispensable for predicting $Y$ 
(in the sense that we can't achieve perfect prediction without this variable) 
\emph{or} it is highly correlated with a variable that is indispensable for 
predicting $Y$. Our results will show that projected gradient descent converges 
to a point whose support is contained in $S$. We begin by presenting a formal
definition of $S$.
\begin{definition}[Signal Set $S$ and Noise Set $S^c$]
\label{definition:signal-set}
The \emph{signal set} $S$ with respect to a measure $\Q$ is defined as the minimal 
subset $S \subseteq [p]$ such that the following holds: 
\begin{itemize}
\item $\E_Q[Y| X] = \E_Q[Y | X_S]$; i.e., the signal $X_S$ has full predictive 
power for $Y$ given $X$. 
\item $X_S \perp X_{S^c}$ under $\Q$; i.e., the noise variables are independent 
of the signal variables. 
\end{itemize}
Such a minimal set $S$ is unique and always exists~\citep[][Proposition 16]{JordanLiRu21}.
\end{definition}
 
The set $S$ represents an upper bound on the set of variables that are chosen via 
kernel feature selection. This upper bound $S$ may be loose in terms of describing 
the variables actually chosen, but it describes the limit of what we are able to 
prove. As a concrete illustration, consider a simple case in which there are two 
dependent but non-equal covariates, $X= (X_1, X_2)$, for which $\E_Q[Y|X] = \E_Q[Y|X_1]$. 
In this case, $S = \{1, 2\}$ (by definition) so our results cannot preclude the 
possibility that gradient descent converges to a point where $\beta_2 > 0$. 
(It is difficult, however, to construct a numerical counterexample in which $\beta_2 > 0$;
see further dicussion in Section \ref{sec:experiments}). As a consequence, the 
characterization of the self-regularization phenomenon that we give in Theorems 
\ref{thm:FML-sparse-one-round} and \ref{thm:FKRR-sparse-one-round} is most accurate 
when signal and noise variables are not highly correlated. When significant correlation 
exists, kernel feature selection continues to deliver sparse solutions but our 
theoretical understanding in these situations is limited.

\setcounter{theorem}{0}
\renewcommand\thetheorem{\arabic{theorem}A}

\begin{theorem}
\label{thm:FML-sparse-one-round}
Given Assumptions~\ref{assumption:mu-assumtion-kernels} and~\ref{assumption:distribution-P}, 
assume that the set of signals is not empty: $S \neq \emptyset$.  Consider the trajectory 
$t \mapsto \wtilde{\beta}(t)$ of the gradient flow with respect to the empirical metric 
learning objective $\FML(\cdot; \Q_n)$. Choose the initialization to be of full support: 
$\supp(\wtilde{\beta}(0)) = [p]$.  Then, with probability at least $1-e^{-cn}$, we have:
\begin{equation}
\label{eqn:FML-sparse-one-round}
	\emptyset \neq \supp(\wtilde{\beta}(t)) \subseteq S~~\text{holds for all $t \ge \tau$},
\end{equation}
where the constants $c, \tau > 0$ are independent of the sample size $n$. 
\end{theorem}
\begin{remark}
The independence assumption $X_S \perp X_{S^c}$ can be relaxed. Instead we can assume
$\varrho(X_S, X_{S^c}) < \vartheta \cdot |F(\wtilde{\beta}(0); \Q)|$
for a constant $\vartheta > 0$ that depends only on $M_X, M_Y$ and $\mu$ (and not on $n$). 
Here $\varrho(X_S, X_{S^c})$ is the maximal correlation between $X_S$ and $X_{S^c}$~\citep{Renyi59}. 
See Appendix~\ref{sec:extension-to-weak-dependence} for details. 

\end{remark}

Theorem~\ref{thm:FKRR-sparse-one-round} shows that the gradient flow 
of the kernel ridge regression objective also produces sparse solutions.

\setcounter{theorem}{0}
\renewcommand\thetheorem{\arabic{theorem}B}

\begin{theorem}
\label{thm:FKRR-sparse-one-round}
Given Assumptions~\ref{assumption:mu-assumtion-kernels} and~\ref{assumption:distribution-P},
assume that the set of main effect signals is not empty 
$\circup{S} = \{l \in S: \Var_Q(\E_Q[Y|X_l]) \neq 0\} \neq \emptyset$. 
Consider the trajectory $t \mapsto \wtilde{\beta}(t)$ of the gradient flow with 
respect to the empirical kernel ridge regression objective $\FKRR_\lambda(\cdot; \Q_n)$. 
Choose the initialization to be $\beta(0) = 0$.
Set the parameter $q = 1$. There exists a constant $\lambda_0 > 0$ such that whenever 
$\lambda_n \le \lambda_0$, we have, with probability at least $1-e^{-cn \lambda_n^8}$,
\begin{equation}
\label{eqn:FKRR-sparse-one-round}
	\emptyset \neq \supp(\wtilde{\beta}(t)) \subseteq S~~\text{holds for all $t \ge \tau$},
\end{equation}
where the constants $\lambda_0, c, \tau$ are independent of the sample size $n$. 
\end{theorem}


\begin{remark}
Theorem~\ref{thm:FKRR-sparse-one-round} requires slightly more stringent conditions than 
Theorem \ref{thm:FML-sparse-one-round}: (i) the existence of the ``main effect'' signals;
i.e., $\E_Q[Y|X_l] \neq \E_Q[Y]$ where $l\in S$ and (ii) $q=1$. These assumptions are 
very likely artifacts of the proof. In fact, our numerical experiments demonstrate that as long as $S \neq \emptyset$, with
initialization $\beta(0)$ of full support, the gradient flow of the kernel ridge 
regression objective produces sparse solutions for both $q=1$ and $q=2$. 
In particular, if $Y$ and $X$ are related through pure interactions, the solutions 
are sparse (Section \ref{sec:experiments}).  Additionally, our experiments show that 
the assumption $X_S \perp X_{S^c}$ can be weakened to allow weak dependence. 
\end{remark}

Theorem \ref{thm:FML-sparse-one-round} and \ref{thm:FKRR-sparse-one-round} describe 
cases where running gradient descent on the metric learning or KRR objective produces 
solutions that include at least one signal variable and exclude all noise variables. 
This intriguing property enhances the further development of algorithms that are 
selection consistent; i.e., outputs $\hat S$ of an algorithm satisfy $\hat S \subseteq S$ 
and $\E[Y|X] = \E[Y|X_{\hat S}]$.  Essentialy, this is achieved by running gradient 
descent multiple times on a carefully constructed sequence of ML or KRR objectives; 
see Section~\ref{sec:statistical-implications}.




\newcommand{\ri}{{\rm ri}}

\subsection{Proof outline}
\label{sec:proof-outline}
In this section we provide a sketch of the main ideas underlying the proofs of 
Theorem~\ref{thm:FML-sparse-one-round} and Theorem~\ref{thm:FKRR-sparse-one-round},
with the details deferred to Sections~\ref{sec:notion-of-self-penalization} and
\ref{sec:gradient-dynamics}. 

We use $F(\beta; \Q)$ and $F(\beta; \Q_n)$ to denote the population and empirical 
objectives, and let $t \mapsto \beta(t)$ and $t \mapsto \wtilde{\beta}(t)$ denote 
the trajectories of the gradient flow with respect to the population and empirical 
objectives. 

Our first goal is to show that the population dynamics satisfy $\emptyset \subsetneq \supp(\beta(t)) \subseteq S$ for all $t \ge \tau$
for some $\tau < \infty$. 
The key is to construct a subset $B \subseteq \mathcal{B}$ such that 
(i) $0 \not\in B$, and (ii) the population objective and its associated gradient flow 
satisfy the following properties: 
\begin{itemize}
\item The gradient flow enters $B$ in finite time; i.e., $\beta(t_0) \in B$ for 
some $t_0 < \infty$.
\item The set $B$ is invariant with respect to the gradient flow---any trajectory 
$t \mapsto \beta(t)$ entering the set $B$ will remain in the set $B$~\citep{HaleKo12}.
\item The objective $F(\beta; \Q)$ is \emph{self-penalizing} on the set $B$---the 
gradients with respect to the noise variables when restricting to $\beta \in B$ 
are uniformly lower bounded away from zero:
	\begin{equation}
	\label{eqn:self-penalizing-population}
		\min_{j \not\in S}\inf_{\beta: \beta \in B}~\partial_{\beta_j} 
		  F(\beta; \Q) \ge c > 0.
	\end{equation}
\end{itemize} 
The first two properties imply that after time $t_0$, we have $\beta(t) \in B$ for
$t > t_0$. The self-penalizing property in Eq.~\eqref{eqn:self-penalizing-population} 
tells us that, since $\beta(t) \in B$, the coordinates corresponding to noise variables,
$\beta_j(t)$ for $j\in S^c$, are either zero or are strictly decreasing at a positive 
rate. The net result of all three properties is that for some $\tau < \infty$, the dynamics 
must satisfy $\emptyset \subsetneq \supp(\beta(t)) \subseteq S$ for all $t \ge \tau$ 
(NB: $\beta(t) \neq 0$ since $0 \not\in B$).

\begin{figure}[h]
\begin{centering}
\begin{overpic}[width=.6\linewidth]{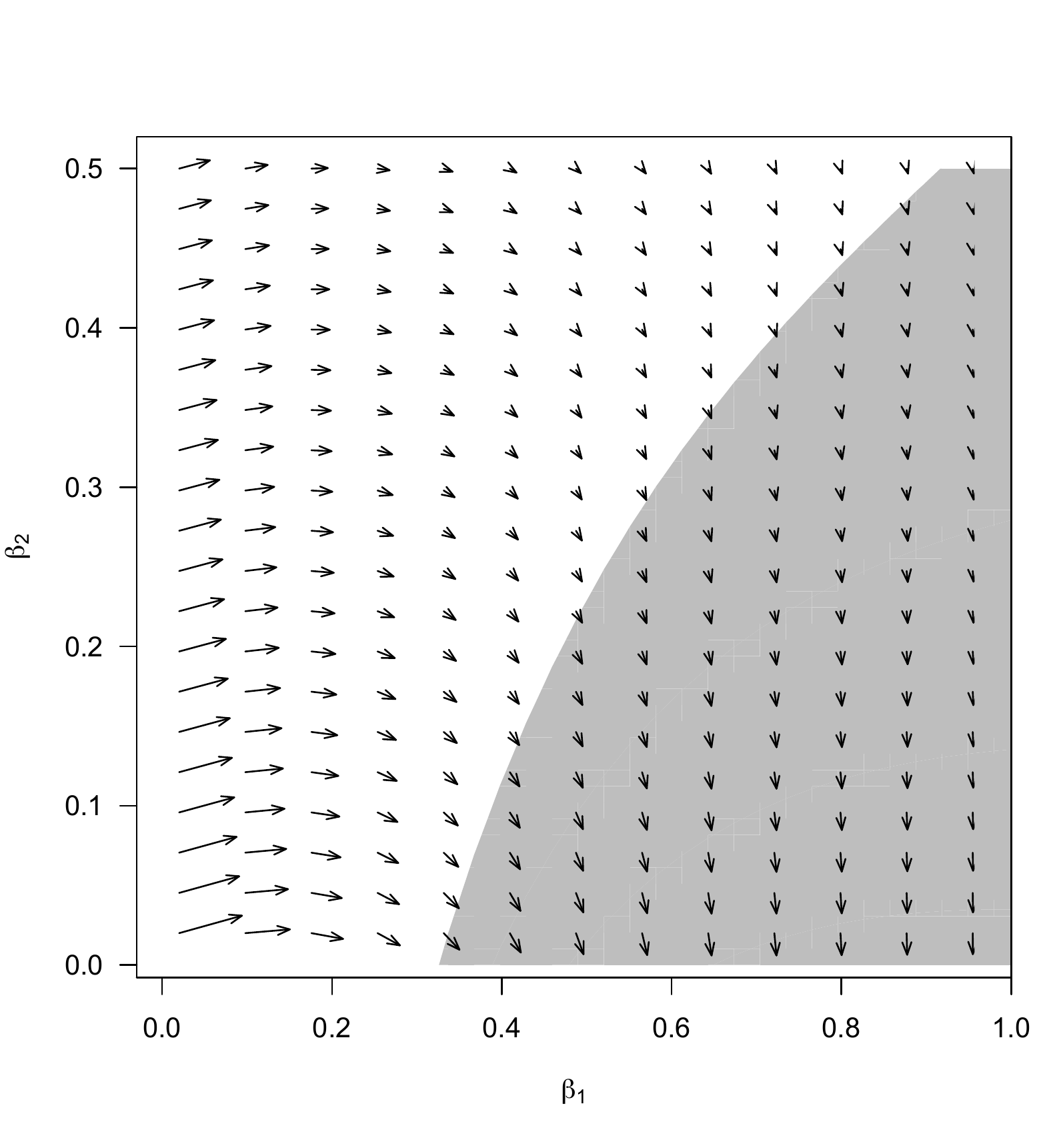}
\put(40,92){\small Gradient Flow}
\put(93,50){set $B$ (or $B_n$)}
\end{overpic}
\par\end{centering}
\caption{\small A pictorial illustration of the proof idea. 
In the example, $\beta_1$ (x-axis) corresponds to the signal variable, 
and $\beta_2$ (y-axis) corresponds to the noise variable. We use the 
shaded area to represent the specific set $B$ (or $B_n$) used in the 
proof. The proof logic is as follows: The gradient flow (i) enters 
$B$ in finite time, (ii) stays in $B$ once it enters (i.e., $B$ is an 
invariant set), and (iii) drives the noise variable $\beta_2$ to exactly 
zero in finite time as the objective is self-penalizing on the set 
$B$ (i.e., the gradient of the objective with respect to the noise 
variable $\beta_2$ is strictly positive, and in fact, bounded away 
from zero on the set $B$).
}
\label{fig:main-effect}
\end{figure}

Next, we want to prove a similar result for the empirical dynamics $t \mapsto \wtilde{\beta}(t)$. To do so, we will construct a set 
$B_n \subseteq B$, and a high probability event $\Omega_n$ such that 
on the event $\Omega_n$, the empirical objective and its gradient flow satisfy

\begin{itemize}
\item The empirical gradient dynamics enters $B_n$ in finite time: 
	$\wtilde{\beta}(t_0) \in B_n$ where $t_0 < \infty$.
\item The set $B_n$ is invariant with respect to the empirical gradient flow---any trajectory $t \mapsto \wtilde{\beta}(t)$ 
	entering the set $B_n$ would remain in the set $B_n$ afterwards. 
\item The objective $F(\beta; \Q_n)$ is \emph{self-penalizing} on the set $B_n$: 
	\begin{equation}
	\label{eqn:self-penalizing-empirical}
		\min_{j \not\in S}\inf_{\beta: \beta \in B_n}~\partial_{\beta_j} F(\beta; \Q_n) \ge c/2 > 0.
	\end{equation}
\end{itemize}
By the same reasoning as above, it follows that for some $\tau < \infty$, the empirical dynamics satisfy 
$\emptyset \subsetneq \supp(\wtilde{\beta}(t)) \subseteq S$ on the event $\Omega_n$ for all $t \ge \tau$. 

%
%

Formally, $\Omega_n$ is defined as the event where the objective values and the 
gradients of the population and empirical objectives are uniformly close (up to 
some $\eps_n$ where $\eps_n \to 0$) on the feasible set $\mathcal{B}$.  Defining 
$\Omega_n$ in this way immediately yields that the self-penalizing property of 
the empirical objective in Eq.~\eqref{eqn:self-penalizing-empirical} is directly 
inherited from that of the population objective.

The construction of the sets $B$ and $B_n$ is trickier. A key idea is to utilize 
the fact that the sublevel sets of the minimizing objective, no matter whether it
is a population or an empirical objective, are always an invariant set with 
respect to the corresponding gradient dynamics. The specific details of $B$ 
and $B_n$ will depend on the particular structure of the metric learning and 
KRR objectives; see Section~\ref{sec:gradient-dynamics}.



\newcommand{\FLIN}{F^{\rm LIN}}
\def\theequation{\arabic{equation}}

\section{The Notion of Self-Penalization}
\label{sec:notion-of-self-penalization}

This section provides a general treatment of self-penalization.  
Section~\ref{sec:principle-of-self-penalization} introduces the formal definition 
of self-penalization.  Section~\ref{sec:classical-objectives-not-self-penalizing} 
shows that the \emph{linear-type objectives}---which include classical least square 
objectives and nonnegative garrottes---are not self-penalizing objectives, even
if their $\ell_1$ penalized versions are. This explains the necessity of using 
explicit $\ell_1$ regularization to attain sparse solutions for linear-type objectives.  
Section~\ref{sec:metric-learning-self-penalizing} and 
Section~\ref{sec:kernel-ridge-regression-self-penalizing} show that, 
in contrast to the linear objectives, the two kernel-based feature selection 
objectives, namely, the metric learning and the kernel ridge regression objectives, 
are self-penalizing objectives. This has significant statistical implications: 
the \emph{nonlinear-type objectives}---which model the prediction function using 
RKHS with $\ell_q$-type kernels in formulation~\eqref{eqn:general-template-population}---are 
able to exclude noise variables even without the aid of $\ell_1$ regularization. 
The implications of these results to feature selection will be deferred to 
Section~\ref{sec:statistical-implications}.

\subsection{Formal definition}
\label{sec:principle-of-self-penalization}
This section formalizes the concept of \emph{self-penalization}. Recall the definition
of the signal set $S$, which is defined as the minimal subset of $\{1, 2, \ldots, p\}$ such that (i) $\E_Q[Y|X] = \E_Q[Y|X_S]$ 
and (ii) $X_S \perp X_{S^c}$ under $\Q$. The set of noise variables is the complement set $S^c$. 

We consider the minimization problem $\min_{\beta \in \mathcal{B}} F(\beta; \Q)$ where $\mathcal{B}$ is a convex 
constraint set. 
The variable $\beta \in \R^p$ is a $p$-vector whose coordinate is associated with the feature $X$ in a way that $\beta_j > 0$
if and only if $X_j$ is active.  Assume $\beta \mapsto F(\beta; \Q)$ is locally Lipschitz.


\begin{definition}
\label{definition:self-penalization}
A minimizing objective $F(\beta; \Q)$ is \emph{self-penalizing} with respect to a relatively open subset $B \subseteq \mathcal{B}$
if (i) the objective $F(\beta; \Q)$ is differentiable on a set $B^0$ dense in $B$ with the property that the 
gradient with respect to noise variables are bounded away from zero:
\begin{equation}
\label{eqn:self-penalization-strong}
	\min_{j \not\in S}\inf_{\beta: \beta \in B^0} \sign(\beta_j) \cdot \partial_{\beta_j}   F(\beta; \Q) > 0, 
\end{equation}
$(ii)$ $B$ intersects non-trivially with the set of stationary points with no false positives. 
Formally, 
$B \cap \{\beta \in \mathcal{B}^*: \supp(\beta) \subseteq S\} 
\neq \emptyset$ where $\mathcal{B}^* = \{\beta \in \mathcal{B}: 0 \in \partial 
F(\beta; \Q) + \normal_{\mathcal{B}}(\beta)\}$ denotes the set of Clarke 
stationary points. Here $\partial F(\beta; \Q)$ denotes the Clarke 
subdifferential of $F$ and $\normal_\mathcal{B}(\beta)$ denotes the 
normal cone of $\mathcal{B}$ at $\beta$.\footnote{We restrict attention to
the case where the objective $\beta \mapsto F(\beta; \Q)$ is either continuously 
differentiable or convex (which can be nonsmooth).  The Clarke subdifferential 
$\partial F(\beta; \Q)$ is the gradient when $F$ is continuously differentiable, 
and the set of subgradients when $F$ is convex~\citep{RockafellarWe09}.} 
\end{definition}

As discussed in Section \ref{sec:proof-outline}, Definition~\ref{definition:self-penalization} 
is particularly useful when the set $B$ is invariant with respect to the gradient 
flow associated with the objective; i.e., once the gradient flow enters $B$, 
it never leaves. In that case, with the assumption that the gradient descent 
enters $B$, the gradient flow $\beta(t)$ must satisfy $\supp(\beta(t)) \subseteq S$ 
for all (sufficiently) large $t$. This follows from the fact that the coefficients 
for the noise variables are guaranteed to approach zero at a strictly positive rate 
once the gradient flow enters $B$; see Eq.~\eqref{eqn:self-penalization-strong}.

\subsection{Linear-type objectives fail to be self-penalizing}
\label{sec:classical-objectives-not-self-penalizing}
This section shows that many linear-type objectives, i.e., those which model the 
prediction function as linear in the features, are not self-penalizing. 
Without loss of generality we assume throughout this section that both the 
features and the response are centered, i.e., $\E_Q[X]= 0$ and $\E_Q[Y] = 0$. 

\begin{example}[Ordinary Least Squares/Lasso~\citep{Tibshirani96}] 
The Lasso (without intercept) is
\begin{equation*}
\begin{split}
	\min_{\beta: \beta \in \R^p}~~ \FOLS_\gamma(\beta; \Q)
		~~\text{where}~~& \FOLS_\gamma(\beta; \Q) = \half \E_Q\Big[(Y - \sum_{i=1}^p \beta_i X_i)^2\Big] + \gamma \norm{\beta}_1.
\end{split}
\end{equation*}
The coefficient $\gamma \ge 0$ is known as the \emph{$\ell_1$ regularization parameter}. 
\end{example}

\begin{example}[Nonnegative Garrotte~\citep{Breiman95}]
\label{example:NG}
The nonnegative garrotte is a two stage procedure. In the first stage, it computes the 
ordinary least squares estimator, $w \in \R^p$, of $Y$ on $X$ under $\Q$. In the second stage, it performs the following minimization 
\begin{equation*}
\begin{split}
	\min_{\beta: \beta \in \R_+^p} ~~ \FNG_\gamma(\beta; \Q) 
	~~\text{where}~~ \FNG_\gamma(\beta; \Q) = \half \E_Q\Big[(Y- \sum_{i=1}^p \beta_i w_i X_i)^2\Big] + \gamma \sum_{i=1}^p \beta_i.
\end{split}
\end{equation*}
The procedure treats $X_j$ active if and only if $\beta_j > 0$.
\end{example}

For later comparison to \emph{nonlinear} objectives, we give an additional example
in which the objective has the same form as the kernel feature objective in 
Eq.~\eqref{eqn:general-template-population}, but it is based on a \emph{linear} kernel.

\vspace{0.3cm}
\begin{example}[Linear Kernel Objective]
\label{example:linear-kernel}
We consider the following procedure: 
\begin{equation*}
	\min_{\beta: \beta \in \R_+^p} \FLIN_{\lambda}(\beta; \Q)~~\text{where}~~
		\FLIN_\lambda(\beta; \Q) = \min_w \half \E_Q \Big[(Y - \sum_{i=1}^p \beta_i w_i X_i)^2\Big] + \frac{\lambda}{2} \ltwo{w}^2.
\end{equation*}
An equivalent way to write $\FLIN$ is (cf. Eq.~\eqref{eqn:general-template-population}): 
\begin{equation*}
\FLIN_\lambda(\beta; \Q) = \min_{f \in \H} \half \E_Q \Big[(Y - f(\beta \odot X))^2\Big] + \frac{\lambda}{2} \norm{f}_{\H}^2,
\end{equation*}
where $\H$ is the RKHS associated with the linear kernel, $k(x, x') = \langle x, x' \rangle$. 
We also define its $\ell_1$-penalized version: 
$\FLIN_{\lambda, \gamma}(\beta; \Q) = \FLIN_{\lambda}(\beta; \Q) + \gamma \sum_i \beta_i$.
\end{example}

The following propositions show that these linear-type objectives are not self-penalizing, 
no matter how the set $B$ is chosen in the Definition~\ref{definition:self-penalization}. 
The proofs for Propositions \ref{proposition:F-NG-self-penalizing} and 
\ref{proposition:F-LIN-self-penalizing} are deferrred to Appendix~\ref{sec:proof-of-proposiiton-NG-self-penalizing} and 
Appendix~\ref{sec:proof-of-proposition-F-LIN-self-penalizing} respectively.

\setcounter{proposition}{0}
\renewcommand\theproposition{1\Alph{proposition}}

\begin{proposition}
\label{proposition:F-OLS-self-penalizing}
The objective $\FOLS_\gamma(\beta; \Q)$ is self-penalizing if and only if $\gamma > 0$.
\end{proposition}
\begin{proof}
No matter how $B$ is defined, it must intersect non-trivially with $\mathcal{B}^*$ (Definition \ref{definition:self-penalization}); hence, to prove that $\FOLS$ cannot be self-penalizing it suffices to show that the property in equation \eqref{eqn:self-penalization-strong} cannot hold for points near $\mathcal{B}^*$. In this case, the set of stationary points $\mathcal{B}^*$ consists only of the set of global minima since the Lasso objective is convex. 

Denote $\mathcal{B}^0 = \{\beta \in \R^p: \beta_i \neq 0~\text{for $i \in [p]$}\}$, the set where
the objective is differentiable. Fix a noise variable $j \not\in S$. We show that for any $\beta^* \in \mathcal{B}^*$,
 \begin{equation}
 \label{eqn:limit-classical-least-square}
 	\lim_{\beta \to \beta^*; \beta \in \mathcal{B}^0} \sign(\beta_j) \cdot \partial_{\beta_j} \FOLS_\gamma (\beta; \Q) = \gamma. 
 \end{equation}
Denote $r_{\beta}(X; Y) = Y - \sum_{i=1}^p \beta_i X_i$ as the residual.
At any \nolinebreak $\beta \in \mathcal{B}^0$,
\begin{equation}
\label{eqn:gradient-formula-classical-least-square}
\begin{split}
	\sign(\beta_j) \cdot \partial_{\beta_j} \FOLS_\gamma (\beta; \Q) &= 
		\gamma - \sign(\beta_j) \cdot \E_Q \big[X_j r_{\beta}(X; Y)\big].
\end{split}
\end{equation}
Note $\lim_{\beta \to \beta^*} \E_Q \big[X_j r_{\beta}(X; Y)\big] = \E_Q [X_j r_{\beta^*}(X; Y)] = \Cov_Q(X_j, r_{\beta^*}(X; Y)) = 0$
where the second identity uses (i) $\supp(\beta^*) \subseteq S$, (ii) $\E[Y|X] = \E[Y|X_S]$ and (iii) $X_S \perp X_{S^c}$. 
\end{proof}

\begin{proposition}
\label{proposition:F-NG-self-penalizing}
The objective $\FNG_\gamma(\beta; \Q)$ is self-penalizing if and only if $\gamma > 0$.
\end{proposition} 

\begin{proposition}
\label{proposition:F-LIN-self-penalizing}
The objective $\FLIN_{\lambda, \gamma}(\beta; \Q)$ is self-penalizing if and only if $\gamma > 0$.
\end{proposition}

\setcounter{proposition}{1}
\renewcommand\theproposition{\arabic{proposition}}

\subsection{Metric learning is self-penalizing}
\label{sec:metric-learning-self-penalizing}
This section shows that the metric learning objective $\FML(\beta; \Q)$ is a 
self-penalizing objective.  Remarkably, there is no explicit $\ell_1$ regularization 
in the definition of $\FML(\beta; \Q)$. Recall the metric learning objective
(see Eq.~\eqref{eqn:obj-metric-learning}): 
\begin{equation}
\FML(\beta; \Q) = -\E[YY' h(\normsmall{X-X'}_{q, \beta}^q)].
\end{equation}
Importantly, the metric learning objective $\FML(\beta; \Q)$ is a nonparametric 
dependence measure~\citep[see][Proposition 3]{LiuRu20}.

\begin{proposition}
\label{proposition:choices-of-f-non-parametric-dependence}
Assume $\E_Q[Y] = 0$. Then the metric learning objective $\FML(\beta; \Q)$ is 
a nonparametric dependence measure with the following properties: 
\begin{itemize}
\item Negativeness: $F(\beta; \Q) \le 0$ for all $\beta \ge 0$
\item Strict negativeness: $F(\beta; \Q) < 0$ if and only if $Y \not\perp X_{\supp(\beta)}$ under $\Q$. 
\end{itemize}
\end{proposition}

We show that the metric learning objective is self-penalizing as long as the signal set $S$ 
is not empty. The core technical argument is the following lower bound on the gradient with respect to the noise variables. The proof is given in Section~\ref{sec:proof-theorem-metric-learning}.

\newcommand{\DeltaML}{\Delta^{\rm ML}}

\renewcommand\theassumption{\arabic{assumption}}

\renewcommand\thetheorem{\arabic{theorem}}

\begin{theorem}
\label{theorem:metric-learning}
Given Assumptions~\ref{assumption:mu-assumtion-kernels} and~\ref{assumption:distribution-P}, 
the following holds for all $\beta\ge 0$ and $j \not\in S$: 
\begin{equation}
\label{eqn:self-penalization-gradient-bound}
\begin{split}
\partial_{\beta_j}  \FML(\beta; \Q) & \ge \underline{c}(\beta) \cdot |\FML(\beta; \Q)|.
\end{split}
\end{equation}
Above  $\underline{c}(\beta) \defeq  m_\mu \cdot \E_Q \Big[e^{-M_\mu \normsmall{X_{S^c}- X_{S^c}'}_{q, \beta}^q} \cdot |X_j - X_j'|^q \Big]$ 
	is continuous and strictly positive. 
In particular, $\inf_{\beta \in \mathcal{B}} \underline{c}(\beta) = \underline{c} > 0$ where $\mathcal{B} = \{\beta: \beta \ge 0, \norm{\beta}_\infty \le M\}$ 
is the feasible set. 
\end{theorem}

\begin{remark}
Theorem~\ref{theorem:metric-learning} shows that the gradient with respect to the 
noise variables is lower bounded by the objective value itself (up to a constant), 
and the objective value is a quantitative measure of the dependence between $Y$ 
and $X_{\supp(\beta)}$. This is a general characteristic of kernel-based objectives---the 
strength of self-penalization (measured by the lower bound of the gradient with 
respect to noise) is dependent on the predictive power of the selected variables. 
\end{remark}

Theorem~\ref{theorem:metric-learning} immediately implies that the metric learning objective is self-penalizing. 

\begin{corollary}
\label{cor:metric-learning-self-penalizing}
Given Assumptions~\ref{assumption:mu-assumtion-kernels} and~\ref{assumption:distribution-P},
assume $S \neq \emptyset$. Let $\mathcal{B}$ denote the feasible set. Define for any 
$c > 0$ the following set (which is relatively open with respect to $\mathcal{B}$):
\begin{equation*}
	B_c = \left\{\beta \in \mathcal{B}: |\FML(\beta; \Q)| > c\right\}.
\end{equation*}
Then (i) the set $B_c \neq \emptyset$ for small enough $c > 0$, and 
(ii) $\FML(\beta; \Q)$ is self-penalizing with respect
to the set $B_c$ whenever $B_c \neq \emptyset$ and $c > 0$. 
\end{corollary}

\begin{proof}
As $\FML(\beta; \Q)$ is a nonparametric dependence measure, $|\FML(\beta; \Q)| > 0$ as long as $S \subseteq \supp(\beta)$
(Proposition~\ref{proposition:choices-of-f-non-parametric-dependence}). 
This proves that $\mathcal{X}_c \neq \emptyset$ for small enough $c > 0$. The rest of the corollary follows from the gradient lower bound (Eq.~\eqref{eqn:self-penalization-gradient-bound}) 
in Theorem~\ref{theorem:metric-learning}. 
\end{proof}

\subsection{Kernel ridge regression is self-penalizing}
\label{sec:kernel-ridge-regression-self-penalizing}
This section shows that the kernel ridge regression objective 
$\FKRR_\lambda(\beta; \Q)$ is a \emph{self-penalizing} objective. 
We define the kernel ridge regression objective as follows~\citep{JordanLiRu21}: 
\begin{equation}
\FKRR_\lambda(\beta; \Q) = \min_{f \in \H}~ \half \E[(Y- f(\beta^{1/q} \odot X))^2] + \frac{\lambda}{2} \norm{f}_\H^2. 
\end{equation}
Intuitively, the kernel ridge regression objective $\FKRR_\lambda(\beta; \Q)$ 
is a proxy for the unexplained variance of $Y$ given $X_{\supp(\beta)}$ 
(if the RKHS $\H$ is universal; see \cite{MicchelliXuZh06} for the definition). 
We have the following result, whose proof is given in 
Section~\ref{sec:proof-proposition-kernel-ridge-consistent-reason}.
\begin{proposition}
\label{proposition:kernel-ridge-consistent-reason}
Given Assumption~\ref{assumption:mu-assumtion-kernels}, the following limit 
holds for any $\beta \ge 0$ for the RKHS $\H$ whose associated kernel 
is of the form $h(\norm{x-x'}_q^q)$, where $q \in \{1, 2\}$:
\begin{equation}
\label{eqn:krr-consistent-reason}
	\lim_{\lambda \to 0^+} \FKRR_\lambda (\beta; \Q) = \half \cdot \E_Q [\Var (Y | X_{\supp(\beta)})].
\end{equation}
\end{proposition}

We show that the KRR objective is self-penalizing as long as the signal set is 
not empty. The core technical argument is the gradient lower bound which we provide in 
the following theorem whose proof is given in 
Section~\ref{sec:proof-theorem-kernel-ridge-regression}. 

\newcommand{\DeltaKRR}{\Delta^{\rm KRR}}

\begin{theorem}
\label{theorem:kernel-ridge-regression}
Given Assumptions~\ref{assumption:mu-assumtion-kernels} and~\ref{assumption:distribution-P},
the following holds for all $\beta \ge 0$ and \nolinebreak $j \not\in S$: 
\begin{equation}
\label{eqn:self-penalization-kernel-ridge-regression}
	\partial_{\beta_j}  \FKRR_\lambda(\beta; \Q) \ge c_j \cdot \lambda \left(\FKRR_\lambda(0; \Q) - \FKRR_\lambda(\beta; \Q)\right)_+^2
		-C \cdot  \frac{1+\lambda^2}{\lambda^2} \cdot \norm{\beta_{S^c}}_1,
\end{equation}
where $c_j = m_\mu \cdot \E_Q[|X_j - X_j'|^q]/(|h(0)|M_Y^2) > 0$, 
and $C > 0$ depends only on $\mu, M_X, M_Y$.
\end{theorem}

\begin{remark}
According to Eq.~\eqref{eqn:krr-consistent-reason}, the term 
$\FKRR_\lambda(0; \Q) - \FKRR_\lambda(\beta; \Q)$ on the right-hand side 
of Eq.~\eqref{eqn:self-penalization-kernel-ridge-regression} is approximately  
(when $\lambda$ is small) equal to the difference between the total variance 
of $Y$ and the unexplained variance of $Y$ given $X_{\supp(\beta)}$,
which is simply the explained variance of $Y$ given $X_{\supp(\beta)}$. 
Theorem~\ref{theorem:kernel-ridge-regression} shows that the gradient 
with respect to the noise variables is lower bounded by the difference 
between the explanatory power of the selected variables, 
$\FKRR_\lambda(0; \Q) - \FKRR_\lambda(\beta; \Q)$, and the size of 
the noise variables  $\norm{\beta_{S^c}}_1$. The dependence of the 
bound on $\norm{\beta_{S^c}}_1$ is likely an artifact of the proof.
\end{remark}

Theorem~\ref{theorem:kernel-ridge-regression} immediately implies that the kernel ridge regression objective is self-penalizing. 

\begin{corollary}
\label{cor:kernel-ridge-regression-self-penalizing}
Given Assumptions~\ref{assumption:mu-assumtion-kernels} and~\ref{assumption:distribution-P},
and  assuming $S \neq \emptyset$, define for $c, \delta, \lambda > 0$ the following set:
\begin{equation*}
	B_{c, \delta, \lambda} = \left\{\beta \in X: \FKRR_\lambda(0; \Q)-\FKRR_\lambda(\beta; \Q) >  c, 
		~\norm{\beta_{S^c}}_1 < \delta c^2 \lambda^3 \right\}.
\end{equation*}
The set $B_{c, \delta, \lambda} \neq \emptyset$ whenever $c \le c_0$, $\lambda \le \lambda_0$.
Furthermore, there exists $\delta_0$ such that the objective $\FKRR_\lambda(\beta; \Q)$ is self-penalizing with respect to
$B_{c, \delta, \lambda}$ when $\delta \le \delta_0$, $\lambda \le 1$ and $B_{c, \delta, \lambda} \neq \emptyset$.
\end{corollary}

\begin{proof}
We show that $B_{c, \delta, \lambda} \neq \emptyset$ for small enough $c, \lambda$. First,
note $\FKRR_\lambda(0; \Q) = \half \Var_Q[Y]$. Next, fix any $\beta^0 \in \mathcal{B}$ such that $\supp(\beta^0) = S \neq \emptyset$. Using
Eq.~\eqref{eqn:krr-consistent-reason} yields
\begin{equation*}
	\lim_{\lambda \to 0} \FKRR_\lambda (\beta^{0}; \Q) = \half \cdot \E_Q [\Var (Y | X_{S})] 
		< \half \cdot \Var_Q(Y) = \FKRR_\lambda(0; \Q). 
\end{equation*}
Hence, there exists $c_0 > 0$ such that $\FKRR_\lambda (\beta^{0}; \Q) < \FKRR_\lambda(0; \Q) - c_0$ for all $\lambda \le \lambda_0$.
As a result,  $\beta^0 \in B_{c, \delta, \lambda}$ whenever $\lambda \le \lambda_0$, $c \le c_0$. This proves the first part 
of the corollary. The rest of the corollary follows from the gradient lower bound~\eqref{eqn:self-penalization-kernel-ridge-regression} in
Theorem~\ref{theorem:kernel-ridge-regression}.
\end{proof}

\subsection{Proof of Theorem~\ref{theorem:metric-learning}}
\label{sec:proof-theorem-metric-learning}

Recalling the integral representation of the kernel function, $h$, in 
Eq.~\eqref{eqn:f-kernels}, we have:
\begin{equation}
\label{eqn:function-h-representation}
	h(z) = \int_0^\infty e^{-tz} \mu(dt)~~\Rightarrow~~h^\prime(z) 
	  = - \int_0^\infty t e^{-tz} \mu(dt),
\end{equation}
for some measure $\mu$ on $\R_+$.
A key observation from Eq.~\eqref{eqn:function-h-representation} is that 
both the function $h$ and its derivative $h'$ are mixtures of exponential 
functions. We shall exploit this fact to see how the gradient of the metric 
learning objective can be bounded by the objective itself.

Using the integral representation in Eq.~\ref{eqn:function-h-representation}, 
we have:
\begin{equation}
\begin{split}
\label{eqn:substitute-two}
	\FML(\beta; \Q) &= -\E_Q[YY' h(\norm{X-X'}_{q, \beta}^q)] \\
	&= - \int_0^\infty \E_Q\left[YY' e^{-t \norm{X-X'}_{q, \beta}^q} \right] \mu(dt). 
\end{split}
\end{equation}
Similarly, we have the following representation for the gradient:
\begin{equation}
\begin{split}
\label{eqn:gradient-representation}
	\partial_{\beta_j} \FML(\beta; \Q) &= -\E_Q[YY' h^\prime(\norm{X-X'}_{q, \beta}^q) |X_j - X_j'|^q]. \\
	&= \int_0^\infty \E_Q\left[YY' t e^{-t \norm{X-X'}_{q, \beta}^q} |X_j - X_j'|^q\right] \mu(dt).
\end{split}
\end{equation}
Based on these representations, it remains to prove the following integral inequality 
from which we derive the desired relationship between $\FML$ and its gradient: for 
any $j \in S^c$,
\begin{equation}
\label{eqn:key-to-proof-self-penalization-metric-learning}
\begin{split}
	&\int_0^\infty \E_Q\left[YY' t e^{-t \norm{X-X'}_{q, \beta}^q} |X_j - X_j'|^q\right] \mu(dt) \\
		&~~~~~~~~~~~~~~\ge \underline{c}(\beta) \cdot \int_0^\infty \E_Q\left[YY' e^{-t \norm{X-X'}_{q, \beta}^q} \right] \mu(dt).
\end{split}
\end{equation}
The proof of inequality~\eqref{eqn:key-to-proof-self-penalization-metric-learning} 
is straightforward. Recalling
that 
$\E[Y|X] = f^*(X_S)$ and $X_S \perp X_{S^c}$, and using the following multiplicative 
property of the exponential function:
\begin{equation*}
	e^{-\norm{X-X'}_{q, \beta}^q} = e^{-\normsmall{X_S-X_S'}_{q, \beta}^q} \cdot e^{-\normsmall{X_{S^c}-X_{S^c}'}_{q, \beta}^q},
\end{equation*} 
we can decompose the integrand on the left-hand side of 
Eq.~\eqref{eqn:key-to-proof-self-penalization-metric-learning} into 
two terms:
\begin{equation}
\label{eqn:integrand-LHS}
\begin{split}
	& \E_Q\left[YY' t e^{-t \norm{X-X'}_{q, \beta}^q} |X_j - X_j'|^q\right] \\
	&= \E_Q \left[YY' e^{-t \normsmall{X_S-X_S'}_{q, \beta_S}^q}\right] \cdot 
			\E_Q\left[t e^{-t \normsmall{X_{S^c} - X_{S^c}'}_{q, \beta_{S^c}}^q} |X_j - X_j'|^q\right]. 
\end{split} 
\end{equation}
Now we lower bound the first term on the right-hand side:
\begin{equation}
\label{eqn:lower-bound-for-the-first-term}
\begin{split}
	 \E_Q \left[YY' e^{-t \normsmall{X_S-X_S'}_{q, \beta_S}^q}\right] 
	 	&\ge  \E_Q \left[YY' e^{-t \normsmall{X_S-X_S'}_{q, \beta_S}^q}\right]  \cdot \E_Q \left[e^{-t \normsmall{X_{S^c}-X_{S^c}'}_{q, \beta_{S^c}}^q}\right] \\
		&=  \E_Q \left[YY' e^{-t \normsmall{X-X'}_{q, \beta}^q}\right]. 
\end{split}
\end{equation}
Note that we have used the fact that $(x, x') \mapsto \exp(-\norm{x-x'}_q^q)$ 
is a positive definite kernel in the derivation of the inequality in
Eq.~\eqref{eqn:lower-bound-for-the-first-term}; this guarantees that the 
quantities on both sides are nonnegative.  Combining 
Eq.~\eqref{eqn:lower-bound-for-the-first-term} 
and the decomposition~\eqref{eqn:integrand-LHS}, we derive the following lower bound 
for the left-hand side of Eq.~\eqref{eqn:key-to-proof-self-penalization-metric-learning}:
\begin{equation}
\label{eqn:lower-bound-for-the-integrand-LHS}
\begin{split}
	&\int_0^\infty \E_Q\left[YY' t e^{-t \norm{X-X'}_{q, \beta}^q} |X_j - X_j'|^q\right] \mu(dt) \\
	&\ge \int_0^\infty \E_Q \left[YY' e^{-t \normsmall{X-X'}_{q, \beta}^q}\right] \cdot  
		\E_Q \left[t  e^{-t \normsmall{X_{S^c} - X_{S^c}'}_{q, \beta_{S^c}}^q} |X_j - X_j'|^q\right] \mu(dt)
\end{split} 
\end{equation}
Finally, we use the assumption that $\supp(\mu) \subseteq [m_\mu, M_\mu]$, 
where $0 < m_\mu < M_\mu < \infty$, to lower bound the second term in the 
integral.  Indeed, for all $t\in \supp(\beta)$, it is lower bounded by 
$\underline{c}(\beta)$, a quantity independent of $t$: 
\begin{equation}
\label{eqn:second-quantity-integrand-RHS}
\begin{split}
	&\E_Q \left[t  e^{-t \normsmall{X_{S^c} - X_{S^c}'}_{q, \beta_{S^c}}^q} |X_j - X_j'|^q\right]  \\
	&\ge m_\mu \E_Q \left[e^{-M_\mu \normsmall{X_{S^c} - X_{S^c}'}_{q, \beta_{S^c}}^q} |X_j - X_j'|^q\right]
		 = \underline{c}(\beta).
\end{split}
\end{equation}
Substituting Eq.~\eqref{eqn:second-quantity-integrand-RHS} into Eq.~\eqref{eqn:lower-bound-for-the-integrand-LHS}, we obtain the target 
inequality~\eqref{eqn:key-to-proof-self-penalization-metric-learning} as desired.

\bigskip
\begin{remark} We summarize the key elements used in the derivations: 
\begin{itemize}
\item Completely monotone functions $h$ and $h^\prime$ are mixtures of exponential functions $\exp(-z)$.
\item The derivative of the exponential function $z\mapsto \exp(-z)$ is the negative 
exponential function.  This supplies the foundation that allows the gradient to be 
bounded by the objective value. 
\item The exponential function is multiplicative: 
$\exp(-(z_1+z_2)) = \exp(-z_1) \cdot \exp(-z_2)$, 
and therefore the contributions of signal and noise variables (which are assumed independent) 
to the gradient can be seamlessly decomposed---any random quantities $Z_1 \perp Z_2$ 
satisfy $\E[\exp(-(Z_1 +Z_2))] = \E[\exp(-Z_1)] \cdot \E[\exp(-Z_2)]$. 
\item The exponential function is a kernel: the mapping $(x, x') \mapsto \exp(-\norm{x-x'}_q^q)$ 
is positive definite.  This guarantees that all bounds are meaningful in the right direction. 
\end{itemize}
The proof sheds light on the crucial need for using \emph{nonlinear} kernels 
in order to obtain the self-penalization property. As we saw in 
Section \ref{sec:classical-objectives-not-self-penalizing}, the objective 
function with the \emph{linear} kernel is not self-penalizing.
\end{remark}

 \subsection{Skech of the proof of Theorem~\ref{theorem:kernel-ridge-regression}}
 \label{sec:proof-theorem-kernel-ridge-regression}
 
The proof is based on the following representation of the 
gradient~\citep[see][Proposition 4]{JordanLiRu21}.  For any $j \in [p]$, we have:
 \begin{equation}
 \label{eqn:gradient-representation-KRR}
 	\partial_{\beta_j} \FKRR_\lambda(\beta; \Q) = 
 		-\frac{1}{\lambda} \cdot \E_Q[z_{\beta; \lambda}(X; Y)z_{\beta; \lambda}(X'; Y') h^\prime(\norm{X-X'}_{q, \beta}^q) |X_j - X_j'|^q],
 \end{equation}
where $z_{\beta; \lambda}(x; y) = y - f_{\beta; \lambda}(\beta^{1/q} \odot x)$ 
denotes the residual function.  Given this result, and noting the resemblance 
between the gradient representations of the KRR objective in 
Eq.~\eqref{eqn:gradient-representation-KRR} and that of the metric learning 
objective in Eq.~\eqref{eqn:gradient-representation}, it is straightforward
to translate the proof for the metric learning objective in 
Theorem~\ref{theorem:metric-learning} to the KRR objective.  We provide the 
detailed proof in Appendix~\ref{sec:proof-theorem-kernel-ridge-regression-appendix}.


\section{The Structure of the Gradient Dynamics}
\label{sec:gradient-dynamics}
We turn to a characterization of the path of projected gradient descent when 
applied to the metric learning and the kernel ridge regression objectives. 
We show that, with high probability, the gradient dynamics satisfy the 
following pattern (outlined in Section~\ref{sec:proof-outline}):
\begin{itemize}
\item With appropriate initialization, the gradient dynamics will enter 
a set $B \subseteq \mathcal{B}$ at some finite time $t_0 < \infty$. 
The set $B$ satisfies two properties: (i) $B$ is invariant with respect to
the gradient dynamics, and (ii) the objective is self-penalizing on $B$. 
\end{itemize}
As a consequence, gradient dynamics will stay in the set $B$ after time $t_0$ and force the coordinates of $\beta$ corresponding to noise variables to 0. Section~\ref{sec:metric-learning-gradient-dynamics} and Section~\ref{sec:KRR-gradient-dynamics} 
present the analysis of the metric learning and the kernel ridge regression objectives respectively. 
The analysis starts with the population objective and then extends the results to finite samples.

\newcommand{\PsiML}{\Psi^{{\rm ML}}}
\subsection{Metric learning} 
\label{sec:metric-learning-gradient-dynamics}

We start by characterizing the dynamics of the population metric learning objective:  
\begin{equation}
\label{eqn:population-dynamics-metric-learning}	
	\dot{\beta}(t) \in - \grad \FML(\beta(t); \Q) - \normal_\mathcal{B}(\beta(t))~~\text{and}~~\beta(0) \in \mathcal{B},
\end{equation}
where $\mathcal{B}= \{\beta: \beta \ge 0, \norm{\beta}_\infty \le M\}$. Let us denote
\begin{equation}
\label{eqn:invariance-set}
	B_c = \{\beta \in \mathcal{B}: \FML(\beta; \Q) < -c\}.
\end{equation}

\begin{proposition}
\label{proposition:metric-learning-population}
Given Assumptions~\ref{assumption:mu-assumtion-kernels} and~\ref{assumption:distribution-P},
let $S \neq \emptyset$ and assume $S \subseteq \supp(\beta(0))$. Then there exists 
a constant $c > 0$ such that the trajectories $t \mapsto \beta(t)$ satisfy the 
following properties:
\begin{enumerate}
\item[(a)] The gradient dynamics enters $B_{c}$ in finite time. In fact, 
$\beta(0) \in B_{3c} \subseteq B_{c}$. 
\item[(b)] The objective $\FML(\cdot; \Q)$ is self-penalizing on the set $B_c$. 
\item[(c)] The set $B_c$ is invariant with respect to the gradient dynamics of $\FML(\cdot; \Q)$. 
\end{enumerate}
As a consequence, there exists $\tau < \infty$ so that $\emptyset \subsetneq \supp(\beta(t)) \subseteq S$ for all $t \ge \tau$.  
\end{proposition}

\begin{proof}
Let $c =  |\FML(\beta(0); \Q)|/4 > 0$. Note $c > 0$ since $\emptyset\subsetneq S \subseteq \supp(\beta(0))$ and 
$\FML$ is a nonparametric dependence measure.  The claim (a) is true by 
construction since $\beta(0) \in B_{3c} \subseteq B_c$.  Claim (b) is a 
restatement of the self-penalizing property (Corollary~\ref{cor:metric-learning-self-penalizing}).
Finally, claim (c) follows from the monotonicity of the gradient dynamics (Theorem~\ref{thm:main-result-grad-inclusion}). 
\end{proof}



Next, we consider the finite-sample dynamics of the metric learning objective:  
\begin{equation}
\label{eqn:finite-sample-dynamics-metric-learning}		
	\dot{\beta}(t) \in - \grad \FML(\beta(t); \Q_n) - \normal_\mathcal{B}(\beta(t))~~\text{and}~~\beta(0) \in \mathcal{B}.
\end{equation}
We use the notation $\wtilde{\beta}(t)$ to denote the solution that is output
by the finite-sample dynamics. Introduce 
\begin{equation}
	\wtilde{B}_{c, n} = \{\beta \in \mathcal{B}: \FML(\beta; \Q_n) < -c\}.
\end{equation}
The definition of $\wtilde{B}_{c, n}$ parallels that of $B_c$ (see Eq.~\eqref{eqn:invariance-set}). 

Let us define $\Omega_n(\eps)$ as the event on which the empirical gradients and objective values are uniformly close to their population counterparts over the feasible set $\mathcal{B}$: 
\begin{equation}
\label{eqn:omega-n-metric-learning}
\begin{array}{c}
	|\FML(\beta; \Q_n) - \FML(\beta; \Q)| \le \eps,~~~\\
	\norm{\grad \FML(\beta; \Q_n) - \grad \FML(\beta; \Q)}_\infty \le \eps
\end{array}~~\text{for all $\beta \in \mathcal{B}$}.
\end{equation}
Lemma~\ref{lemma:uniformly-close-population-finite-samples} shows that the event $\Omega_n(\eps)$ 
occurs with high probability. The proof, which is based on standard concentration inequalities, 
is deferred to Section~\ref{sec:proof-lemma-uniformly-close-population-finite-samples}.
\begin{lemma}
\label{lemma:uniformly-close-population-finite-samples}
Given Assumptions~\ref{assumption:mu-assumtion-kernels} and~\ref{assumption:distribution-P}, 
there exists a constant $c' > 0$ independent of $n$ such that $\Omega_n(\eps)$ happens with 
probability at least $1-\exp(-c'n\eps^2)$. 
\end{lemma}


\setcounter{proposition}{5}
\renewcommand\theproposition{\arabic{proposition}'}

Now we are ready to state the empirical analogue of Proposition~\ref{proposition:metric-learning-population}.

\begin{proposition}
\label{proposition:metric-learning-finite-sample}
Given Assumptions~\ref{assumption:mu-assumtion-kernels} and~\ref{assumption:distribution-P},
and let $S \neq \emptyset$ and $S \subseteq \supp(\wtilde{\beta}(0))$.
There exist constants $\eps, c, \wtilde{c} > 0$ independent of $n$ such that on the event $\Omega_n(\eps)$, the empirical gradient 
dynamics $t \mapsto \wtilde{\beta}(t)$ and the set $\wtilde{B}_{\wtilde{c}, n}$ satisfy the following properties: 
\begin{itemize}
\item[(a)] The gradient dynamics enters $\wtilde{B}_{\wtilde{c}, n}$ in finite time. 
\item[(b)] The objective $\FML(\cdot; \Q_n)$ is self-penalizing on the set $\wtilde{B}_{\wtilde{c}, n}$. 
\item[(c)] The set $\wtilde{B}_{\wtilde{c}, n}$ is invariant with respect to the gradient dynamics of $\FML(\cdot; \Q_n)$. 
\end{itemize}
Hence, on the event $\Omega_n(\eps)$, there exists $\tau < \infty$ so that
$\emptyset \subsetneq \supp(\beta(t)) \subseteq S$ for all $t \ge \tau$. The constant $\tau > 0$ is independent of $n$. 
\end{proposition}

\begin{proof}
Proposition~\ref{proposition:metric-learning-finite-sample} is almost an immediate consequence of 
Proposition~\ref{proposition:metric-learning-population}. Let $c > 0$ be a constant such that 
Proposition~\ref{proposition:metric-learning-population} holds. 
The following fact is used throughout the proof, which follows from the triangle 
inequality and the definition of $\Omega_n(\eps)$ for $\eps < c$: 
\begin{equation}
\label{eqn:basic-set-containment}
B_{3c} \subseteq \wtilde{B}_{2c, n} \subseteq B_c~~\text{on the event $\Omega_n(\eps)$}.
\end{equation} 
Set $\wtilde{c} = 2c$, and note that $\wtilde{\beta}(0) = \beta(0) \in B_{3c}$ 
by definition of $c$ (see Proposition~\ref{proposition:metric-learning-population}). 
Hence $\wtilde{\beta}(0) \in \wtilde{B}_{\wtilde{c}, n}$ on the event $\Omega_n(\eps)$ 
by virtue of Eq.~\eqref{eqn:basic-set-containment}. 

Next, we show that for some $0 < \eps \le c$, the objective $\FML(\cdot; \Q_n)$ 
is self-penalizing on $\wtilde{B}_{\wtilde{c}, n}$ on the event $\Omega_n(\eps)$. 
Proposition~\ref{proposition:metric-learning-population} says that $\FML(\beta; \Q)$ 
is self-penalizing on the set $B_c$. Hence the \emph{population} gradient 
$\partial_{\beta_j} \FML(\beta; \Q)$ has a uniform lower bound $\underline{c} > 0$ 
over the set of $\beta \in B_c$ for noise variables $j\not\in S$. Let us pick 
$\eps = \underline{c}/2 \wedge c$. On the event $\Omega_n(\eps)$, the \emph{empirical} 
gradient $\partial_{\beta_j} \FML(\beta; \Q_n)$ has a uniform lower bound 
$\underline{c}/2 > 0$ over $\beta \in B_c$ and $j \not\in S$ (by the triangle 
inequality). This implies $B_{\wtilde{c}, n} \subseteq B_c$ on $\Omega_n(\eps)$.

Finally, the invariance of $\wtilde{B}_{\wtilde{c}, n}$ follows from the 
monotonicity of the gradient dynamics (Theorem~\ref{thm:main-result-grad-inclusion}). 
\end{proof}


\paragraph{Proof of Theorem~\ref{thm:FML-sparse-one-round}}
Theorem~\ref{thm:FML-sparse-one-round} follows from
Proposition~\ref{proposition:metric-learning-finite-sample} and Lemma~\ref{lemma:uniformly-close-population-finite-samples}.

\subsection{Kernel ridge regression} 
\label{sec:KRR-gradient-dynamics}

In this section we characterize the gradient dynamics of the population 
kernel ridge regression (KRR) objective: 
 \begin{equation}
 \label{eqn:population-dynamics-KRR}	
 	\dot{\beta}(t) \in - \grad \FKRR_\lambda(\beta(t); \Q) - \normal_\mathcal{B}(\beta(t))~~\text{and}~~\beta(0) \in \mathcal{B},
 \end{equation}
where $\mathcal{B}= \{\beta \in \R_+^p: \norm{\beta}_\infty \le M\}$. Let us denote 
 \begin{equation}
 \label{eqn:invariance-set-KRR}
 	B_{c, \delta, \lambda} = \left\{\beta \in \mathcal{B}: 
 		\FKRR_\lambda(\beta; \Q) - \FKRR_\lambda(0; \Q)  < -c,~\norm{\beta_{S^c}}_1 < c^2 \delta \lambda^3 \right\}.
 \end{equation}
Let $\circup{S} = \{l \in S: \Var_Q(\E_Q[Y|X_l]) \neq 0\}$ denote the set of main effects. 

 \setcounter{proposition}{6}
 \renewcommand\theproposition{\arabic{proposition}}

 \begin{proposition}
 \label{proposition:KRR-population}
Given Assumptions~\ref{assumption:mu-assumtion-kernels} and~\ref{assumption:distribution-P}, 
assume $\circup{S} \neq \emptyset$. Set $q=1$ and $\beta(0) = 0$. There exist constants 
$c, C, \delta, \lambda_0 > 0$ such that the following holds for any $\lambda \le \lambda_0$:
 \begin{itemize}
 \item[(a)] The gradient dynamics enters $B_{3c, \delta/81, \lambda} 
   \subseteq B_{c, \delta, \lambda}$ at the time $t = C\lambda^2$. 
 \item[(b)] The objective $\FKRR_\lambda(\cdot; \Q)$ is self-penalizing on the set 
   $B_{c, \delta, \lambda}$. 
 \item[(c)] The set $B_{c, \delta, \lambda}$ is invariant with respect to the 
   gradient dynamics of $\FKRR_\lambda(\cdot; \Q)$. 
 \end{itemize}
As a consequence, there exists $\tau < \infty$ so that $\emptyset \subsetneq \supp(\beta(t)) \subseteq S$ for all $t \ge \tau$.  
 \end{proposition}

 \begin{proof}
We sketch the proof of Proposition~\ref{proposition:KRR-population}; for a detailed proof, 
see Appendix~\ref{sec:proof-proposition-KRR-population}.

To prove (a), note that to enter $B_{3c, \delta/81, \lambda}$ by time $t = C\lambda^2$, 
the gradient dynamics must (i) decrease the objective value sufficiently quickly 
(shown by Lemma~\ref{lemma:kernel-ridge-regression-decay} below), and 
(ii) keep $\norm{\beta_{S^c}(t)}_1$ small. For (ii), note that $\norm{\beta_{S^c}(0)}_1 = 0$
for all $t > 0$ since the self-penalizing property (Theorem~\ref{theorem:kernel-ridge-regression}) 
implies $\partial_{\beta_j} \FKRR_\lambda(\beta; \Q) \ge 0$ for all $j \not \in S$ 
when $\supp(\beta) \subseteq S$.

Claim (b) is simply a restatement of the self-penalizing property 
(see Corollary~\ref{cor:kernel-ridge-regression-self-penalizing}).

For claim (c), note that the self-penalizing property 
(Corollary~\ref{cor:kernel-ridge-regression-self-penalizing}) guarantees 
that $\norm{\beta_{S^c}}_1$ will always be decreasing on $B_{c, \delta, \lambda}$. 
This fact combined with the monotonicity of the gradient dynamics 
(Theorem~\ref{thm:main-result-grad-inclusion}) implies the claim.
 \end{proof}

\begin{lemma}
\label{lemma:kernel-ridge-regression-decay}
Given Assumptions~\ref{assumption:mu-assumtion-kernels} and~\ref{assumption:distribution-P},
and  assuming $\circup{S} \neq \emptyset$, set $q=1$ and $\beta(0) = 0$.  
There exist constants $\lambda_0, c, C > 0$ such that the following holds 
for any $\lambda \le \lambda_0$:
\begin{equation}
 	\FKRR_\lambda(\beta(t); \Q) \le \FKRR_\lambda(\beta(0); \Q) - c~~\text{when $t \ge C\lambda^2$}.
\end{equation}
 \end{lemma}

\begin{proof}
The result is a consequence of two results: (i) the Lyapunov convergence theorem 
for gradient inclusion, which says that the objective value has a certain rate 
of decay governed by the size of the gradient along the trajectory: 
 \begin{equation}
 	\FKRR_\lambda(\beta(t); \Q) \le \FKRR_\lambda(\beta(0); \Q) - \int_0^t \norm{\subgrad\opt(\beta(s))}_2^2 ds, 
 \end{equation}
 where the mapping $\subgrad\opt(\beta) \defeq \argmin\{\norm{g}_2: 
g \in \grad \FKRR_\lambda(\beta; \Q) + \normal_\mathcal{B}(\beta)\}$
is the minimum norm gradient at $\beta$, and (ii) a lower bound on the 
size of the kernel ridge regression gradient at the beginning stage of 
the differential inclusion. The details of the proof of 
Lemma~\ref{lemma:kernel-ridge-regression-decay} are given in 
Section~\ref{sec:proof-lemma-kernel-ridge-regression-decay}. 

 \end{proof}

Next, we consider the finite-sample dynamics of the kernel ridge regression objective: 
 \begin{equation}	
 \label{eqn:finite-sample-dynamics-KRR}
 	\dot{\beta}(t) \in - \grad \FKRR_\lambda(\beta(t); \Q_n) 
	  - \normal_\mathcal{B}(\beta(t))~~\text{and}~~\beta(0) \in \mathcal{B}.
 \end{equation}
Let $\wtilde{\beta}(t)$ denote the solution of the finite-sample dynamics. Introduce 
 \begin{equation}
 \label{eqn:invariance-set-KRR-empirical}
 	\wtilde{B}_{c, \delta, \lambda, n} = \left\{\beta \in \mathcal{B}: 
 		\FKRR_\lambda(\beta; \Q_n) - \FKRR_\lambda(0; \Q_n)  < -c,~\norm{\beta_{S^c}}_1 < c^2 \delta \lambda^3 \right\}.
 \end{equation}
 The definition of $\wtilde{B}_{c, \delta, \lambda, n}$ parallels that of $B_{c, \delta, \lambda}$ (see Eq.~\eqref{eqn:invariance-set-KRR}). Introduce $\Omega_n(\eps)$, defined as the event on which the empirical and population gradients and objective values are uniformly close over the feasible set $\mathcal{B}$: 
 \begin{equation}
 \label{eqn:omega-n-KRR}
 \begin{array}{c}
 	|\FKRR_\lambda(\beta; \Q_n) - \FKRR_\lambda(\beta; \Q)| \le \eps,~~~\\
 	\norm{\grad \FKRR_\lambda(\beta; \Q_n) - \grad \FKRR_\lambda(\beta; \Q)}_\infty \le \eps
 \end{array}~~\text{for all $\beta \in \mathcal{B}$}.
 \end{equation}
 Lemma~\ref{lemma:uniformly-close-population-finite-samples-KRR} shows that the event $\Omega_n(\eps)$ happens with 
 high probability. The proof, which is based on standard concentration inequalities, is deferred to Appendix~\ref{sec:proof-lemma-uniformly-close-population-finite-samples-KRR}. 
 \begin{lemma}
 \label{lemma:uniformly-close-population-finite-samples-KRR}
Given Assumptions~\ref{assumption:mu-assumtion-kernels} and~\ref{assumption:distribution-P}, 
assume $\lambda \le 1$. 
 Then the event $\Omega_n(\eps)$ happens with probability at least $1-\exp(-c'n\eps^2 \lambda^6)$ for any $\eps > 0$.
 Here $c'$ is independent of $n, \lambda$. 
 \end{lemma}

 Recall that $\circup{S} = \{l \in S: \Var_Q(\E_Q[Y|X_l]) \neq 0\}$ denotes the set of main effects.

 \setcounter{proposition}{6}
 \renewcommand\theproposition{\arabic{proposition}'}


 \begin{proposition}
 \label{proposition:KRR-finite-sample}
 Given Assumptions~\ref{assumption:mu-assumtion-kernels} and~\ref{assumption:distribution-P}, 
 assume $\circup{S} \neq \emptyset$, and set $q=1$ and $\wtilde{\beta}(0) = 0$. 
Let $\lambda \le 1$. There exist
 constants $\wtilde{c}, \wtilde{C}, \wtilde{\delta}, \eps, \lambda_0 > 0$ independent of $n$ such that for any $\lambda \le \lambda_0$
 the empirical gradient dynamics $t \mapsto \wtilde{\beta}(t)$ 
 satisfies the following on the event $\Omega_n(\eps\lambda)$: 
 \begin{itemize}
 \item[(a)] The gradient dynamics enters $\wtilde{B}_{\wtilde{c}, \wtilde{\delta}, \lambda, n}$ at the time $t = \wtilde{C}\lambda^2$. 
 \item[(b)] The objective $\FKRR_\lambda(\cdot; \Q)$ is self-penalizing on the set $\wtilde{B}_{\wtilde{c}, \wtilde{\delta}, \lambda, n}$. 
 \item[(c)] The set $\wtilde{B}_{\wtilde{c}, \wtilde{\delta}, \lambda, n}$ is invariant with respect to the gradient dynamics of $\FKRR_\lambda(\cdot; \Q)$. 
 \end{itemize}
 As a consequence, there exists $\tau < \infty$ so that $\emptyset \subsetneq \supp(\beta(t)) \subseteq S$ for all $t \ge \tau$
 on the event $\Omega_n(\eps\lambda)$. Furthermore, we can choose $\tau$ to be $\tau = \overline{c}\lambda^2$ where
 $\overline{c}$ is independent of $\lambda, n$. 
 \end{proposition}

\begin{proof}
Proposition~\ref{proposition:KRR-finite-sample} follows from its population version
(Proposition~\ref{proposition:KRR-population}), and a Gr\"{o}nwall deviation bound 
that controls the difference between the population and empirical trajectories 
(see Lemma~\ref{lemma:gronwall-deviation-bounds}). 
For details, see Section~\ref{sec:proof-proposition-KRR-finite-sample}. 
\end{proof}

 \paragraph{Proof of Theorem~\ref{thm:FKRR-sparse-one-round}}
 Theorem~\ref{thm:FKRR-sparse-one-round} follows from
 Proposition~\ref{proposition:KRR-finite-sample} and Lemma~\ref{lemma:uniformly-close-population-finite-samples-KRR}.

\renewcommand\theproposition{\arabic{proposition}}

\section{Statistical Implications}
\label{sec:statistical-implications}

This section shows how we can leverage self-penalizing objectives to develop procedures 
that consistently select signal variables without any need for $\ell_1$ regularization. 
In Section \ref{sec:gradient-dynamics}, we showed that, when properly initialized, 
gradient flow applied to a self-penalizing objective converges to a stationary point 
which automatically excludes all noise variables. That is, $\what{S} \subseteq S$, 
where $\what{S} = \supp(\what{\beta})$ is the support of the stationary point found 
by gradient flow. Note, however, that we do not have the guarantee that 
$\E[Y|X] = \E[Y|X_{\what{S}}]$; i.e., we may miss relevant variables.\footnote{The 
precise meaning of the notation $\E[Y|X_{\what{S}}]$ is as follows.  We take a 
sample $(X, Y)$ (a test sample) independent from the raw data and evaluate its 
conditional expectation $\E[Y|X_{\what{S}}]$.} Indeed, counterexamples show that relevant 
variables can be missing from even the \emph{global} minima of the \emph{population} 
objective~\citep{LiuRu20}.

Below we show how to use the metric learning and kernel ridge regression objectives to construct $\what{S}$ such that 
(i) $\what{S} \subseteq S$ and (ii) $\E[Y|X] = \E[Y|X_{\what{S}}]$. 
There is enough similarity in the algorithms used for these two objectives to 
abstract out a recipe. Adaptation of this recipe may be useful when designing selection algorithms 
based on other self-penalizing objectives, although no general theory is pursued here. 
The recipe involves solving a \emph{sequence} of minimization problems. At each iteration 
$k$ in the sequence, the minimizing problem we solve depends on the variables that have 
been selected prior to iteration $k$, denoted $\what{S}^{(k)}$.  The objective function 
$F(\beta; \Q^{\what{S}^{(k)}})$, the constraint set $\mathcal{B}_{ \what{S}^{(k)}}$ and 
the initializer, $\beta(0; \what{S}^{(k)})$, can all depend on $\what{S}^{(k)}$. 
Our recipe proceeds as follows:
\begin{enumerate}
	\item Initialize $\what{S}^{(0)} = \emptyset$ and $k=0$.
	\item Run gradient descent to minimize $F(\beta; \Q^{\what{S}^{(k)}})$ over the constraint set $\mathcal{B}_{\what{S}^{(k)}}$ 
		starting from initialization $\beta(0; \what{S}^{(k)})$. Let $\what{\beta}^{(k)}$ be the solution.
	\item Update $\what{S}^{(k + 1)} = \what{S}^{(k)} \cup \supp(\what{\beta}^{(k)})$.
	\item Perform a statistical test for the null hypothesis $H_0: \E[Y|X_{\what{S}^{(k + 1)}}] = \E[Y|X_{\what{S}^{(k)}}]$. If rejected, return to Step 2 and increment $k = k + 1$. Else, return $\what{S} = \what{S}^{(k)}$.
\end{enumerate}
Intuitively, the minimization problem in Step 2 should be set up with the goal of 
finding a set of variables that is predictive of $Y$ \emph{conditional} on the selected 
variables in $S^{(k)}$.  The specifics of how to do this will vary with the 
type of self-penalizing objective since it depends on the mechanism by which an 
objective tries to detect signal. 
Section \ref{subsec:ml_sequence} describes the sequence of minimization problems used for the metric learning objective and 
Section \ref{subsec:krr_sequence} for the kernel ridge regression objective.

\bigskip
\begin{remark}
It is natural to wonder why Step 4 is even necessary if a self-penalizing objective 
automatically controls false positives.  A careful review of 
Theorems \ref{thm:FML-sparse-one-round} and \ref{thm:FKRR-sparse-one-round} 
shows that a necessary condition for the metric learning and kernel ridge 
regression objectives to be self-penalizing is that $S \neq \emptyset$. 
Therefore, the step is needed in the first iteration, $k=1$, to ensure 
that the signal set $S$ is not empty. Additionally, once all the relevant 
variables have been included in $\what{S}^{(k)}$, the objective in the 
subsequent iteration of Step 2 may no longer be self-penalizing 
(see Section~\ref{subsec:ml_sequence}); hence the need for Step 4.
\end{remark}

\begin{remark}
Why use a self-penalizing objective to select variables if in the end we must 
apply a nonparametric test to ascertain the validity of the discovered variables? 
One can certainly build variable selection procedures using only nonparametric 
conditional independence tests~\citep{FukumizuGrSuSc07, VergaraEs14, AzadkiaCh19} but such procedures would either require conducting exponentially many tests (over 
subsets of variables) or would require additional assumptions about the signal, 
e.g., that the signal is hierarchical. In contrast, our procedure computes at 
most $p$ conditional independence tests and possesses statistical guarantees even for 
non-hierarchical signals; see, e.g., Proposition~\ref{proposition:emp_algo}. 
The procedure attains these attractive properties by using a self-penalizing 
objective to generate a sequence of candidate variable sets to which we apply 
a nonparametric conditional test.
\end{remark}

\subsection{Classification: Metric learning algorithm}
\label{subsec:ml_sequence}


The metric learning algorithm, Algorithm~\ref{alg:metric-learning}, 
uses the metric learning objective to consistently select signal variables 
in a classification setting. Proposition~\ref{proposition:emp_algo} shows 
that the variables selected by Algorithm~\ref{alg:metric-learning} achieves 
(i) $\what{S} \subseteq S$ and (ii) $\E[Y|X] = \E[Y|X_{\what{S}}]$. 
In Algorithm~\ref{alg:metric-learning}, we solve a sequence of minimization 
problems where the minimizing objective at each step depends on the currently 
selected variables, $\what{S}$. Let $\Q$ and $\Q_n$ denotes the population 
and empirical distribution of the data.  For any subset of features,
$X_{A}$, we define the reweighted distribution $\Q_n^A$ as follows:
\begin{equation}
\label{eqn:definition-of-Q-A}
\begin{split}
\frac{d\Q_n^{A}}{d\Q_n}(x, y) &\propto \Q(Y = -y| X_{A} = x_{A}).
\end{split}
\end{equation}
The minimization objective used in each step is 
the metric learning objective with respect to the weighted distribution $\Q_n^{\what{S}}$. The reweighting has the effect of removing the effect 
of the selected variables, $X_{\what{S}}$, while leaving intact any signal attributable to the unselected variables. 
See~\cite{LiuRu20} for further details.

\begin{remark}
To evaluate $\Q_n^{\what{S}}$ in practice, we need a nonparametric estimate of the population conditional distribution
$\Q(Y| X_{\what{S}})$ (see equation \eqref{eqn:definition-of-Q-A}). In our proof of Proposition \ref{proposition:emp_algo}, 
we will ignore this estimation error and simply pretend that $\Q(Y| X_{\what{S}})$ is available. With some work, one 
can adapt our results to the case of estimated $\Q(Y|X_{\what{S}})$. Primarily, the threshold $\eps_n$ must be increased 
to account for the additional noise in the estimation process.


\end{remark}

\begin{algorithm}[!tph]
\begin{algorithmic}[1]
\Require{ Initializer $\beta^{(0)}$, $M > 0$ and $\eps_n > 0$}
\Ensure{Initialize $\what S=\emptyset$. }
\While{$\what{S}$ not converged}
\State Run projected gradient descent (with initialization $\beta^{(0)}$) to solve
	\begin{equation*}
		\minimize_{\beta} \FML(\beta; \Q^{\what{S}}_n)~~~\text{subject to}~~\beta \ge 0, ~\norm{\beta}_\infty \le M.
	\end{equation*}
~~~~~Let $\what{\beta}$ denote the stationary point found by the projected gradient descent iterates.
\State  Update $\what{S} = \what{S} \cup \supp(\what{\beta})$ if $\FML(\what{\beta}; \Q^{\what{S}}_n) < - \eps_n$.  
\EndWhile
\end{algorithmic}
\caption{Empirical Metric Learning Procedure}
\label{alg:metric-learning}
\end{algorithm}


\begin{proposition}
\label{proposition:emp_algo}
Given Assumptions~\ref{assumption:mu-assumtion-kernels} and~\ref{assumption:distribution-P}, 
assume $\E_Q[\Var_Q(Y|X)] > 0$. 
Fix an initialization $\beta^{(0)}$ that satisfies $\supp(\beta^{(0)}) = [p]$.
There exist constants $C, c, \eps_0 > 0$ independent of $n$ such that the following holds.
With probability at least $1-2^{|S|}e^{-cn \eps_n^2}$, for the choice of the threshold 
$\eps_n \ge C/\sqrt{n}$, the output $\what{S}$ of 
Algorithm~\ref{alg:metric-learning} satisfies (i) $\what S \subseteq S$, 
and (ii) $\E[Y|X] = \E[Y|X_{\what{S}}]$ if $\eps_n \le \eps_0$ in addition.
\end{proposition}

\subsection{Regression: Kernel Ridge Regression Algorithm}
\label{subsec:krr_sequence}

Algorithm \ref{alg:kernel-feature-selection-empirical} selects variables by repeatedly minimizing the KRR objective. While the same 
KRR objective is used in each minimization problem, the initialization and feasible region differ depending on what variables have 
been selected so far. Suppose $\what{S}$ has already been selected. For the next minimization problem:
\begin{itemize}
	\item Initialize at $\beta(0; \what{S})$ where $\beta_{\what{S}^c}(0; \what{S}) = 0$ and $\beta_{\what{S}}(0; \what{S}) = M \mathbf{1}_{\what{S}}$.
	\item Restrict the minimization to the feasible set $\mathcal{B}_{\what{S}} = \{\norm{\beta}_\infty \le M~\text{and}~\beta_{\what{S}} = M \mathbf{1}_{\what{S}}\}$.
\end{itemize}
That is, the variables selected in prior rounds of minimization will always be kept active in subsequent rounds. This choice of 
initialization and constraint region ensures that we do not miss possible interaction signals between selected and unselected variables. 
Proposition \ref{proposition:emp_algo_krr} shows that Algorithm~\ref{alg:kernel-feature-selection-empirical} achieves 
(i) $\what{S} \subseteq S$ and (ii) $\E[Y|X] = \E[Y|X_{\what{S}}]$ upon termination. 


\begin{algorithm}[!ht]
\begin{algorithmic}[1]
\Require{Initializer $\beta(0; \emptyset)$, $M > 0$ and $\eps_n > 0$.}
\While{$\what S$ not converged}
\State Run projected gradient descent (with initialization $\beta(0; \what{S})$) to solve 
\begin{equation*}
\begin{split}
	&\minimize_\beta \FKRR_{\lambda_n}(\beta; \Q_n)~~\text{subject to}~~\beta \ge 0,~ \norm{\beta}_\infty \le M~\text{and}~\beta_{\what{S}} = M \mathbf{1}_{\what{S}}.
\end{split}
\end{equation*}
~~~~~Let $\what{\beta}$ denote the stationary point found by the projected gradient descent iterates.
\State 
 Update $\what{S} = \what{S} \cup \supp(\what{\beta})$ if $\FKRR_\lambda(\beta^{*;\what{S}}; \Q_n) - \FKRR_\lambda(\what{\beta}; \Q_n) > \eps_n$,
	  where $\beta^{*; \what{S}}$ is defined by $\beta_{\what{S}}^{*; \what{S}} = M \mathbf{1}_{\what{S}}$ and $\beta_{\what{S}^c}^{*; \what{S}} = 0$.
\EndWhile
\end{algorithmic}
\caption{Empirical Kernel Ridge Regression Procedure}
\label{alg:kernel-feature-selection-empirical}
\end{algorithm}

Due to the nature of the provable self-penalizing mechanism for the KRR objective 
(Theorem~\ref{thm:FKRR-sparse-one-round}), we require the following hierarchical signal assumption.  
\begin{assumption}[Hierarchical signal]
\label{assumption:insufficient-set-T}
For any subset $T \subsetneq S$ for which $\E[Y|X] \neq \E[Y|X_T]$, there exists an index $j \in S \backslash T$ such that 
$\E[Y| X_T] \neq \E[Y | X_{T \cup \{j\}}]$.
\end{assumption}

\begin{proposition}
\label{proposition:emp_algo_krr}
Given Assumptions~\ref{assumption:mu-assumtion-kernels}-~\ref{assumption:insufficient-set-T},
let the initializers $\{\beta(0; T)\}_{T \subseteq [p]}$ be as follows:  $\beta(0; T)_T = M \mathbf{1}_T$
and $\beta(0; T)_{T^c} = 0$. Set $q = 1$. There exist constants $c, C, \lambda_0, \eps_0 > 0$  independent of 
$n$ such that the following holds. For any parameter $\lambda_n \le \lambda_0$ and threshold $\eps_n \ge C/(\sqrt{n}\lambda_n^3)$, 
we have that with probability at least 
$1-e^{-c n \lambda_n^6(\eps_n^2 \wedge \lambda_n^2)}$, the output of Algorithm~\ref{alg:kernel-feature-selection-empirical} satisfies
(i) $\what S \subseteq S$ and (ii) $\E[Y|X] = \E[Y|X_{\what{S}}]$ if $\eps_n \le \eps_0$ in addition. 

\end{proposition}

\section{Experiments}
\label{sec:experiments}
Our theory has established the existence of the self-penalization phenomenon---that 
the finite-sample solutions of kernel-based optimization are naturally sparse---under 
the following conditions: 
\begin{itemize}
\item Exact independence or weak dependence between the signal and the noise variables. 
\item Existence of certain type of signal variables (e.g., the main effect signal for the KRR).
\end{itemize}
In this section we present the results of numerical experiments that corroborate the theory.
These experiments also suggest that the self-penalization phenomenon occurs in a broader range 
of settings than what are currently able to establish theoretically. 

\subsection{The effect of correlation} 
Our first set of experiments investigates how correlation structures between the 
covariates affect the self-penalization.  The experimental settings are as follows. 
We draw the covariates $X$ from a normal distribution whose covariance structure 
follows a standard autoregressive model, where $\Cov(X_i, X_j) = \rho^{|i-j|}$, for 
$1\le i, j\le p$, and where a parameter $\rho \in [-1,1]$ that measures the strength 
of the correlations. The response $Y$ is constructed as follows.
\begin{itemize}
\item In the classification setting, we construct the label $Y = U \sign(f(X))$. Here $f(X)$ is the signal and $U \in \{\pm 1\}$
	is independent noise with $\P(U = 1) = 1- \P(U = -1) = \eps$. 
\item In the regression setting, we construct the label $Y = f(X) + \eps \cdot U$. Here $f(X) \in \R$ is the true signal, 
	$U \sim \normal(0, 1)$ is independent noise and $\eps > 0$ is the noise level.   
\end{itemize}
We set the initialization $\beta^{(0)}_i = 1/p$ for $i \in [p]$ and the stepsize $\alpha = 1$. 

We begin by considering the results of two simple experiments that explore the role
of correlation.  In the first setup, we consider a simple situation in which there 
are only two covariates $X_1, X_2$. We define a linear signal, $f(X) = X_1$, 
so that $X_1$ is the only signal variable. We set $n= 200$ samples, $p = 2$ 
features, and noise level $\eps = 0.1$. We choose the correlation parameter $\rho$ 
from the set $\{-1, -0.9, -0.8, \ldots, +0.8, +0.9, +1\}$. For both classification 
and regression, we run gradient descent on the metric learning objective and kernel 
ridge regression for $100$ steps. In this setup our finding is that the algorithm 
always \emph{only} picks up  the true variable $X_1$ as long as $\rho \neq \pm 1$. 
In other words, as long as the features are not perfectly correlated, the algorithm 
automatically penalizes the noise variable $X_2$ and shrinks its coefficient 
$\beta_2$ to exactly $0$. Note the same behavior does not happen if we run a 
linear least square (without $\ell_1$ penalization); in that case we always obtain 
dense solutions (i.e., $\beta_1 \neq 0, \beta_2 \neq 0$). 

In the second setup, we consider the case where $f(X) = X_1^3 + X_2^3$ is a nonlinear 
signal that involves two signal variables. The setup has $n=300$ samples, $p=10$ features, 
and noise level $\eps = 0.1$.  We choose the correlation parameter $\rho$ from the 
set $\{-1, -0.9, -0.8, \ldots, +0.8, +0.9, +1\}$, and we observe that with extremely 
high probability ($\ge 99\%$) the algorithms only pick up the two signal variables 
$X_1, X_2$ as long as the features are not perfectly correlated; i.e., $\rho \neq \pm 1$. 
This shows that the self-penalization can occur in the presence of strong correlations. 

%
%


	
\subsection{Kernel ridge regression always yields self-penalization}
This section presents numerical evidence showing that the kernel ridge regression 
performs self-penalization even if the data has no main effect signals (as required 
in Theorem~\ref{thm:FKRR-sparse-one-round}).

The experimental setup is as follows. Draw the covariate $X$ where the coordinates are i.i.d.,
$X_i \sim \normal(0,1)$, and the response is $Y = f(X) + \eps\cdot U$, where the signal 
$f(X) = X_1 X_2$ and the noise $U \sim \normal(0, 1)$. Under this setup, although 
$X_1, X_2$ are signal variables, they are not main effect signals since 
$\E[f(X) \mid X_1] = \E[f(X) \mid X_2] = 0$ (we call such signal variables 
``pure interaction signals''). In the experiments, we set $n = 200$, $p = 10$, 
andn $\eps = 0.1$. We set the ridge penalty $\lambda = 0.01$, the gradient 
descent stepsize $\alpha =1$ and the initialization $\beta^{(0)}$ where 
$\beta^{(0)}_i = 1/p$ for $i \in [p]$. 

Our results are the following. Among all the $100$ repeated experiments, the 
kernel ridge regression algorithm always includes all the pure interaction signal 
variables $X_1, X_2$, and excludes the noise variables $X_3, X_4, \ldots, X_{10}$. 
This numerical evidence suggests the kernel ridge regression is able to self-penalize 
in the absence of main effects.




\newcommand{\logit}{{\rm logit}}
\subsection{Performance in finite samples}
Finally, we conduct further numerical experiments to illustrate the self-penalization 
phenomenon under different choices of the ratio $n/p$.  We consider different forms 
of the signal in our experiments. In all experiments, we draw the covariates $X$ 
from a standard normal distribution where each coordinate is independently 
$X_i \sim \normal(0, 1)$.

\paragraph{Main effect} 
Our first setup is the main effect model. 
\begin{itemize}
\item In the classification setting, we simulate the logistic model, $\logit(Y|X) = X_1$.
\item In the regression setting, we simulate $Y = X_1 + \normal(0, 1)$. 
\end{itemize}
We fix the sample size $n = 50$ and choose the dimensionality $p$ in the range from $20$ to $400$. 
Figure~\ref{fig:main-effect} reports the true positive and false positive rate across different choices of $p$. 

\begin{figure}[!tph]
\begin{centering}
\includegraphics[width=.50\linewidth]{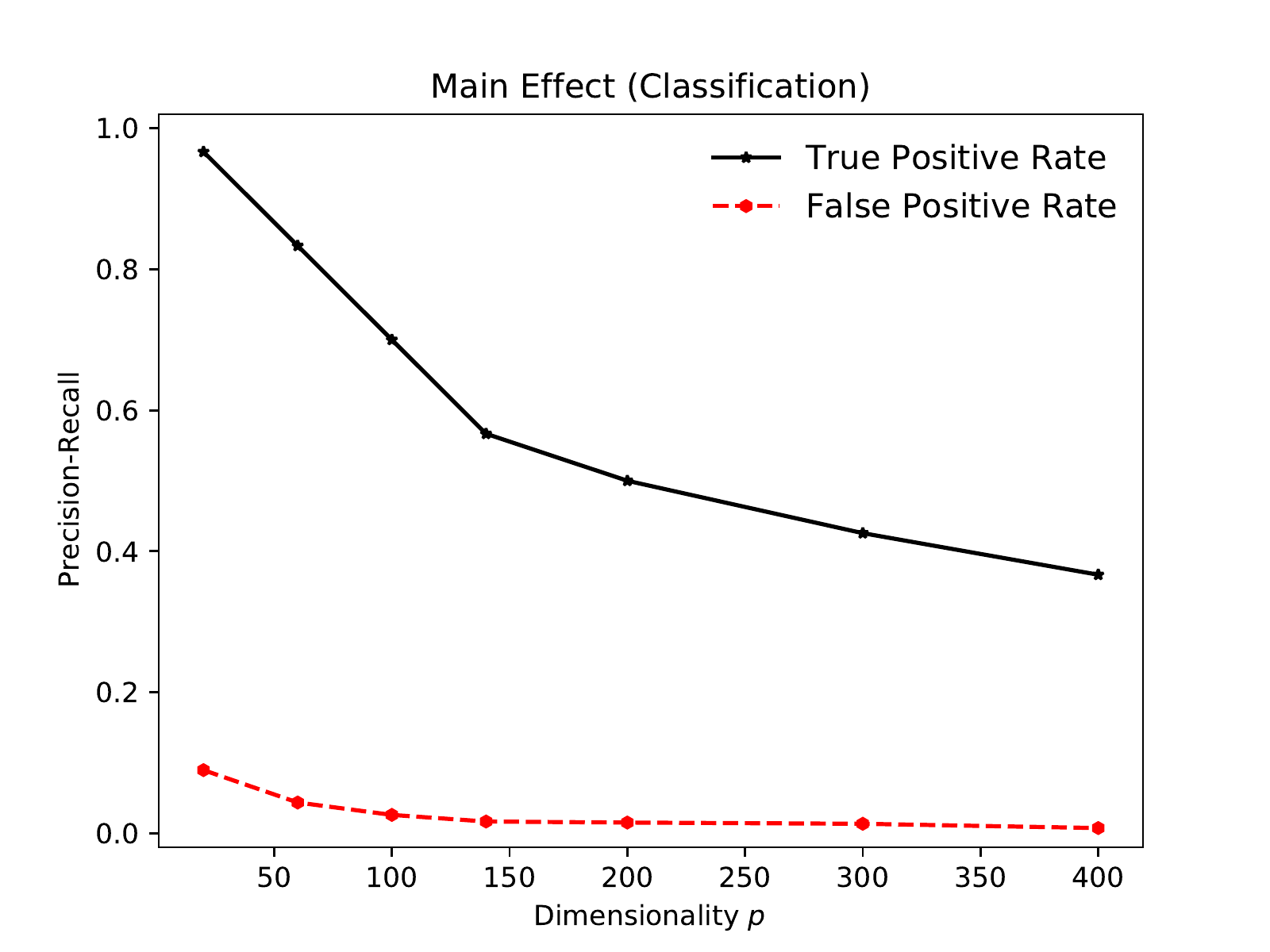}\includegraphics[width=.50\linewidth]{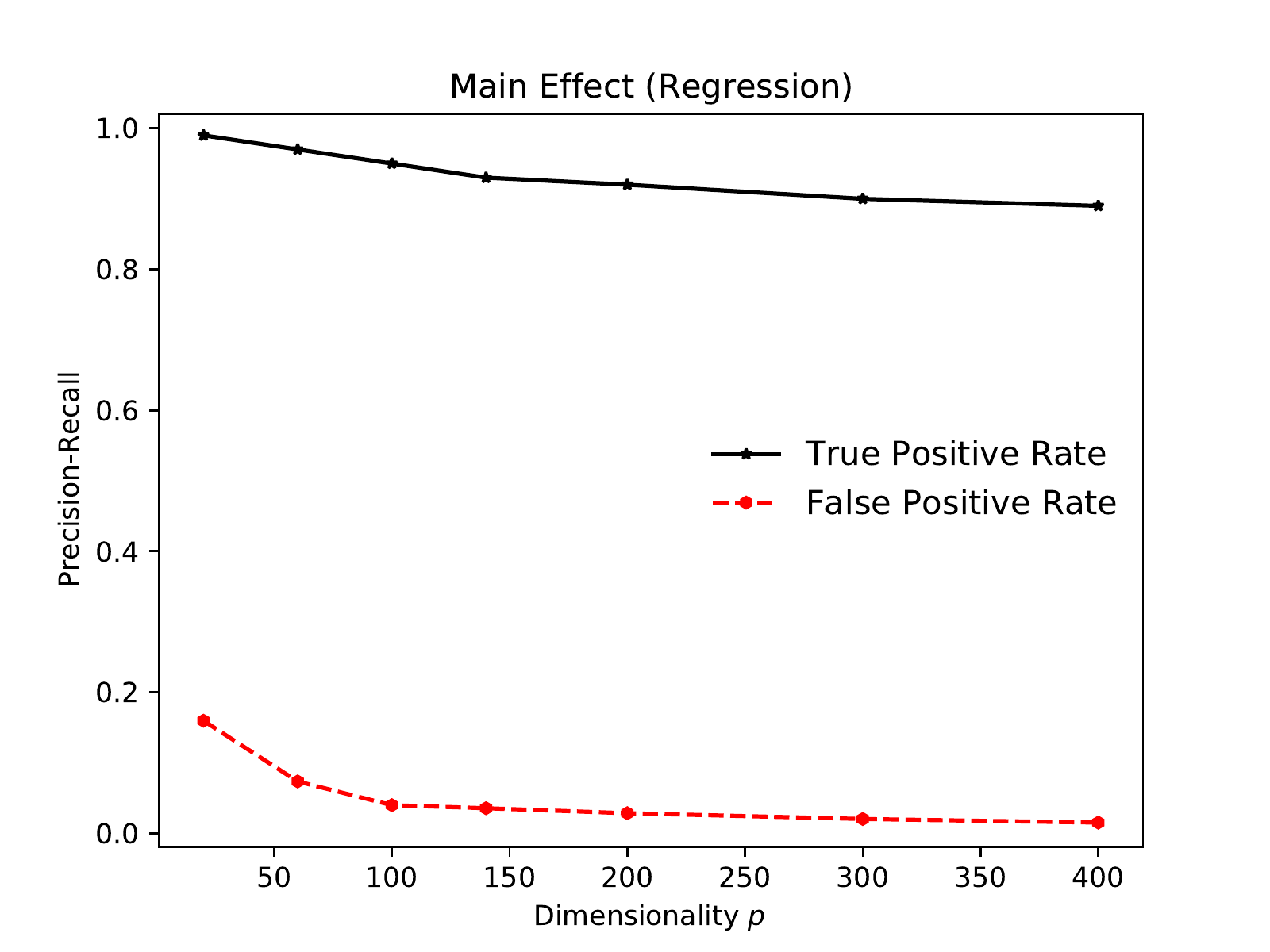}
\par\end{centering}
\caption{\small True positive rate and false positive rate curve under the main effect model against different choice of dimensionality $p$.
Here the sample size $n = 50$ is fixed.  
}
\label{fig:main-effect}
\end{figure}

\paragraph{Pure Interaction}
Our second setup is the pure interaction model. 
\begin{itemize}
\item In the classification setting, we simulate the logistic model, $\logit (Y|X) = 2X_1 X_2$.
\item In the regression setting, we simulate $Y = X_1X_2 + \normal(0, 1)$. 
\end{itemize}
We fix the sample size $n = 200$ and tune the dimensionality $p$ from $10$ to $50$. 
Figure~\ref{fig:pure-interaction} reports the true positive and false positive rate across different choices of $p$.

\begin{figure}[!tph]
\begin{centering}
\includegraphics[width=.50\linewidth]{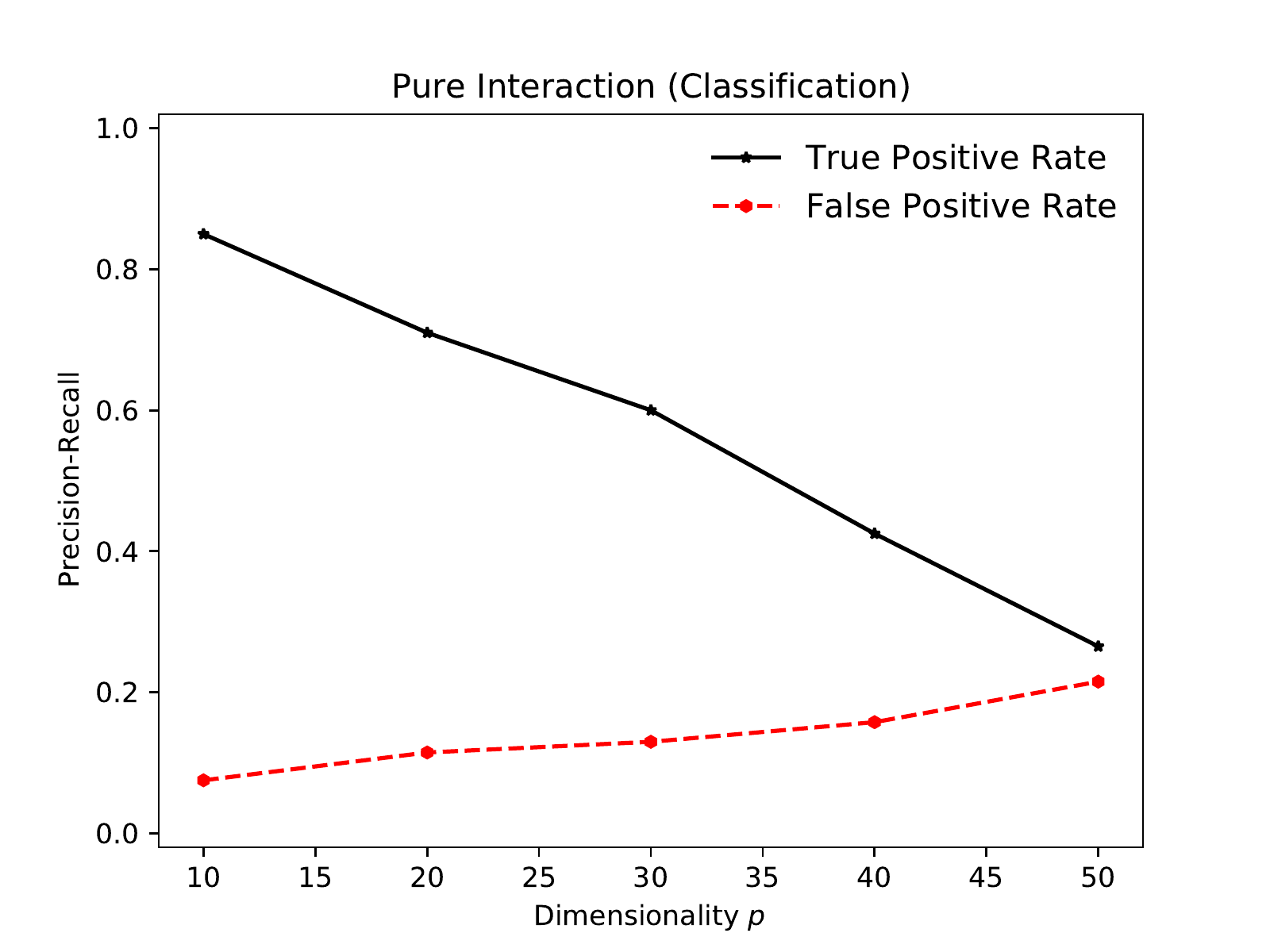}\includegraphics[width=.50\linewidth]{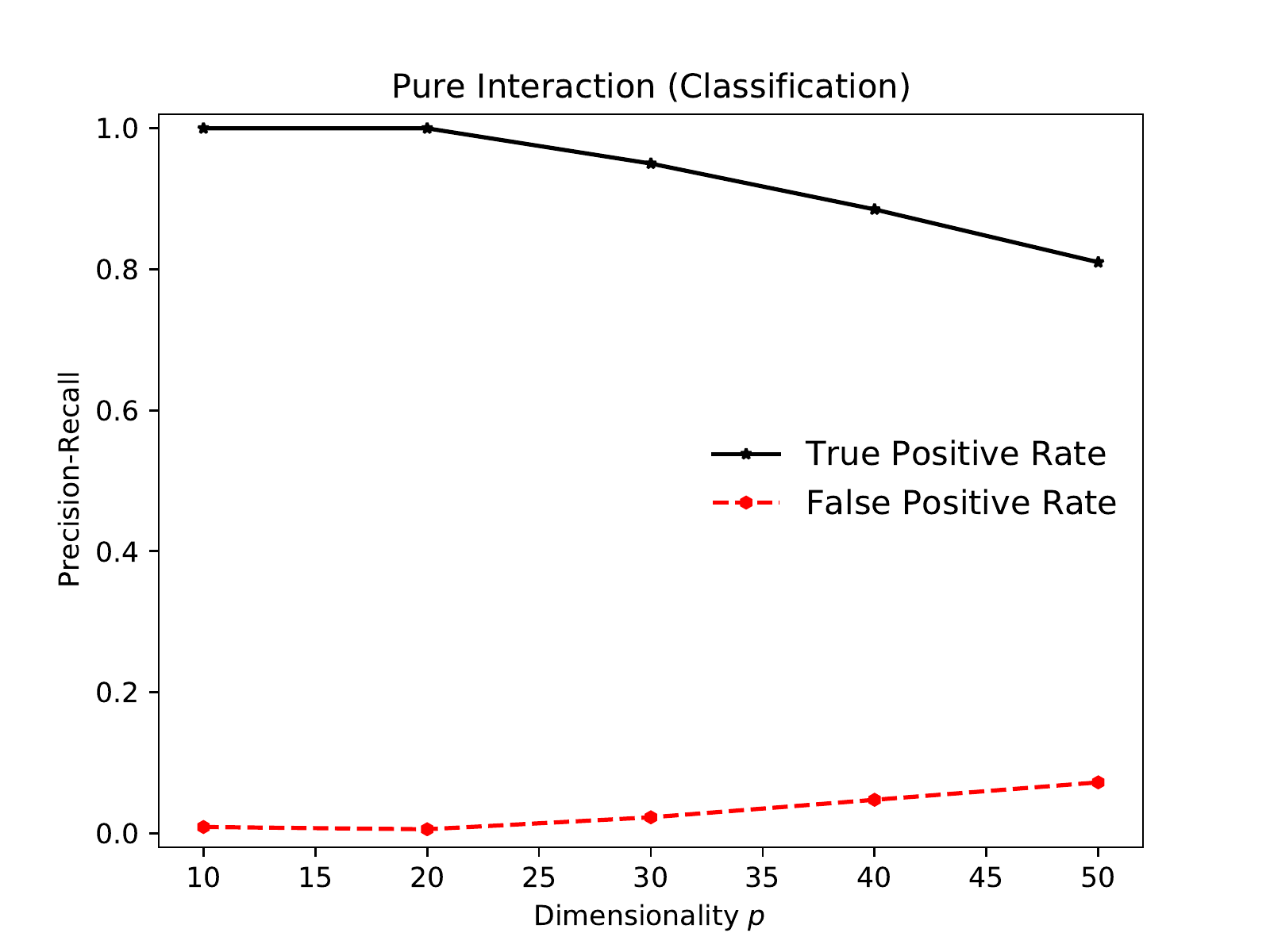}
\par\end{centering}
\caption{\small True positive rate and false positive rate curve under the pure interaction model against different choice of dimensionality $p$.
Here the sample size $n = 200$ is fixed.  
}
\label{fig:pure-interaction}
\end{figure}

\section{Discussion}
\label{sec:discussion}
We have described a new sparsity-inducing mechanism based on minimization over a 
family of kernels.  The mechanism does not involve classical explicit regularization 
(such as $\ell_1$ penalization, early stopping etc.), yet we find that in many cases 
the mechanism is able to generate solutions that are exactly sparse in finite samples. 
We coin this phenomenon \emph{self-penalization}.  We develop a theory that provides
sufficient conditions under which the phenomenon arises, showing in particular that 
in cases where there are two complementary subsets of variables that are weakly dependent, 
the self-penalization is largely attributed to the self-penalizing property of the 
objective---a property that is exhibited by (nonlinear) kernel-based objectives.

The sparsity-inducing mechanism in our work is closely related to the notion of 
 \emph{implicit regularization} in the literature, which---broadly defined---means that 
the optimization algorithms are biased towards certain type of structured solutions. In particular, 
our results illustrate that the gradient-type algorithms (e.g., gradient descent), when applied to
the properly designed kernel-based objectives, are automatically biased towards solutions that
are sparse.

There are nonetheless several important distinctions between our work and existing 
work on the implicit regularization.  In particular, most existent work on implicit 
regularization focuses on exploring and explaining the connection between the 
$\ell_2$ regularization and the gradient descent~\citep{BuhlmannYu03, FriedmanPo04, 
YaoRoCa07, GunasekarLeSoSr18, ZhangBeHaReVi21}---in this case the focus is not on
generation of sparse solutions. Indeed, there is a paucity of work in which implicit 
regularization is associated with solutions that are \emph{exactly} sparse.  The
main exception is in the boosting literature, where it is shown that the coordinate 
descent tends towards sparse solutions~\citep{FriedmanHaTi01, ZhangYu05, Telgarsky13}.
For instance, least squares gradient boosting, which is steepest coordinate descent,
generates a sparse regularization path~\citep{FeundGrMa17}.
 
Our work also has a natural connection to the literature on sparse 
recovery~\citep{BuhlmannVa11,HastieTiWa19}; indeed, it can be viewed as
introducing a new sparsity-inducing mechanism that can be employed to
aid in the design of algorithms that are consistent in recovery of features. 

Finally, we note some natural extensions of the results in the paper that are
worthy of further exploration.  First, one can consider discrete projected 
gradient descent in place of continuous gradient flow.  Our results go through 
as long as the stepsize $\alpha$ satisfies $\alpha \le 1/L$ where $L$ is the 
Lipschitz constant of the gradient of the kernel feature selection objective.  
Second, it is worth investigating various relaxations of the use of full gradient 
descent. These would include replacing full gradient descent by other deterministic 
gradient-type methods such as coordinate descent or accelerated gradient descent, 
as well as stochastic methods.

\section*{Acknowledgements}
The authors would like to thank Yichen Zhang for helpful discussions. The work of Michael I. Jordan and Feng Ruan is partially supported 
by NSF Grant IIS-1901252. 

\bibliographystyle{apalike}
\bibliography{bib}

\appendix
\newpage

\section{Preliminaries on Kernel Ridge Regression}
This section presents some basic properties of kernel ridge regression (KRR).
Several of our technical results are drawn from~\cite{JordanLiRu21} and we
refer the reader to that paper for further exposition.

We consider the following family of kernel ridge regression problems. 
Fix a given Hilbert space $\H$ whose kernel is given by 
$k(x,x') = h(\norm{x-x'}_q^q)$ where $h \in \mathcal{C}^\infty[0, \infty)$ is a completely 
monotone function and $q \in \{1, 2\}$. Consider a problem indexed by $\beta\in \R_+^p$
\begin{equation}
\label{eqn:KRR-beta}
\begin{split}
\KRR(\beta):~
&\minimize_{f \in \H}~ \half \E_Q[(Y - f(\beta^{1/q} \odot X))^2] + \frac{\lambda}{2} \norm{f}_{\H}^2.
\end{split}
\end{equation}

\paragraph{Notation} 

Let $f_{\beta; \lambda}$ denote the minimizer of the KRR problem. Let $r_{\beta; \lambda}$, $z_{\beta; \lambda}$ denote the residual function 
$r_{\beta; \lambda}(x, y) = y - f_{\beta; \lambda}(x)$ and 
$z_{\beta; \lambda}(x, y) = y -f_{\beta; \lambda}(\beta^{1/q} \odot x)$\footnote{Note: $r_{\beta; \lambda}(x, y) = y- f_{\beta; \lambda}(x)$ is the 
only residual function that appears in the cited paper~\cite{JordanLiRu21}.}.

\subsection{General properties ($q=1$ or $q=2$)}

This section provides general properties of KRR when $q= 1$ or $q = 2$. 

\paragraph{Main Results}

Our first set of results gives the characterization of the minimum $f_{\beta; \lambda}$.
\begin{lemma}[KKT Characterization]
\label{lemma:KKT-characterization-general}
The following identity holds for any function $g \in \H$: 
\begin{equation}
\label{eqn:KKT-characterization-general}
	\E_Q\left[z_{\beta; \lambda}(X; Y) g(\beta^{1/q} \odot X)\right] = \lambda \langle f_{\beta; \lambda}, g\rangle_{\H}. 
\end{equation}
\end{lemma}
\begin{proof}
This is Proposition 9 in~\cite{JordanLiRu21}.
\end{proof}

As a consequence of the KKT condition, we derive the following important connections between the regression 
function $f_{\beta; \lambda}$ and the residual function $z_{\beta; \lambda}$.  

\newcommand{\ball}{\mathsf{B}}

\begin{lemma}[Balance between minimum $f_{\beta; \lambda}$ and residual $z_{\beta; \lambda}$]
\label{lemma:balance-between-min-residual}
The following identity \nolinebreak holds: 
\begin{equation}
	\E_Q\left[z_{\beta; \lambda}(X; Y) z_{\beta; \lambda}(X'; Y') h(\norm{X-X'}_{q, \beta}^q) \right]  
		= \lambda^2 \norm{f_{\beta; \lambda}}_{\H}^2.
\end{equation}
\end{lemma}

\begin{proof}
By Lemma~\ref{lemma:KKT-characterization-general}, the following identity holds for all function $g \in \H$: 
\begin{equation}
\label{eqn:copy-paste-KKT-characterization}
	\E_Q\left[z_{\beta; \lambda}(X; Y) g(\beta^{1/q} \odot X)\right] = \lambda \langle f_{\beta; \lambda}, g\rangle_{\H}. 
\end{equation}
According to the reproducing property of the kernel function $k$, we can rewrite the left-hand side of 
Eq.~\eqref{eqn:copy-paste-KKT-characterization} and obtain the following 
identity that holds for all function $g \in \H$: 
\begin{equation}
\label{eqn:copy-paste-KKT-characterization-one-step}
	\langle \E_Q\left[z_{\beta; \lambda}(X; Y) k(\beta^{1/q} \odot X, \cdot )\right], g \rangle_{\H}
		=   \lambda \langle f_{\beta; \lambda}, g\rangle_{\H}. 
\end{equation}
Let $\ball_{\H}(1)$ denote the unit ball of $\H$ (i.e., $\ball_\H(1) = \norm{g}_{\H} \le 1$). 
Now we take supremum over all possible $g \in \ball_\H(1)$ in the unit ball on both sides of the 
Eq.~\eqref{eqn:copy-paste-KKT-characterization-one-step}. This gives us
\begin{equation}
\label{eqn:copy-paste-KKT-characterization-one-more-step}
	\norm{\E_Q\left[z_{\beta; \lambda}(X; Y) k(\beta^{1/q} \odot X, \cdot )\right]}_{\H} =  \lambda \norm{f_{\beta; \lambda}}_{\H}.
\end{equation}
To finalize the proof of Lemma~\ref{lemma:balance-between-min-residual}, we square both sides of 
Eq.~\eqref{eqn:copy-paste-KKT-characterization-one-more-step}. Note then 
\begin{equation}
\label{eqn:copy-paste-KKT-characterization-last-step}
\begin{split}
	&\norm{\E_Q\left[z_{\beta; \lambda}(X; Y) k(\beta^{1/q} \odot X, \cdot )\right]}_{\H}^2  \\
		&~~~\stackrel{(i)}{=} \langle \E_Q\left[z_{\beta; \lambda}(X; Y) k(\beta^{1/q} \odot X, \cdot )\right], 
			\E_Q\left[z_{\beta; \lambda}(X'; Y') k(\beta^{1/q} \odot X', \cdot )\right]\rangle_\H \\
		&~~~\stackrel{(ii)}{=} 
			\E_Q \left[\langle z_{\beta; \lambda}(X; Y) k(\beta^{1/q} \odot X, \cdot ),  z_{\beta; \lambda}(X'; Y') k(\beta^{1/q} \odot X', \cdot )\rangle_{\H}\right] \\
		&~~~\stackrel{(iii)}{=} 
			\E_Q\left[z_{\beta; \lambda}(X; Y) z_{\beta; \lambda}(X'; Y') k(\beta^{1/q}\odot X, \beta^{1/q} \odot X') \right].
\end{split}
\end{equation}
In the first step (i), we introduce independent variables $X' \sim X \sim Q$, and in the second step (ii), we 
use the fact that the inner product $\langle \cdot, \cdot\rangle_{\H}$ is bilinear, and in the last step (iii), we use
the reproducing property of the kernel function $k$. 
\end{proof}

\begin{lemma}[Relations of Different Norms]
\label{lemma:relation-of-different-metrics}
Any function $f$ satisfies $\norm{f}_\infty \le |h(0)|^{1/2} \norm{f}_{\H}$. 
\end{lemma}
\begin{proof}
The evaluation functional $\mathcal{F}_x: f \mapsto f(x)$ is a bounded operator for all $x$ with the 
operator norm uniformly bounded by $|h(0)|^{1/2}$ (as $\H$ has the kernel 
$h(\norm{x-x'}_q^q)$). 
\end{proof}

\begin{lemma}[Lower Bound of $ \norm{f_{\beta; \lambda}}_{\H}^2$]
\label{lemma:lower-bound-f-beta-norm}
The following holds for all $\beta$: 
\begin{equation}
	 \norm{f_{\beta; \lambda}}_{\H}^2 \ge \frac{1}{|h(0)| M_Y^2} (\FKRR_\lambda(0; \Q) - \FKRR_\lambda(\beta; \Q))_+^2.
\end{equation}
\end{lemma}

\begin{proof}
Note that $f_{0; \lambda} = 0$ since $\E_Q[Y] = 0$ by assumption. As its consequence, we obtain
\begin{equation*}
	\FKRR_\lambda(0; \Q) = \half \E_Q[Y^2].  
\end{equation*}
As a result, we obtain the following elementary upper bound on $\FKRR_\lambda(0) - \FKRR_\lambda(\beta)$: 
\begin{equation*}
\begin{split}
	\FKRR_\lambda(0; \Q) - \FKRR_\lambda(\beta; \Q) 
		&= \E_Q[Y f_{\beta; \lambda}(\beta^{1/q} \odot X)] - 
			\half \E_Q[f_{\beta; \lambda}(\beta^{1/q} \odot X)^2] - \frac{\lambda}{2} \norm{f_{\beta; \lambda}}_{\H}^2 \\
		&\le \E_Q[Y f_{\beta; \lambda}(\beta^{1/q} \odot X)].
\end{split}
\end{equation*}
Using the Cauchy-Schwartz inequality and Lemma~\ref{lemma:relation-of-different-metrics}, we immediately obtain 
\begin{equation*}
	( \FKRR_\lambda(0; \Q) - \FKRR_\lambda(\beta; \Q) )_+^2 \le \E_Q[Y]^2 \cdot \E_Q[f_{\beta; \lambda}(\beta^{1/q} \odot X)^2] \le |h(0)|M_Y^2  \cdot 
		\norm{f_{\beta; \lambda}}_{\H}^2.
\end{equation*}
\end{proof}

\begin{lemma}[Representation of $\FKRR_\lambda(\beta; \Q)$]
\label{lemma:representation-of-KRR}
The following holds for all $\beta$: 
\begin{equation*}
	\FKRR_\lambda(\beta) = \half \E[z_{\beta;\lambda} (X; Y)Y].
\end{equation*}
\end{lemma}

\begin{proof}
Our starting point is the following identity. By Lemma~\ref{lemma:KKT-characterization-general}, we have 
\begin{equation}
\begin{split}
	\lambda \norm{f_{\beta; \lambda}}_{\H}^2 &= 
		\E_Q\left[z_{\beta; \lambda}(X; Y) f_{\beta; \lambda}(\beta^{1/q}\odot X)  \right] 
\end{split}
\end{equation}
Note $f_{\beta; \lambda}(x, y) = y - z_{\beta; \lambda}(x; y)$ by definition. 
As a result, we obtain the identity 
\begin{equation*}
\begin{split}
	\FKRR_\lambda(\beta; \Q) &= \half \E_Q[z_{\beta; \lambda}(X, Y)^2] + \frac{\lambda}{2} \norm{f_{\beta; \lambda}}_{\H}^2
		= \half \E_Q\left[z_{\beta; \lambda}(X; Y) Y \right].
\end{split}
\end{equation*}
\end{proof}

\begin{lemma}[Representation of $\grad \FKRR_\lambda(\beta; \Q)$]
\label{lemma:representation-of-KRR-gradient}
The following holds for all $\beta$ and $l \in [p]$,
\begin{equation*}
\partial_{\beta_l} \FKRR_\lambda(\beta; \Q) 
	= -\frac{1}{\lambda} \E\left[z_{\beta; \lambda}(X; Y) z_{\beta; \lambda}(X', Y') h^\prime(\norm{X-X'}_{q, \beta}^q) |X_l - X_l'|^q\right]. 
\end{equation*}
\end{lemma}
\begin{proof}
This is Proposition 4 in~\cite{JordanLiRu21}. 
\end{proof}

\begin{lemma}[Lipschitzness of $\beta \mapsto \grad \FKRR_\lambda(\beta; \Q)$]
\label{lemma:Lipschitz-of-gradient}
The following bound holds for all $\beta, \beta'$: 
\begin{equation*}
	\norm{\grad \FKRR_\lambda(\beta; \Q)- \grad \FKRR_\lambda(\beta'; \Q)}_\infty
		\le \frac{1}{\lambda^2} \cdot (2M_X)^{2q} \cdot M_Y^2 \cdot (2|h^\prime(0)|^2+\lambda|h^{\prime\prime}(0)|) \cdot \norm{\beta-\beta'}_1.
\end{equation*}
\end{lemma}

\begin{proof}
This is Lemma C.7 in~\cite{JordanLiRu21}. 
\end{proof}

\begin{lemma}[Uniform Boundedness of $\beta \mapsto \grad \FKRR_\lambda(\beta; \Q)$] 
\label{lemma:uniform-boundedness-of-gradient}
The following bound holds: 
\begin{equation*}
	\sup_{\beta \in \R_+^p}\norm{\grad \FKRR_\lambda(\beta; \Q)}_\infty
		\le \frac{1}{\lambda} \cdot |h^\prime(0)| \cdot (2M_X)^q \cdot M_Y^2. 
\end{equation*}
\end{lemma}

\begin{proof}
This is Proposition 15 in~\cite{JordanLiRu21}.
\end{proof}

\begin{lemma}[Lipschitzness of $\beta\mapsto \FKRR_\lambda(\beta)$]
\label{lemma:Lipschitzness-of-objective}
The following bound holds for all $\beta, \beta'$: 
\begin{equation}
\label{eqn:Lipschitzness-of-f-norm-square}	
|\FKRR_\lambda(\beta; \Q) - \FKRR_\lambda(\beta'; \Q)|
		\le \frac{1}{\lambda} \cdot |h^\prime(0)| \cdot (2M_X)^{q} \cdot M_Y^2 \cdot \norm{\beta-\beta'}_1.
\end{equation}

\end{lemma}
\begin{proof}
This is a consequence of Lemma~\ref{lemma:uniform-boundedness-of-gradient} and Taylor's intermediate theorem. 
\end{proof}

\begin{lemma}[Lipschitzness of $\beta\mapsto (\FKRR_\lambda(0; \Q)-\FKRR_\lambda(\beta; \Q))_+^2$]
\label{lemma:Lipschitzness-of-objective-square}
It holds for all $\beta, \beta'$: 
\begin{equation*}
|(\FKRR_\lambda(0; \Q)-\FKRR_\lambda(\beta; \Q))_+^2 - (\FKRR_\lambda(0; \Q)-\FKRR_\lambda(\beta';\Q))_+^2|
		\le \frac{1}{2\lambda} \cdot |h^\prime(0)| \cdot (2M_X)^{q} \cdot M_Y^4 \cdot \norm{\beta-\beta'}_1.
\end{equation*}
\end{lemma}
\begin{proof}
This is a consequence of Lemma~\ref{lemma:Lipschitzness-of-objective}. Note that 
$\FKRR_\lambda(0; \Q) \le M_Y^2$.
\end{proof}

\subsection{Special properties $q=1$}

\subsubsection{Generic Bounds on the Gradient $\grad \FKRR_\lambda(\beta)$}

Introduce the auxiliary function $\circup{\grad} \FKRR_\lambda(\beta)$ coordinate-wisely defined \nolinebreak by
\begin{equation}
\label{eqn:auxiliary-gradient}
(\circup{\grad} \FKRR_\lambda(\beta))_j = -\frac{1}{\lambda} \cdot 
	\E[z_{\beta; \lambda}(X; Y) z_{\beta; \lambda}(X'; Y') h(\norm{X-X'}_{1, \beta}+|X_j-X_j'|)]
\end{equation}
We abuse notation and write $\circup{\partial_{\beta_j}} \FKRR_\lambda(\beta) \defeq (\circup{\grad} \FKRR_\lambda(\beta))_j$.
The following  Lemma~\ref{lemma:gradient-bound-useful} is a direct consequence of
Lemma C.9 and Lemma C.10 in~\cite{JordanLiRu21}.

\begin{lemma}[Generic Bound]
\label{lemma:gradient-bound-useful}
Let $q=1$. The following holds for all $\beta \in \R_+^p$ and \nolinebreak $l \in [p]$: 
\begin{equation*}
	\partial_{\beta_l} \FKRR_\lambda(\beta; \Q)  \le -\frac{1}{2\lambda}\cdot 
		 \left(-\lambda \cdot \circup{\partial_{\beta_l}} \FKRR_\lambda(\beta; \Q) - C \lambda^{1/2}(1+\lambda^{1/2}) \right).
\end{equation*}
Here the constant $C > 0$ depends only on $M_X, M_Y, M_\mu$ (and does not depend on $\beta$).
\end{lemma}

\subsubsection{Specific Bounds on the Gradient $\grad \FKRR_\lambda(\beta)$}

\begin{lemma}[Main Effect]
\label{lemma:gradient-bound-main-effect}
Assume $X_l$ is a main effect signal: $\Var_Q(\E_Q[Y | X_l]) \neq 0$. Fix $\lambda_0 > 0$. 
Then there exist $c, C > 0$ such that the following bound holds for all $\lambda \le \lambda_0$: 
\begin{equation}
\partial_{\beta_l}  \FKRR_\lambda(\beta; \Q)\mid_{\beta = 0} \le -\frac{1}{\lambda} \cdot (c - C \lambda^{1/2}).
\end{equation}
The constants $c, C > 0$ depend only on $\Q, \mu, \lambda_0 > 0$ (and does not rely on $\lambda$). 
\end{lemma}

\begin{proof}
By Lemma~\ref{lemma:gradient-bound-useful}, it suffices to prove that, for some constant $c > 0$, 
\begin{equation*}
\lambda \circup{\partial_{\beta_l}} \FKRR_\lambda(\beta) =  \E_Q[YY' h(|X_l-X_l'|)] \ge c.
\end{equation*}
As $X_l$ is a main effect, we know that $U(X_l) \defeq \E[Y|X_l] \neq 0$. Consequently, we obtain 
\begin{equation*}
	\E_Q[YY' h(|X_l - X_l'|)] = \E_Q[U(X_l) U(X_l') h(|X_l - X_l'|)] > 0
\end{equation*}
where the last inequality uses the fact that $h$ is positive definite. 
\end{proof}

\begin{lemma}[Conditional Main Effect]
\label{lemma:gradient-bound-conditional-main-effect}
Assume $X_l$ is a conditional main effect signal with respect to $X_T$: $\E_Q[Y|X_T] \neq \E_Q[Y|X_{T \cup \{l\}}]$. 
Fix $\lambda_0, \tau > 0$. Let $\beta^{0; T}$ be the vector where $\beta^{0; T}_T = \tau \mathbf{1}_T$ and $\beta^{0; T}_{T^c} = 0$. 
Then there exist $c, C > 0$ such that the bound below holds when $\lambda \le \lambda_0$: 
\begin{equation}
\partial_{\beta_l}  \FKRR_\lambda(\beta; \Q)\mid_{\beta = \beta^{0; T}} \le -\frac{1}{\lambda} \cdot (c - C \lambda^{1/2}),
\end{equation}
where the constants $c$ and $C > 0$ depend only on $\Q, \mu, \lambda_0, \tau > 0$ 
(and do not depend on $\lambda$). 
\end{lemma}

\begin{proof}
For notational simplicity, we write $Z = z_{\beta; \lambda}(X; Y)\mid_{\beta = \beta^{0; T}}$ and write
$\circup{\beta}$ to be the vector which has the same coordinates as $\beta^{0; T}$ except at $l$ where $\circup{\beta}_l = 1$.
By Lemma~\ref{lemma:gradient-bound-useful}, it suffices to prove that, for some constant $c > 0$, 
\begin{equation*}
\begin{split}
&\lambda \circup{\partial_{\beta_l}} \FKRR_\lambda(\beta)\mid_{\beta = \beta^{0; T}} = \E_Q[ZZ'h(\normsmall{X-X'}_{1, \circup{\beta}})] \ge c.
\end{split}
\end{equation*}
Now, $X_l$ is a conditional main effect with respect to $X_T$ by assumption. Hence, we obtain 
$U(X_{T \cup \{l\}}) \defeq \E_Q[Z \mid X_{T \cup \{l\}}] \neq 0$. Consequently, we obtain the inequality
\begin{equation*}
 \E_Q[ZZ'h(\norm{X-X'}_{1, \circup{\beta}})] =   \E_Q[U(X_{T \cup \{l\}})U(X_{T \cup \{l\}}')h(\normsmall{X-X'}_{1, \circup{\beta}})] > 0,
\end{equation*}
where the last inequality uses the fact that $h$ is positive definite. 
\end{proof}

\newcommand{\mapssto}{\rightrightarrows}
\newcommand{\domain}{{\rm dom}}
\newcommand{\nc}{\mathcal{N}}

\newcommand{\bigindic}{\mathbb{I}}

\section{Basics on Differential Inclusions}
\label{sec:preliminaries-on-differential-inclusions}

This section reviews basic results on differential inclusions, 
and in particular the gradient inclusions. Our definitions and 
notation follow closely the standard references~\cite{Aubin12}.

\begin{definition}[Differential Inclusion~\citep{Aubin12}]
\label{definition:differential-inclusion}
The \emph{differential inclusion} associated with the set-valued mapping $G: \R^p \mapssto \R^p$ from 
the point $x_0 \in \R^p$ denoted by
\begin{equation*}
	\dot{x} \in G(x),~~x(0) = x_0
\end{equation*}
is defined by any absolutely continuous function $x: U \mapsto \R^p$ satisfying for all $t \in U$
\begin{equation*}
	x(t) = x_0 + \int_0^t v(s)ds,
\end{equation*}
where $v: \R_+ \to \R^p$ is some measurable function satisfying $v(s) \in G(x(s))$ for all $s \in U$. 
\end{definition}

Our focus is on the gradient inclusion. We restrict the objective functions to be weakly convex 
functions, whose definition is given below.

\begin{definition}
A function $f: \R^p \to \R \cup \{\infty\}$ is called \emph{weakly convex}  if there exists $\lambda \ge 0$ 
such that the function $f(x) + \frac{\lambda}{2} \ltwo{x}^2$ is convex. 
\end{definition}

All the objective functions considered in this paper are weakly convex (Example~\ref{example:range-of-weakly-convex-functions}). 

\begin{example}
\label{example:range-of-weakly-convex-functions}
A convex function $f$ is a weakly convex function. A function $f$ with domain convex on which the gradient 
is Lipschitz is a weakly convex function.
\end{example}

Let $f$ be a weakly convex function and $X$ be a closed convex set. Consider 
\begin{equation*}
	\minimize_{x \in X} f(x). 
\end{equation*}
The associated gradient inclusion is given by
\begin{equation}
\label{eqn:general-differential-inclusion}
	\dot{x}(t) \in - \partial f(x) - \nc_X(x) ~~\text{and}~~x(0) = x_0.
\end{equation}
Theorem~\ref{thm:main-result-grad-inclusion} establishes the existence, uniqueness, and monotonicity 
of the solution of the gradient inclusion~\eqref{eqn:general-differential-inclusion}. The proof is standard~\citep[see, e.g.,][Theorem 3.15]{DuchiRu18}.

\setcounter{assumption}{2}
\renewcommand\theassumption{\Alph{assumption}}

\begin{theorem}[Existence, Uniqueness, Monotonicity of the Solution of the Gradient Inclusion]
\label{thm:main-result-grad-inclusion}
Assume the function $f$ is closed and weakly convex. Let $X$ be a closed convex set with $X \subseteq \domain(f)$. Assume 
$f(x) + \bigindic_X(x)$ is \emph{coercive}: $\lim_{x \to \infty} f(x) + \bigindic_X(x) = \infty$. 
Then for any $x_0 \in X$, there is a unique solution $x$ of the differential inclusion~\eqref{eqn:general-differential-inclusion}
defined on $\R_+$. The solution satisfies the properties $(i)$ $x(t) \in X$, $(ii)$ $x(t)$ is Lipschitz in $t$, and $(iii)$
\begin{equation*}
	f(x(t))  \le f(x(0)) -  \int_0^t \ltwo{\subgrad \opt(x(s))}^2 ds~~\text{for all $t\in \R_+$}.
\end{equation*}
We define $\subgrad \opt(x): X \mapsto \R^p$ is defined by $\subgrad \opt(x) = \argmin\{\ltwo{g}: g\in \partial f(x) + \nc_X(x)\}$.
\end{theorem}

The differential inclusion has benign stability guarantees when it is a Lipschitz differential inclusion. 
We say that the mapping $G: \R^p \mapssto \R^p$ is $L$-Lipschitz if $\dist(G(x), G(y)) \le L\ltwo{x-y}$ for all $x, y\in \R^p$,
where $\dist(G_1, G_2) \defeq \inf_{g_1 \in G_1, g_2 \in G_2} \ltwo{g_1 - g_2}$. 

\vspace{0.3cm}

\begin{example}
Let $X$ be a closed convex set. Let $f$ be a function which is $L$-Lipschitz on $X$ and which satisfies
$f(x) = 0$ for all $x \not\in X$. Consider the set-valued mapping $G(x) = f(x) + \nc_X(x)$ which is defined 
on $\R^p$. Then $G: \R^p \mapssto \R^p$ is $L$-Lipschitz. 
\end{example}

The following result concerns the stability of Lipschitz differential inclusion. 
We say that the mapping $G$ is \emph{outer-semicontinuous} if for any sequence $x_n \to x \in \domain(G)$, we have 
the inclusion $\limsup_{n \to \infty} G(x_n) \subseteq G(x)$. Here the limit supremum for a sequence
of sets $A_n \subseteq\R^p$, is defined by $\limsup_n A_n = \{y: \exists y_{n_k} \in A_{n_k},  y_{n_k} \to y~\text{as}~k \to \infty\}$.
Note that the gradient mapping $G(x) = -\partial f(x) - \normal_X(x)$ that appears in the right-hand side of Eq.~\eqref{eqn:general-differential-inclusion}
is outer-semicontinuous~\citep[see, e.g.,][Lemma 3.6]{DuchiRu18}. 
%


\begin{theorem}[Gr\"{o}nwall inequality of Lipschitz Differential Inclusion~\citep{Aubin12})]
\label{thm:perturbation-of-gradient-inclusion}
Let $G: \R^p \mapssto \R^p$ and $\wtilde{G}: \R^p \mapssto \R^p$ be non-empty, closed convex-valued, 
and outer-semicontinuous. Consider the two differential inclusions $x, \wtilde{x}$
associated with the set-valued mapping $G$, $\wtilde{G}$: 
\begin{align}
	\dot{x} = G(x),~~ x(0) = x_0. \label{eqn:before-perturbation}\\
	\dot{\wtilde{x}} = \wtilde{G}(\wtilde{x}),~~ \wtilde{x}(0) = x_0.
		 \label{eqn:after-perturbation}
\end{align}
Assume (i) $G$ is $L$-Lipschitz set-valued mapping, (ii) $\sup_{x\in \R^p} \dist(G(x), \wtilde{G}(x)) = \Delta < \infty$ and (iii) 
there exists a solution $x$ for the inclusion~\eqref{eqn:before-perturbation} on the interval $[0, T]$. Then 
there exists a solution $\wtilde{x}$ for the inclusion~\eqref{eqn:after-perturbation} on the interval $[0, T]$ which satisfies 
\begin{equation}
\label{eqn:general-gronwall-inequality}
	\sup_{t \in [0, T]} \ltwo{\wtilde{x}(t) - x(t)} \le \frac{\Delta}{L} \cdot e^{LT}. 
\end{equation}
\end{theorem}

%

\section{Deferred Technical Proofs}

 \subsection{Details of the proof of Theorem~\ref{theorem:kernel-ridge-regression}}
 \label{sec:proof-theorem-kernel-ridge-regression-appendix}
 This section continues the proof of Theorem~\ref{theorem:kernel-ridge-regression}.
 
 We start by considering the case where $\beta_{S^c} = 0$, and show that the desired gradient lower
 bound~\eqref{eqn:self-penalization-kernel-ridge-regression}  in the statement of Theorem~\ref{theorem:kernel-ridge-regression}
 holds for all $\beta$ such that $\beta_{S^c} = 0$. Fix a noise variable $j \not\in S$. Substituting 
 Eq.~\eqref{eqn:function-h-representation} into Eq.~\eqref{eqn:gradient-representation-KRR} yields for 
 $j \not\in S$
 \begin{equation}
 \label{eqn:substitute-one-KRR}
 	\partial_{\beta_j} \FKRR_\lambda(\beta; \Q) = 
 		\frac{1}{\lambda} \int_0^\infty \E_Q\left[z_{\beta; \lambda}(X; Y) z_{\beta; \lambda}(X'; Y') t e^{-t \norm{X-X'}_{q, \beta}^q} |X_j - X_j'|^q\right] \mu(dt).
 \end{equation}
 We analyze the integrand on the right-hand side. 
 As $\beta_{S^c} = 0$, $\E[z_{\beta; \lambda}(X; Y)|X] = f^*(X_S)- f_{\beta; \lambda}(\beta^{1/q} \odot X)$ is a function only of the signal variable $X_S$, 
 and similarly, $e^{-t \norm{X-X'}_{q, \beta}^q}$ is a function only of $X_S$, too. 
 As $X_S \perp X_{S^c}$, we can decompose the integrand on the right-hand side into: 
 \begin{equation}
 \label{eqn:substitute-one-step-one-KRR}
 \begin{split}
 	&\E_Q\left[z_{\beta; \lambda}(X; Y) z_{\beta; \lambda}(X'; Y') t e^{-t \norm{X-X'}_{q, \beta}^q} |X_j - X_j'|^q\right] \\
 	& = \E_Q\left[z_{\beta; \lambda}(X; Y) z_{\beta; \lambda}(X'; Y') t e^{-t \norm{X-X'}_{q, \beta}^q} \right]  \cdot \E_Q\left[|X_j - X_j'|^q\right].
 \end{split} 
 \end{equation}
 Substitute Eq.~\eqref{eqn:substitute-one-step-one-KRR} back into Eq.~\eqref{eqn:substitute-one-KRR}. We obtain 
 for all $\beta$ such that $\beta_{S^c} = 0$
 \begin{equation}
 \label{eqn:substitute-one-step-two-KRR}
 \begin{split}
 	&\partial_{\beta_j} \FKRR_\lambda(\beta; \Q)  \\
 	&~~~~~~= \frac{1}{\lambda}
 		 \int_0^\infty  \E_Q\left[z_{\beta; \lambda}(X; Y) z_{\beta; \lambda}(X'; Y') t e^{-t \norm{X-X'}_{q, \beta}^q} \right]  \mu(dt) 
 		\cdot \E_Q\left[|X_j - X_j'|^q\right].
 \end{split}
 \end{equation}
 Since the following integral representation of the kernel function holds 
 \begin{equation}
 \label{eqn:integral-representation-of-kernel-KRR}
 	h(\norm{X-X'}_{q, \beta}^q) = \int_0^\infty  e^{-t \norm{X-X'}_{q, \beta}^q} \mu(dt), 
 \end{equation}
 we can use the assumption 
 that $\supp(\mu) \subseteq [m_\mu, M_\mu]$ where $0 < m_\mu < M_\mu < \infty$ and apply the Markov's 
 inequality to Eq.~\eqref{eqn:substitute-one-step-two-KRR} to conclude that for all $\beta$ such that $\beta_{S^c} = 0$
 \begin{equation}
 \label{eqn:substitute-one-step-three-KRR}
 \begin{split}
 	&\partial_{\beta_j} \FKRR_\lambda(\beta; \Q) \\ 
 		&~~~~~~\ge \frac{1}{\lambda} \cdot m_\mu \cdot \E_Q\left[z_{\beta; \lambda}(X; Y) z_{\beta; \lambda}(X'; Y') h(\norm{X-X'}_{q, \beta}^q) \right]  
 		\cdot \E_Q\left[|X_j - X_j'|^q\right].
 \end{split}
 \end{equation}
 To finalize the proof of the desired gradient lower bound~\eqref{eqn:self-penalization-kernel-ridge-regression} 
 for all the $\beta$ where $\beta_{S^c} = 0$, we invoke Lemma~\ref{lemma:balance-between-min-residual} 
 and Lemma~\ref{lemma:lower-bound-f-beta-norm}. Lemma~\ref{lemma:balance-between-min-residual} itself is of crucial 
 importance because it builds a non-trivial identity that connects the regressor $f_{\beta; \lambda}$ and the 
 residual $z_{\beta; \lambda}$: 
 \begin{equation}
 \label{eqn:r-beta-f-beta-relation}
 	 \E_Q\left[z_{\beta; \lambda}(X; Y) z_{\beta; \lambda}(X'; Y') h(\norm{X-X'}_{q, \beta}^q) \right]  = \lambda^2 \norm{f_{\beta; \lambda}}_{\H}^2,
 \end{equation}
 where above $f_{\beta, \lambda}$ denotes the minimum of the kernel ridge regression indexed by $\beta$, i.e., 
 \begin{equation*}
 	f_{\beta, \lambda} = \argmin_{f \in \H}~ \half \E_Q \left[(Y - f(\beta^{1/q} \odot X))^2\right] + \frac{\lambda}{2} \norm{f}_{\H}^2.
 \end{equation*}
 Lemma~\ref{lemma:lower-bound-f-beta-norm} further lower bounds the Hilbert norm of the fitted regression function 
 $f_{\beta; \lambda}$ in terms of the deviation of the objective value from $\FKRR_\lambda(\beta; \Q)$ to $\FKRR_\lambda(0; \Q)$: 
 \begin{equation}
 \label{eqn:f-beta-lambda-lower-bound}
 	\norm{f_{\beta; \lambda}}_{\H}^2 \ge \frac{1}{|h(0)|M_Y^2} \cdot \left(\FKRR_\lambda(0; \Q) - \FKRR_\lambda(\beta; \Q)\right)_+^2. 
 \end{equation}
 Substitute equations~\eqref{eqn:r-beta-f-beta-relation}---\eqref{eqn:f-beta-lambda-lower-bound} into inequality~\eqref{eqn:substitute-one-step-three-KRR}.
 This shows for all $\beta$ where $\beta_{S^c} = 0$,
 \begin{equation}
 \label{eqn:final-lower-bound-special}
 \begin{split}
 	\partial_{\beta_j} \FKRR_\lambda(\beta; \Q)
 		&\ge c_j \cdot \lambda \left(\FKRR_\lambda(0; \Q) - \FKRR_\lambda(\beta; \Q)\right)_+^2, \\
 	~\text{where}~ c_j &=  \frac{m_\mu}{|h(0)|M_Y^2}\cdot \E_Q\left[|X_j - X_j'|^q\right].
 \end{split}
 \end{equation}
 This gives the desired gradient lower bound for the $\beta$ when $\beta_{S^c} = 0$. 

 Next, we extend the gradient lower bound to the general case where we do not put any restriction on $\beta$. 
 The key mathematical result that allows this is Lemma~\ref{lemma:Lipschitz-of-gradient} and Lemma~\ref{lemma:Lipschitzness-of-objective-square} 
 where it is shown that the following two mappings are Lipschitz on $\R_+^p$ 
 \begin{equation}
 \label{eqn:two-Lipschitz-mappings}
 	\beta \mapsto \partial_{\beta_j} \FKRR_\lambda(\beta; \Q),~~\text{and}~~\beta \mapsto \lambda(\FKRR_\lambda(0; \Q) -\FKRR_\lambda(\beta; \Q))_+^2.
 \end{equation}
 with the Lipschitz constants bounded by $C_1(1+\lambda)/\lambda^2$ and $C_2$, where $C_1, C_2$ are constants 
 that depend only on $M_X, M_Y, \mu$.  As a direct consequence of this fact, 
 the general gradient lower bound~\eqref{eqn:self-penalization-kernel-ridge-regression}  follows immediately from the 
 special lower bound that we have established in Eq.~\eqref{eqn:final-lower-bound-special}. Formally, for any 
 $\beta \in \R_+^p$, we construct a surrogate $\beta'$ such that they differ only on the noise coordinates: 
 $\beta_S' = \beta_S$ and $\beta'_{S^c} = 0$. We first apply the special gradient lower bound to the surrogate $\beta'$
 (where $\beta'_{S^c} = 0$)
 \begin{equation}
 \label{eqn:special-gradient-bound-to-surrogate}
 	\partial_{\beta_j} \FKRR_\lambda(\beta'; \Q)
 		\ge  c_j \cdot \lambda \left(\FKRR_\lambda(0; \Q) - \FKRR_\lambda(\beta'; \Q)\right)_+^2,
 \end{equation}
 and then using the Lipschitzness of the mappings in Eq.~\eqref{eqn:two-Lipschitz-mappings}, we obtain 
 \begin{equation}
 \label{eqn:bound-the-gap-between-beta-and-surrogate}
 \begin{split}
 	 \partial_{\beta_j} \FKRR_\lambda(\beta; \Q) &\ge  \partial_{\beta_j} \FKRR_\lambda(\beta'; \Q) - C_1 \cdot \frac{1+\lambda}{\lambda^2} \cdot
 	 	\norm{\beta_{S^c}}_1 \\
 	\lambda \left(\FKRR_\lambda(0; \Q) - \FKRR_\lambda(\beta'; \Q)\right)_+^2 &\ge  \lambda \left(\FKRR_\lambda(0; \Q) - \FKRR_\lambda(\beta; \Q)\right)_+^2 - C_2 \norm{\beta_{S^c}}_1
 \end{split}
 \end{equation}
 Now we can combine the two inequalities in Eq.~\eqref{eqn:special-gradient-bound-to-surrogate} and 
 Eq.~\eqref{eqn:bound-the-gap-between-beta-and-surrogate} to derive the final gradient lower bound 
 that holds for all $\beta$ (with some appropriate $M$)
 \begin{equation}
 	\partial_{\beta_j} \FKRR_\lambda(\beta; \Q) \ge 
 		c_j \cdot \lambda \left(\FKRR_\lambda(0; \Q) - \FKRR_\lambda(\beta; \Q)\right)_+^2 - C \cdot \frac{1+\lambda^2}{\lambda^2} \norm{\beta_{S^c}}_1.
 \end{equation}

\subsection{Proof of Proposition~\ref{proposition:F-NG-self-penalizing}}
\label{sec:proof-of-proposiiton-NG-self-penalizing}

The proof is similar to that of Proposition~\ref{proposition:F-OLS-self-penalizing}. 
Note that $\mathcal{B}^*$ consists only of the global minimum $\beta^*$ of the non-negative garrotte. 
Fix a noise variable $j \not\in S$. We show that
 \begin{equation}
 \label{eqn:limit-classical-NG}
 	\lim_{\beta \to \beta^*} \sign(\beta_j) \cdot \partial_{\beta_j} \FNG_\gamma (\beta; \Q) = \gamma. 
 \end{equation}
 Denote $z_{\beta; \lambda}(X; Y) = Y - \sum_{i=1}^p w_i \beta_i X_i$ to be the residual, where 
 we recall that $w_i$ is the solution of the ordinary least square between $Y$ and $X$.
At any \nolinebreak $\beta$, 
\begin{equation}
\label{eqn:gradient-formula-NG}
\begin{split}
	\sign(\beta_j) \cdot \partial_{\beta_j} \FNG_\gamma (\beta; \Q) &= 
		\gamma - \sign(\beta_j) \cdot w_j \cdot \E_Q \big[X_j z_{\beta; \lambda}(X; Y)\big].
\end{split}
\end{equation}
Note $\lim_{\beta \to \beta^*} \E_Q \big[X_j z_{\beta; \lambda}(X; Y)\big] = \E_Q [X_j r_{\beta^*}(X; Y)] = \Cov_Q(X_j, r_{\beta^*}(X; Y)) = 0$
where the second identity uses (i) $\supp(\beta^*) \subseteq S$, (ii) $\E[Y|X] = \E[Y|X_S]$ and (iii) $X_S \perp X_{S^c}$. 

\subsection{Proof of Proposition~\ref{proposition:F-LIN-self-penalizing}}
\label{sec:proof-of-proposition-F-LIN-self-penalizing}
We compute the gradient of the objective $\FLIN_{\lambda, \gamma}(\beta; \Q)$ with respect to $\beta$. Applying the envelope theorem, 
we obtain the following expression that holds for all coordinate $j \in [p]$
\begin{equation*}
	\partial_{\beta_j} \FLIN_{\lambda, \gamma}(\beta; \Q) = \gamma - \E_Q[X_j z_{\beta; \lambda}(X; Y)],
\end{equation*}
where $z_{\beta; \lambda}(X; Y) = Y - \sum_{i=1}^p \beta_i w_i^*(\beta) X_i$ denotes the residual, where $w^*(\beta)$ 
is the solution 
\begin{equation*}
w^*(\beta) = \argmin_w \half \E_Q\Big[(Y - \sum_{i=1}^p \beta_i w_i X_i)^2\Big] + \frac{\lambda}{2} \ltwo{w}^2.
\end{equation*}
Note $w_j^*(\beta) = 0$ for all $j \not \in S$ since (i) $\E[Y|X] = \E[Y|X_S]$ and (ii) $X_S \perp X_{S^c}$. 
Hence, $\E_Q[z_{\beta; \lambda}(X; Y)|X] \perp X_j$ for any variable $j \not \in S$. As a result, we obtain
$\E_Q[X_j z_{\beta; \lambda}(X; Y)] = 0$ for all $j \not \in S$ and for all $\beta$. This proves in particular that 
for all $j \not \in S$, and for all $\beta$, 
\begin{equation}
	\sign(\beta_j) \cdot \partial_{\beta_j} \FLIN_{\lambda, \gamma}(\beta; \Q) = \gamma. 
\end{equation}
Hence, the objective $\FLIN_{\lambda, \gamma}(\beta; \Q)$ is self-penalizing if and only if $\gamma > 0$.

\newcommand{\K}{\mathcal{K}}
\subsection{Proof of Proposition~\ref{proposition:kernel-ridge-consistent-reason}}
\label{sec:proof-proposition-kernel-ridge-consistent-reason}
The proof is based on standard approximation-theoretic arguments.  

The key to the proof is to show that $\H$ is a universal RKHS~\citep{MicchelliXuZh06}. 
Write the kernel in terms of $v(x-x') = h(\norm{x-x'}_q^q)$. The characterization of $h$ in 
Eq.~\eqref{eqn:f-kernels} gives
\begin{equation}
\label{eqn:v-expression}
v(z) = \int_0^\infty e^{-t\norm{z}_q^q} \mu(dt). 
\end{equation}
 According to \citep[][Proposition 15]{MicchelliXuZh06},  $\H$ is universal if the Fourier transform of $v$ 
 has full support. With the expression~\eqref{eqn:v-expression}, we can easily show that the 
 Fourier transform $\F(v)$ is
 \begin{equation}
 \label{eqn:fourier-v-expression}
 \F(v) (\omega) = (2\pi)^{p/2} \cdot \int_0^\infty \prod_{i=1}^p \frac{1}{t^{1/q}}\vartheta\left(\frac{\omega}{t^{1/q}}\right) \mu(dt)
 \end{equation}
 where $\vartheta(\omega) = 1/(1+\omega^2)$ when $q=1$ and $\vartheta(\omega) = e^{-\omega^2/4}/2\sqrt{\pi}$ 
 when $q=2$. As $\vartheta$ has full support, expression~\eqref{eqn:fourier-v-expression} shows that $\F(v)$ has full support, and hence $\H$ is universal. 

The universal approximation property of the universal RKHS $\H$ implies that for 
any compact set $\mathcal{W}$, any continuous function $f$ on $\mathcal{W}$ and 
$\eps > 0$, there exists a function $g \in \H$ such that $\sup_{x \in \mathcal{W}} |f(x) - g(x)| \le \eps$~\citep{MicchelliXuZh06}.
Consequently, the universal approximating property implies that for any $\eps > 0$, there exists a function $\wbar{f} \in \H$
such that 
\begin{equation}
\label{eqn:universal-approximation-consequence}
\E_Q\Big[ \big(\E[Y|X_{\supp(\beta)}] - \wbar{f}(\beta^{1/q} \odot X) \big)^2\Big] \le \eps. 
\end{equation}
We can take $\lambda$ sufficiently small, say $\lambda \le \lambda_0$ so that $\lambda \norms{\wbar{f}}_\H^2 \le \eps$. Consequently, 
the ANOVA decomposition and triangle inequality lead to the following bound that holds for all $\lambda \le \lambda_0$: 
\begin{equation}
\begin{split}
\FKRR_\lambda(\beta; \Q) \le \half \E[(Y - \wbar{f}(\beta^{1/q} \odot X))^2] + \frac{\lambda}{2} \norms{\wbar{f}}_\H^2
	& \le \half \E [\Var(Y |X_{\supp(\beta)})] + \eps. 
\end{split}
\end{equation}
This immediately leads to Proposition~\ref{proposition:kernel-ridge-consistent-reason} as desired. 

%
%
%
%
%

\subsection{Proof of Lemma~\ref{lemma:kernel-ridge-regression-decay}}
\label{sec:proof-lemma-kernel-ridge-regression-decay}
Lemma~\ref{lemma:kernel-ridge-regression-decay} is a 
mere consequence of the following two components: (i) the Lyapunov convergence theorem for the gradient inclusion
and (ii) a lower bound on the size of the gradient at the beginning stage of the differential inclusion. 
In below discussions, we use $C_1, C_2, \ldots, $ to denote constants that does not depend on $\lambda, n$.

Formally, according to the Lyapunov theorem (Theorem~\ref{thm:main-result-grad-inclusion}), we have for all $t \ge 0$: 
\begin{equation}
\label{eqn:Lyapunov-convergence-theorem}
	\FKRR_\lambda(\beta(t); \Q) \le \FKRR_\lambda(\beta(0); \Q) - \int_0^t \norm{\subgrad\opt(\beta(s))}_2^2 ds, 
\end{equation}
where the mapping $\subgrad\opt(\beta) \defeq \argmin\{\norm{g}_2: g \in \grad \FKRR_\lambda(\beta; \Q) + \normal_\mathcal{Z}(\beta)\}$
denotes the minimal gradient at each $\beta \in \mathcal{Z}$. In words, the theorem says that the objective value 
has a certain rate of decay governed by the minimal size of the gradient along the trajectory.

Lemma~\ref{lemma:representation-of-KRR-gradient} and Lemma~\ref{lemma:conditional-negative-definite} 
give the following formula of the gradient at $\beta(0) = 0$. For each coordinate $j \in [p]$, the gradient of the objective has the following expression: 
\begin{equation}
	\partial_{\beta_j} \FKRR_\lambda(0; \Q) = \frac{1}{\lambda} \cdot |h^\prime(0)| \cdot \E_Q[YY'|X_j - X_j'|]
		= -\frac{1}{\lambda} \cdot |h^\prime(0)| \cdot \int_0^\infty \left|\E_Q[Ye^{i\omega X_j}]\right|^2 \cdot \frac{d\omega}{\pi \omega^2}.
\end{equation}
Note the integral is positive $\int_0^\infty \left|\E_Q[Ye^{i\omega X_l}]\right|^2\cdot \frac{d\omega}{\pi\omega^2} > 0$
whenever $\Var_Q(\E_Q[Y|X_l]) \neq 0$. As its consequence, the following lower bound holds for some constant $C_1$:  
\begin{equation}
\label{eqn:gradient-lower-bound-at-0}
	\norm{\subgrad\opt(\beta(0))}_2 = \norm{\subgrad\opt(0)}_2 \ge -\partial_{\beta_l} \FKRR_\lambda(0; \Q) \ge \frac{1}{\lambda} \cdot C_1.
\end{equation}
Below we extend the bound~\eqref{eqn:gradient-lower-bound-at-0} to $\norm{\subgrad\opt(\beta(t))}_2$ for small $t > 0$ by exploiting the 
Lipschitzness of the mapping $t \mapsto \norm{\subgrad\opt(\beta(t))}_2$ near $t=0$.  Note the following facts. 
\begin{enumerate}[(i)]
\item The trajectory $t \mapsto \beta(t)$ is $C_2/\lambda$ Lipschitz as the gradient $\grad\FKRR_\lambda(\beta; \Q)$ is uniformly 
	bounded in the feasible set $\mathcal{Z}$ (Lemma~\ref{lemma:uniform-boundedness-of-gradient})
\item The gradient mapping $\beta \mapsto \grad \FKRR_\lambda(\beta; \Q)$ is $C_3/\lambda^2$ Lipschitz.
\end{enumerate}
The composition rule gives that $t \mapsto \grad \FKRR_\lambda(\beta(t); \Q)$ is also Lipschitz with the constant $M_3 = C_4/\lambda^3$ 
where $C_4 = C_2C_3$. Consequently, this implies the lower bound 
$\norm{\subgrad\opt(\beta(t))}_2 \ge \half \norm{\subgrad\opt(\beta(0))}_2 \ge \frac{1}{2\lambda} \cdot C_1$ when 
$t \le C_5\lambda^2$ for sufficiently small $C_5$.
As a consequence, we can now apply Lyapunov theorem~\eqref{eqn:Lyapunov-convergence-theorem} to obtain 
for $C_6 = C_1^2 C_5/4$ 
\begin{equation}
	\FKRR_\lambda(\beta(t); \Q) \le \FKRR_\lambda(\beta(0); \Q) - C_6~~\text{when $t = C_5\lambda^2$}.
\end{equation}

\subsection{Proof of Proposition~\ref{proposition:KRR-population}}
\label{sec:proof-proposition-KRR-population}
We prove the result accordingly in the three bullet points below. 
\begin{itemize}
\item The first point follows from Lemma~\ref{lemma:kernel-ridge-regression-decay}. 
Lemma~\ref{lemma:kernel-ridge-regression-decay} shows that
\begin{equation*}
	\FKRR_\lambda(\beta(t); \Q) \le \FKRR_\lambda(\beta(0); \Q) - 3c~~\text{when $t = C\lambda^2$}.
\end{equation*}
where $c, C > 0$ do not depend on $n, \lambda$.  The self-penalizing property implies
$\norm{\beta_{S^c}(t)}_1 = 0$ for $t \ge 0$. This implies that the gradient dynamic enters $\X_{3c, \delta', \lambda}$
at time $t= C\lambda^2$ for any $\delta'>0$. The inclusion $\X_{3c, \delta/81, \lambda} \subseteq \X_{c, \delta, \lambda}$
simply holds according to the definition. 
\item The second point is simply a restatement of the self-penalizing property of the kernel ridge regression objective 
(Corollary~\ref{cor:kernel-ridge-regression-self-penalizing}).
\item The third point follows from the self-penalizing property of $\X_{c, \delta, \lambda}$ and the monotonicity of the 
	gradient dynamics. Indeed, any gradient dynamics initiated from $\X_{c, \delta, \lambda}$ must 
	(i) monotonically decrease the objective values by Theorem~\ref{thm:main-result-grad-inclusion}  
	(ii) monotonically decrease the size of the noise variables as the objective is self-penalizing 
	on the set $\X_{c, \delta, \lambda}$. This proves invariance of $\X_{c, \delta, \lambda}$.
\end{itemize}

\subsection{Proof of Proposition~\ref{proposition:KRR-finite-sample}}
\label{sec:proof-proposition-KRR-finite-sample}
Proposition~\ref{proposition:KRR-finite-sample} is almost an immediate consequence of the population version, i.e., 
Proposition~\ref{proposition:KRR-population}. Let $c, C, \delta > 0$ be the constant such that 
Proposition~\ref{proposition:KRR-population} holds. 
The following fact is used throughout the proof, which follows from 
triangle inequality and definition of $\Omega_n(c)$: 
\begin{equation}
\label{eqn:basic-set-containment-KRR}
\X_{3c, \delta/81, \lambda, n} \subseteq \wtilde{\X}_{2c, \delta/16, \lambda, n} \subseteq \X_{c, \delta, \lambda, n}
	~~\text{on the event $\Omega_n(c)$}.
\end{equation} 
Below we set $\wtilde{c} = 2c$, $\wtilde{C} = C$, $\wtilde{\delta} = \delta/16$, $\wtilde{\delta} = \delta$. Let $\eps \le c$ so that 
$\Omega_n(\eps \lambda) \subseteq \Omega_n(c)$ as $\lambda \le 1$.

We shall determine the constant $c \ge \eps > 0$ as the
minimum of $\eps_1, \eps_2$ where $\eps_1, \eps_2 > 0$ are given in the first and second bullet point below. 
Noticeably, the constants $\eps_1, \eps_2 > 0$ thus defined are independent of $\lambda, n$.


\begin{itemize}
\item Below we prove that $\wtilde{\beta}(t) \in \wtilde{\X}_{\wtilde{c}, \wtilde{\delta}, \lambda, n}$ at time $t = \wtilde{C}\lambda^2$ 
	on event $\Omega_n(\eps_1 \lambda)$ where $\eps_1 > 0$ is determined below. By, 
	Proposition~\ref{proposition:KRR-population}, $\beta(t) \in \X_{3c, \delta/81, \lambda, n}$ at time $t = C\lambda^2$, i.e., 
	\begin{equation}
	\begin{array}{c} 
		\normsmall{\beta_{S^c}(t)}_1 \le c^2 \delta \lambda^3/9 ~~\text{and}~~\FKRR_\lambda(\beta(t); \Q) < \FKRR_\lambda(0; \Q) -3c
	\end{array}~\text{for $t = C\lambda^2$}.
	\end{equation} 	 
	By Lemma~\ref{lemma:gronwall-deviation-bounds}, we can pick $c/2 > \eps_1 > 0$ small enough such that
	on event $\Omega_n(\eps_1 \lambda)$: 
	\begin{equation}
	\begin{array}{c} 
		\normsmall{\wtilde{\beta}(t) - \beta(t)}_1 \le c^2 \delta \lambda^3/9 ~\text{and}~
		 |\FKRR_\lambda(\beta(t); \Q)- \FKRR_\lambda(\wtilde{\beta}(t); \Q_n)|\le c/2
	\end{array}~\text{for $t \le C\lambda^2$}.
	\end{equation} 
	Now the triangle inequality gives the following bound on the event $\Omega_n(\eps_1 \lambda)$:
	\begin{equation}
		\normsmall{\wtilde{\beta}_{S^c}(t)}_1 \le c^2 \delta \lambda^3/4~\text{and}~\FKRR_\lambda(\wtilde{\beta}(t); \Q_n) < \FKRR_\lambda(0; \Q_n)-2c
			~\text{for $t = C\lambda^2$}.
	\end{equation}
	In other words, this says that $\wtilde{\beta}(t) \in \wtilde{\X}_{\wtilde{c}, \wtilde{\delta}, \lambda, n}$ at time $t = \wtilde{C}\lambda^2$. 
\item The objective is self-penalizing on the set $\wtilde{\X}_{\wtilde{c}, \wtilde{\delta}, \lambda, n}$ on the event $\Omega_n(\eps_2 \lambda)$ 
	where $0 < \eps_2 \le c$ is determined below. Proposition~\ref{proposition:metric-learning-population} says that $\FKRR(\beta; \Q)$ is self-penalizing 
	on the set $\X_{c, \delta, \lambda}$. This means that the \emph{population} gradient $\partial_{\beta_j} \FKRR(\beta; \Q)$ has a uniform lower 
	bound over the set of $\beta \in \X_{c, \delta, \lambda}$ and the set of noise variables $j\not\in S$. A careful traceback of the proof shows 
	that the lower bound can be taken as $\underline{c} \lambda$ (Theorem~\ref{theorem:kernel-ridge-regression}) where $\underline{c}$
	is some constant that's independent of $n, \lambda$.

	Now we pick $\eps_2 = \underline{c}/2 \wedge c$ and consider the event $\Omega_n(\eps_2\lambda)$. On the event 
	$\Omega_n(\eps_2 \lambda)$, by triangle inequality, the \emph{empirical} gradient $\partial_{\beta_j} \FKRR(\beta; \Q_n)$ 
	has a uniform lower bound $\underline{c}\lambda/2 > 0$ over $\beta \in \X_{c, \delta, \lambda}$ and $j \not\in S$. Note the exactly same 
	lower bound also holds over $\beta \in \wtilde{\X}_{\wtilde{c}, \wtilde{\delta}, \lambda, n}$ since the inclusion 
	$\wtilde{\X}_{\wtilde{c}, \wtilde{\delta}, \lambda, n} \subseteq \X_{c, \delta, \lambda}$ holds on the event $\Omega_n(\eps_2 \lambda)$.
\item The set $\wtilde{\X}_{\wtilde{c}, \wtilde{\delta}, \lambda, n}$ is invariant with respect to the gradient dynamics  of $\FKRR(\cdot; \Q_n)$
	on the event $\Omega_n(\eps \lambda)$ where $\eps = \eps_1 \wedge \eps_2$. 
	This follows from the self-penalizing property of $\wtilde{\X}_{\wtilde{c}, \wtilde{\delta}, \lambda, n}$ and the monotonicity of the 
	gradient dynamics. Indeed, any gradient dynamics initiated from $\wtilde{\X}_{\wtilde{c}, \wtilde{\delta}, \lambda, n}$ must 
	(i) monotonically decrease the objective values by Theorem~\ref{thm:main-result-grad-inclusion} 
	(ii) monotonically decrease the size of the noise coordinates as the objective is self-penalizing
	on the set $\wtilde{\X}_{\wtilde{c}, \wtilde{\delta}, \lambda, n}$. This proves the invariance of $\wtilde{\X}_{\wtilde{c}, \wtilde{\delta}, \lambda, n}$ on 
	$\Omega_n(\eps_2 \lambda)$.
\end{itemize}
As a consequence, there exists $\tau < \infty$ so that $\emptyset \subsetneq \supp(\beta(t)) \subseteq S$ for all $t \ge \tau$
on the event $\Omega_n(\eps\lambda)$. Furthermore, we can choose the value of $\tau = \overline{c}\lambda^2$ where 
$\overline{c}$ is independent of $n, \lambda$ (since within the set $\wtilde{\X}_{\wtilde{c}, \wtilde{\delta}, \lambda, n}$ 
(i) the size of the noise variable is at most $c^2 \delta \lambda^3$ and (ii) the gradient with respect to the noise variables is at least 
$\underline{c}\lambda/2$, and hence it takes at most 
$(2c^2 \delta/\underline{c}) \cdot \lambda^2$ time for the gradient descent to move the noise variables to exactly $0$).

\begin{lemma}
\label{lemma:gronwall-deviation-bounds}
Consider the population and empirical gradient flow $t \mapsto \beta(t)$ and $t \mapsto \wtilde{\beta}(t)$ that 
share the same initializations $\beta(0) = \wtilde{\beta}(0)$. Then there exists a constant $C > 0$ which 
does not depend on $n, \lambda$ such that 
for any $\eps > 0$, the deviation bound below holds on $\Omega_n(\eps)$:  
\begin{equation}
\label{eqn:gronwall-type-deviation-bound}
\normbig{\wtilde{\beta}(t) - \beta(t)}_1 \le C\eps \lambda^2 \cdot e^{Ct/\lambda^2}~~\text{for all $t \ge 0$}. 
\end{equation}
Moreover, the empirical and population objectives evaluated at its corresponding gradient flow has 
the following deviation bound on $\Omega_n(\eps)$. 
\begin{equation}
\label{eqn:deviation-in-objective-value}
	\left|\FKRR_\lambda(\beta(t);\Q) - \FKRR_\lambda(\wtilde{\beta}(t); \Q_n)\right| \le  \eps + 
		C\eps \lambda \cdot e^{Ct/\lambda^2}~~\text{for all $t \ge 0$}. 
\end{equation}
\end{lemma}
\begin{proof}
Throughout the proof, we use $C_1, C_2, \ldots, $ to denote constants independent of $n, \lambda$. 

The first point is a consequence of the Gr\"{o}nwall's inequality in the theory of differential inclusion. 
On the event $\Omega_n(\eps)$, we have uniform control on the maximum $\ell_\infty$ deviations between the 
empirical and population gradient vector fields in the feasible set $\mathcal{Z}$. Note that the population 
gradient field is Lipschitz in the following sense (see Proposition~\ref{lemma:Lipschitz-of-gradient}): 
\begin{equation*}
	\norm{\grad \FKRR_\lambda(\beta; \Q) - \grad \FKRR_\lambda(\beta'; \Q)}_\infty \le \frac{C_1}{\lambda^2} \cdot \norm{\beta- \beta'}_1.
\end{equation*}
Viewing the empirical gradient flow as a perturbed version of the population gradient flow, 
the Gr\"{o}nwall's inequality (Theorem~\ref{thm:perturbation-of-gradient-inclusion}) yields the desired estimate in Eq.~\eqref{eqn:gronwall-type-deviation-bound}. 

The second point follows from the uniform $\ell_\infty$
bound on the gradient and the $\ell_1$ deviation bound in Eq.~\eqref{eqn:gronwall-type-deviation-bound}. Formally, 
Lemma~\ref{lemma:uniform-boundedness-of-gradient} shows that
$\norm{\grad \FKRR_\lambda(\beta; \Q)}_\infty \le C_2/\lambda$ holds for all $\beta$. With the $\ell_1$ bound in 
Eq.~\eqref{eqn:gronwall-type-deviation-bound}, Taylor's intermediate theorem gives
\begin{equation*}
\left|\FKRR_\lambda(\beta(t);\Q) - \FKRR_\lambda(\wtilde{\beta}(t); \Q)\right| \le
		C_3\eps \lambda \cdot e^{C_3 t/\lambda^2}~~\text{for all $t \ge 0$}.
\end{equation*}
Note $\Big|\FKRR_\lambda(\wtilde{\beta}(t); \Q) - \FKRR_\lambda(\wtilde{\beta}(t); \Q_n) \Big| \le \eps$ for all $t \ge 0$ on the event $\Omega_n(\eps)$ by definition.
By the triangle inequality, the desired estimate in Eq.~\eqref{eqn:deviation-in-objective-value} now simply follows.
\end{proof}

\subsection{A technical lemma}

\begin{lemma}
\label{lemma:conditional-negative-definite}
Assume $X, Y$ are random variables such that $\E_Q[Y] = 0$ and $\E_Q[X^2] < \infty$. 
Let $(X', Y')$ be independent copies of $(X, Y)$. 
Then $\E_Q[YY'|X-X'|^q] \le 0$ for $q = 1, 2$.
\end{lemma}
\begin{proof}
The lemma follows immediately from the following identities: 
\begin{equation}
\label{eqn:identities-on-conditional-negative-kernel}
\begin{split}
	\E_Q[YY'|X-X'|] &= - \int_0^\infty \left|\E_Q[\Cov(Y, e^{i\omega X})]\right|^2 \cdot \frac{1}{\pi \omega^2} d\omega \\
	\E_Q[YY'|X-X'|^2] &= - \Cov_Q^2(Y, X). 
\end{split}
\end{equation}
The proof of Eq.~\eqref{eqn:identities-on-conditional-negative-kernel} can be found 
in~\citep[][Section 3]{JordanLiRu21}. 
\end{proof}

\subsection{On the global minimum of the empirical metric learning objective}
\label{sec:global-minimum-empirical-metric-learning}
\begin{proposition}
\label{proposition:global-minimum-empirical-metric-learning}
Assume Assumptions~\ref{assumption:mu-assumtion-kernels} and~\ref{assumption:distribution-P}. 
Assume $S \neq \emptyset$ and $S \subseteq \supp(\beta(0))$. Then with probability at least $1-\exp(-c'n)$, 
the global minimizer $\wtilde{\beta}^*$ of the empirical metric learning objective $\FML(\beta; \Q_n)$ satisfies 
$\emptyset \subsetneq \supp(\wtilde{\beta}^*) \subseteq S$. Here $c' > 0$ is independent of $n$.
\end{proposition}

\begin{proof}
This is a consequence of Proposition~\ref{proposition:metric-learning-finite-sample} and 
Lemma~\ref{lemma:uniformly-close-population-finite-samples}. Indeed, 
Proposition~\ref{proposition:metric-learning-finite-sample} says that on some event $\Omega_n(\eps)$
where $\eps > 0$, the sublevel set $\mathcal{X}_{c, n}= \{\beta \in \mathcal{Z}: \FML(\beta;\Q_n) < -c\}$ for some $c > 0$  
is invariant and self-penalizing. As a result, the global minimizer, which belongs to the set $\mathcal{X}_{c, n}$ must 
have support contained in the signal set $S$. Note then $\Omega_n(\eps)$ happens with probability 
at least $1-\exp(-c'n)$ by Lemma~\ref{lemma:uniformly-close-population-finite-samples}.
\end{proof}

\section{Concentration}

This section presents a uniform concentration result for U-statistics. The result is particularly useful in showing
the concentration of the objective values and gradients for the metric learning and for the kernel ridge regression 
(Section~\ref{sec:proof-lemma-uniformly-close-population-finite-samples} and~\ref{sec:proof-lemma-uniformly-close-population-finite-samples-KRR}).
The proof combines the idea of chaining standard in the empirical process theory~\citep{VanWe13} and a specific decoupling 
technique common in the analysis of U-statistics~\citep{Hoeffding94}. 

To state the result, we introduce some notation. Let $U_\beta(z, z')$ be a family of U-statistics indexed by $\beta \in \mathcal{B}$. 
Let $H_\beta(\Q) \defeq \E_Q[U_\beta(Z, Z')]$ where $Z, Z'$ are independent copies sampled from the distribution $\Q$. Let 
$\Q_n$ denote the empirical distribution. Let $\mathcal{Q} = \supp(\Q)$.

\begin{proposition}
\label{proposition:uniform-convergence-general-U}
Assume the following conditions. Let $\norm{\cdot}$ be a norm on $\R^p$. Let $\alpha > 0$.
\begin{itemize}
\item Boundedness: $|U_\beta(z, z')| \le M_Z$ for $\beta \in \mathcal{B}$ and $z, z' \in \mathcal{Q}$. 
\item H\"{o}lder Continuity: $|U_\beta(z, z') - U_{\beta'}(z,z')| \le L \norm{\beta-\beta'}^\alpha$ for $\beta, \beta' \in \mathcal{B}$ and $z, z' \in \mathcal{Q}$. 
\item Norm constraint: $\mathcal{B} \subseteq \{\beta: \norm{\beta} \le M\}$. 
\end{itemize}
There exists a constant $C_\alpha$ depending only on $\alpha$ such that with probability at least $1-e^{-t^2}$
\begin{equation*}
\sup_{\beta \in \mathcal{B}} |H_\beta(\Q_n) - H_\beta(\Q)| \le \frac{C_\alpha}{\sqrt{n}} \cdot (ML\sqrt{p} + M_Z(1+t)).
\end{equation*}
\end{proposition}

\begin{proof}
Let $W \equiv \sup_{\beta \in \mathcal{B}} |H_\beta(\Q_n) - H_\beta(\Q)|$ denote the maximum deviation. 
We view $W \equiv W(Z_{1:n})$ as a function of the i.i.d. data $Z_{1:n} = (Z_1, Z_2, \ldots, Z_n)$. Clearly, 
the function is of bounded difference with bound $2M_Z/n$, i.e., for any $Z_{1:n}$ and $Z_{1:n}'$ differing
in only one coordinate, 
\begin{equation*}
	\left|W(Z_{1:n}) - W(Z_{1:n}') \right| \le 2M_Z/n. 
\end{equation*} 
McDiarmid's inequality~\citep{Mcdiarmid89} yields that with probability at least $1-e^{-t^2}$: 
\begin{equation}
\label{eqn:link-concentration-to-expectation}
	W \le \E[W] + \frac{2M_Z}{\sqrt{n}} t.
\end{equation}
Below we bound $\E[W]$. Our major technique is the symmetrization argument followed by an upper 
bound on the Dudley's metric entropy integral~\citep{VanWe13}. Since the U-statistics $H_\beta(\Q_n)$ is an average
 of \emph{dependent} variables, we need to apply the decoupling technique due to Hoeffding before the symmetrization~\citep{Hoeffding94}. 
Formally, introduce the following notation. 
\begin{itemize}
\item Let $\Delta_\beta(z, z') = U_\beta(z, z') - \E[U_\beta(z, z')]$.
\item Let $\sigma_{i, i'}$ be independent Radamacher random variables. 
\item The total index $\mathcal{I} = \{(i, i') | 1\le i, i' \le n\}$ can be divided into two groups:  $\mathcal{I} = \mathcal{I}_+ \cup \mathcal{I}_0$
	where $\mathcal{I}_+ = \{(i, i') | i \neq i', 1\le i, i' \le n\}$ and $\mathcal{I}_0 = \{(i, i) | 1\le i \le n\}$. A simple combinatorial argument 
	shows that we can further decompose $\mathcal{I}_+  = \cup_{j=1}^J \mathcal{I}_j$ where any two different tuple in the same $\mathcal{I_j}$
	has no intersection---for any $(i_1, i_2)$, $(i_3, i_4) \in \mathcal{I}_j$ we have $i_{l_1} \neq i_{l_2}$ where 
	$1\le l_1 < l_2 \le 4$---and where $|\mathcal{I}_j| \ge \floor{\frac{n}{2}}$ for all $1\le j\le J$.
\end{itemize}
Now, we are ready to bound $\wbar{W} = \E[W]$. As $\mathcal{I} = \cup_{j=0}^J \mathcal{I}_j$, we have 
\begin{equation*}
\wbar{W} = \E\left[\sup_{\beta \in \mathcal{B}} \frac{1}{n^2} \left|\sum_{(i, i') \in \mathcal{I}} \Delta_\beta(Z_i, Z_i')\right|\right]
	\le  \sum_{j=0}^J\E\left[\sup_{\beta \in \mathcal{B}}\left| \frac{1}{n^2} \sum_{(i, i') \in \mathcal{I}_j} \Delta_\beta(Z_i, Z_i')\right|\right] = \sum_{j=0}^J \wbar{W}_j
\end{equation*}
Now we bound each term $\wbar{W}_j$ on the right-hand side. Use the boundedness assumption, it is easy to see that $\wbar{W}_j \le 2M_Z/n$ for all $j$. Below
we give a tighter bound for $\wbar{W}_j$ when $j \ge 1$. 

The proof is based on the empirical process theory. Fix $j \ge 1$. By symmetrization, we obtain
\begin{equation}
\label{eqn:symmetrization-step}
	\wbar{W}_j 
		\le \frac{1}{n^2} \cdot \E \Bigg[\sup_{\beta \in \mathcal{B}} \sum_{(i, i') \in \mathcal{I}_j} \sigma_{i, i'} \Delta_\beta(Z_i, Z_i')\Bigg]. 
\end{equation}
Write $G_\beta = \sum_{(i, i') \in \mathcal{I}_j} \sigma_{i, i'} \Delta_\beta(Z_i, Z_i')$. Note then
$\beta \mapsto G_\beta$ is continuous and $G_\beta - G_{\beta'}$ is subgaussian with parameter  
at most $\sqrt{n} L \norm{\beta -\beta'}^\alpha$. Dudley's integral bound yields
\begin{equation}
\label{eqn:dudley-integral}
	\E \Bigg[\sup_{\beta \in \mathcal{B}} \sum_{(i, i') \in \mathcal{I}_j} \sigma_{i, i'} \Delta_\beta(Z_i, Z_i')\Bigg]
		\le 12 \cdot \sqrt{n} L \cdot \int_0^\infty \sqrt{\log N(\mathcal{B}, \norm{\cdot}, \eps^{1/\alpha})} d\eps,
\end{equation}
where $N(\mathcal{B}, \norm{\cdot}, \eps)$ is the covering number of the set $\mathcal{B}$ using the $\norm{\cdot}$-ball of radius $\eps$. 
As $\log N(\mathcal{B}, \norm{\cdot}, \eps) = 0$ for $\eps > M$ and $\log N(\mathcal{B}, \norm{\cdot}, \eps) \le p \log (3M/\eps)$
for $\eps \le M$, we obtain
\begin{equation}
\label{eqn:metric-entropy-bound}
	\int_0^\infty \sqrt{\log N(\mathcal{B}, \norm{\cdot}, \eps^{1/\alpha})} d\eps
	\le (3M)^\alpha \sqrt{p} \cdot \int_0^{1/3} \eps^{\alpha-1}\sqrt{\log (1/\eps)} d\eps.
\end{equation}
Consequently, we use equations~\eqref{eqn:symmetrization-step}---\eqref{eqn:metric-entropy-bound} to obtain that 
$\wbar{W}_j \le C_\alpha M^\alpha L\sqrt{p}/n^{3/2}$ for $j \ge 1$, where $C_\alpha = 12 \cdot 3^\alpha \int_0^{1/3}\eps^{\alpha-1}\sqrt{\log (1/\eps)} d\eps < \infty$.
This in turn implies the bound on $W$:
\begin{equation}
\label{eqn:expectation-upper-bound}
	\wbar{W} = \wbar{W}_0 + \sum_{j=1}^J \wbar{W}_j \le C_\alpha \cdot \left(\frac{M_Z}{n} + M^\alpha L \cdot  \sqrt{\frac{p}{n}}\right).
\end{equation}
The result now follows from the concentration~\eqref{eqn:link-concentration-to-expectation} and the bound~\eqref{eqn:expectation-upper-bound}. 
\end{proof}

\subsection{Metric learning: Proof of Lemma~\ref{lemma:uniformly-close-population-finite-samples}}
\label{sec:proof-lemma-uniformly-close-population-finite-samples}

Lemma~\ref{lemma:uniformly-close-population-finite-samples} is simply a consequence of 
Proposition~\ref{proposition:uniform-convergence-general-U}. Below we use notation $C, C_1, C_2, \ldots, $
to denote constants that depend only on $M, M_X, \mu, p$. Throughout the section, we use the 
notation $Z = (X, Y)$ to denote the data. We use $Z, Z'$ to denote independent copies. 

\begin{itemize}
\item We first prove the concentration of the objective value. We show that there exists a constant $C > 0$
such that for any $\eps > 0$, with probability at least $1-e^{- n\eps^2}$: 
\begin{equation}
\label{eqn:concentration-of-objectives}
	\sup_{\beta \in \mathcal{Z}} \left|\FML(\beta; \Q) - \FML(\beta; \Q_n)\right| \le C \cdot \left(\frac{1}{\sqrt{n}} + \eps\right).
\end{equation} 
Note that $\FML(\beta; \Q) = \E[U_\beta(Z, Z')]$ where $U_\beta(z, z') = yy' h(\norm{x-x'}_{1, \beta})$. 
As the function $h$ is completely monotone, we know that the function $h$ is uniformly bounded and 
Lipschitz on $\R_+$. As a consequence, the function $\beta \mapsto U_\beta$ is bounded by $C_1$
and Lipschitz on the box constraint set $\mathcal{Z}$ with the Lipschitz constant at most $C_2$. 
Now the concentration result in Eq.~\eqref{eqn:concentration-of-objectives} simply follows from 
Proposition~\ref{proposition:uniform-convergence-general-U}. 
\item We next prove the concentration of the gradient. We show that there exists a constant $C > 0$
such that for any $\eps > 0$, for any $j \in [p]$, with probability at least $1-e^{- n\eps^2}$: 
\begin{equation}
\label{eqn:concentration-of-gradient}
	\sup_{\beta \in \mathcal{Z}} \left|\partial_{\beta_j} \FML(\beta; \Q) - \partial_{\beta_j} 
		\FML(\beta; \Q_n)\right| \le C \cdot \left(\frac{1}{\sqrt{n}} + \eps\right).
\end{equation} 
Note that $\partial_{\beta_j} \FML(\beta; \Q) = \E[U_\beta(Z, Z')]$ where $U_\beta(z, z') = yy' h(\norm{x-x'}_{1, \beta})|x_j - x_j'|$. 
Similar to before, it's easy to show that the function $\beta \mapsto U_\beta$ is bounded by $C_3$
and Lipschitz on the box constraint set $\mathcal{Z}$ with the Lipschitz constant at most $C_4$. 
Now the concentration result in Eq.~\eqref{eqn:concentration-of-gradient} simply follows from 
Proposition~\ref{proposition:uniform-convergence-general-U}. 
\end{itemize}
Lemma~\ref{lemma:uniformly-close-population-finite-samples} now follows from the concentration 
results in Eq.~\eqref{eqn:concentration-of-objectives} and~\eqref{eqn:concentration-of-gradient}.

\newcommand{\FKRRX}{\wbar{F_\lambda^{\rm KRR}}}

\subsection{Kernel ridge regression: Proof of Lemma~\ref{lemma:uniformly-close-population-finite-samples-KRR}}
\label{sec:proof-lemma-uniformly-close-population-finite-samples-KRR}

\paragraph{Notation}
Let $\what{f}_{\beta; \lambda}$ and $f_{\beta; \lambda}$ denote the minimizer of the empirical and population kernel ridge regression. 
Mathematically, they are defined by the following equations
\begin{equation*}
\begin{split}
	\what{f}_{\beta; \lambda} = \argmin_{f\in \H} \half \E_{\Q_n}[(Y - f(\beta^{1/q} \odot X))^2] + \frac{\lambda}{2} \norm{f}_\H^2. \\
	f_{\beta; \lambda} = \argmin_{f\in \H} \half \E_{\Q}[(Y - f(\beta^{1/q} \odot X))^2] + \frac{\lambda}{2} \norm{f}_\H^2. 
\end{split}
\end{equation*}
Let $\what{z}_\beta(x, y) = y - \what{f}_{\beta; \lambda}(\beta^{1/q}\odot x; y)$ and $z_\beta(x, y) = y - f_{\beta; \lambda}(\beta^{1/q} \odot x; y)$.
We use the 
notation $Z = (X, Y)$ to denote the data. We use $Z, Z'$ to denote independent copies. 

Throughout the section, we use notation $C, C_1, C_2, \ldots, $ to denote constants that depend only on $M, M_X, M_Y, \mu, p$. 
In particular, these constants do not depend on $n, \lambda$. 

\paragraph{Main Proof}
Lemma~\ref{lemma:uniformly-close-population-finite-samples-KRR} basically follows from 
Proposition~\ref{proposition:uniform-convergence-general-U}. 

\begin{itemize}
\item We first prove the concentration of the objective value. We show that there exists a constant $C > 0$
such that for any $\eps > 0$, with probability at least $1-e^{- n\eps^2}$: 
\begin{equation}
\label{eqn:concentration-of-objectives-krr}
	\sup_{\beta \in \mathcal{Z}} \left|\FKRR_\lambda(\beta; \Q) - \FKRR_\lambda(\beta; \Q_n)\right| \le \frac{C}{\lambda^{3/2}} \cdot \left(\frac{1}{\sqrt{n}} + \eps\right).
\end{equation} 
According to Lemma~\ref{lemma:representation-of-KRR}, we have the representation of the objective value: 
\begin{equation*}
	\FKRR_\lambda(\beta; \Q) = \half \E_Q[z_\beta(X; Y)Y]~~\text{and}~~\FKRR_\lambda(\beta; \Q_n) = \half \E_{Q_n}[\what{z}_\beta(X; Y)Y].
\end{equation*}
To facilitate the proof, it's natural to introduce the auxiliary quantity $\FKRRX(\beta) = \half \E_{\Q_n} [z_\beta(X; Y) Y]$. Below we show 
the existence of constants $C_1, C_2 > 0$ such that for any $\eps > 0$, the following holds with probability at least $1-e^{- n\eps^2}$: 
\begin{align}
\label{eqn:concentration-of-objectives-krr-two-steps-1}
	\sup_{\beta \in \mathcal{Z}} \left|\FKRR_\lambda(\beta; \Q) - \FKRRX(\beta)\right|  &\le  \frac{C_1}{\lambda^{3/2}} \cdot \left(\frac{1}{\sqrt{n}} + \eps\right), \\
\label{eqn:concentration-of-objectives-krr-two-steps-2}
	\sup_{\beta \in \mathcal{Z}} \left|\FKRR_\lambda(\beta; \Q_n) - \FKRRX(\beta)\right|  &\le  \frac{C_2}{\lambda^{3/2}} \cdot \left(\frac{1}{\sqrt{n}} + \eps\right).
\end{align}
It's clear that the desired concentration result~\eqref{eqn:concentration-of-objectives-krr} would follow immediately by the triangle inequality. 
It remains to prove the high probability bound~\eqref{eqn:concentration-of-objectives-krr-two-steps-1} and~\eqref{eqn:concentration-of-objectives-krr-two-steps-2}.

The concentration~\eqref{eqn:concentration-of-objectives-krr-two-steps-1} follows from a straightforward application of 
Proposition~\ref{proposition:uniform-convergence-general-U}.
Indeed, $\FKRR_\lambda(\beta; \Q) = \E_Q[U_\beta(Z, Z')]$ and $\FKRRX(\beta) = \E_{Q_n}[U_\beta(Z, Z')]$ where 
$U_\beta(z, z') = \frac{1}{4} (y z_\beta(x; y) + y' z_\beta(x'; y'))$. Note the function $\beta \mapsto U_\beta(z, z')$ is bounded by 
$C_3/\lambda^{1/2}$ and is Lipschitz on the box constraint $\mathcal{Z}$ with the Lipschitz constant at most 
$C_4/\lambda^{3/2}$ when $z, z' \in \supp(\Q)$, thanks to Lemma~\ref{lemma:bound-lipschitzness-of-f} and the assumption 
that $\supp(\Q)$ is compact.
The concentration result thus follows. 

The high probability bound in Eq.~\eqref{eqn:concentration-of-objectives-krr-two-steps-2} is straightforward from 
Lemma~\ref{lemma:uniform-close-what-r-true-r} and the Cauchy-Schwartz inequality (recall that $|Y| \le C_5$ almost surely by assumption).

\item Next we prove the concentration of the gradient. We show that there exists a constant $C > 0$
such that for any $\eps > 0$, for any $j \in [p]$, with probability at least $1-e^{- n\eps^2}$: 
\begin{equation}
\label{eqn:concentration-of-gradient-krr}
	\sup_{\beta \in \mathcal{Z}} \left|\partial_{\beta_j} \FML(\beta; \Q) - \partial_{\beta_j} 
		\FML(\beta; \Q_n)\right| \le \frac{C}{\lambda^3} \cdot \left(\frac{1}{\sqrt{n}} + \eps\right).
\end{equation} 
According to Lemma~\ref{lemma:representation-of-KRR-gradient}, we have the representation of the objective value: 
\begin{equation*}
\begin{split}
	\partial_{\beta_j} \FKRR_\lambda(\beta; \Q) &= -\frac{1}{\lambda} \E_Q \left[z_\beta(X; Y) z_\beta(X'; Y') h^\prime(\norm{X-X'}_{q, \beta}^q) |X_l - X_l'|^q\right], \\
	\partial_{\beta_j} \FKRR_\lambda(\beta; \Q_n) &= -\frac{1}{\lambda} \E_{Q_n} \left[\what{z}_\beta(X; Y) \what{z}_\beta(X'; Y') 
		h^\prime(\norm{X-X'}_{q, \beta}^q) |X_l - X_l'|^q\right].
\end{split}
\end{equation*}
To facilitate the proof, it's natural to introduce the auxiliary quantity 
\begin{equation*}
	\FKRRX(\beta) = -\frac{1}{\lambda} \E_{Q_n} \left[z_\beta(X; Y) z_\beta(X'; Y') h^\prime(\norm{X-X'}_{q, \beta}^q) |X_l - X_l'|^q\right].
\end{equation*} Below we show 
the existence of constants $C_1, C_2 > 0$ such that for any $\eps > 0$, the following holds with probability at least $1-e^{- n\eps^2}$: 
\begin{align}
\label{eqn:concentration-of-gradient-krr-two-steps-1}
	\sup_{\beta \in \mathcal{Z}} \left|\FKRR_\lambda(\beta; \Q) - \FKRRX(\beta)\right|  &\le  \frac{C_1}{\lambda^{3}} \cdot \left(\frac{1}{\sqrt{n}} + \eps\right) \\
\label{eqn:concentration-of-gradient-krr-two-steps-2}
	\sup_{\beta \in \mathcal{Z}} \left|\FKRR_\lambda(\beta; \Q_n) - \FKRRX(\beta)\right|  &\le  \frac{C_2}{\lambda^{3}} \cdot \left(\frac{1}{\sqrt{n}} + \eps\right).
\end{align}
It's clear that the desired concentration result~\eqref{eqn:concentration-of-gradient-krr} would follow immediately by the triangle inequality. 
It remains to prove the high probability bound~\eqref{eqn:concentration-of-gradient-krr-two-steps-1} and~\eqref{eqn:concentration-of-gradient-krr-two-steps-2}.

The concentration~\eqref{eqn:concentration-of-gradient-krr-two-steps-1} follows from a straightforward application of 
Proposition~\ref{proposition:uniform-convergence-general-U}.
Indeed, $\FKRR_\lambda(\beta; \Q) = \E_Q[U_\beta(Z, Z')]$ and $\FKRRX(\beta) = \E_{Q_n}[U_\beta(Z, Z')]$ where 
the function $U_\beta(z, z') = -\frac{1}{\lambda} z_\beta(x; y) z_\beta(x'; y') h^\prime(\norm{x-x'}_{q, \beta}^q) |x-x'|^q$. 
Note the function $\beta \mapsto U_\beta(z, z')$ is bounded by 
$C_3/\lambda^{2}$ and is Lipschitz on the box constraint $\mathcal{Z}$ with the Lipschitz constant at most 
$C_3/\lambda^{3}$ when $z, z' \in \supp(\Q)$, thanks to Lemma~\ref{lemma:bound-lipschitzness-of-f} and the 
assumption that $\supp(\Q)$ is compact.
The concentration result thus follows. 

The proof of high probability bound in Eq.~\eqref{eqn:concentration-of-gradient-krr-two-steps-2} is straightforward.  
Indeed, (i) $|Y| \le C_5$ almost surely by assumption and (ii) $\sup_x |h^\prime(x)| \le |h^\prime(0)|$. Hence, 
the bound follows from Lemma~\ref{lemma:uniform-close-what-r-true-r} and Cauchy-Schwartz inequality.

\end{itemize}
Lemma~\ref{lemma:uniformly-close-population-finite-samples} now follows from the concentration 
results in Eq.~\eqref{eqn:concentration-of-objectives-krr} and~\eqref{eqn:concentration-of-gradient-krr}.

\subsection{Technical lemma for kernel ridge regression}
\label{sec:technical-lemma-krr}

\begin{lemma}[Concentration of $\what{z}_{\beta; \lambda}$ to $z_{\beta; \lambda}$]
\label{lemma:uniform-close-what-r-true-r}
There exists some constant $C > 0$ such 
that the following holds: for any $\eps > 0$, with probability at least $1-e^{-n\eps^2}$
\begin{equation*}
	\sup_{\beta \in \mathcal{Z}}\left(\E_{\Q_n}[(\what{z}_\beta(X, Y) - z_\beta(X, Y))^2] \right)^{1/2} \le \frac{C}{\lambda} \cdot \left(\frac{1}{\sqrt{n}} + \eps\right), 
\end{equation*}
where the constant $C > 0$ depends only on $M, M_X, M_Y, \mu, p$ and not on $n, \lambda$.
\end{lemma}
\begin{lemma} [Properties of the mapping $\beta \mapsto f_{\beta; \lambda}$]
\label{lemma:bound-lipschitzness-of-f}
The mapping $\beta \mapsto f_{\beta; \lambda}$ satisfies: 
\begin{itemize}
\item Uniform Boundedness: $\norm{f_{\beta; \lambda}}_\infty \le \frac{C}{\lambda^{1/2}}$ for all $\beta \in \mathcal{Z}$.
\item H\"{o}lder Continuity: $\norm{f_{\beta; \lambda} - f_{\beta'; \lambda}}_\infty \le \frac{C}{\lambda^{3/2}} \norm{\beta - \beta'}_1^{1/2q}$ for all $\beta, \beta' \in \mathcal{Z}$, 
\end{itemize}
where the constant $C > 0$ depends only on $M, M_X, M_Y, \mu, p$ and not on $n, \lambda$.
\end{lemma}

\subsection{Proof of Lemma~\ref{lemma:uniform-close-what-r-true-r} and Lemma~\ref{lemma:bound-lipschitzness-of-f}}

\subsubsection{Preliminaries I}
We introduce a new representation of $f_{\beta; \lambda}$ and $\what{f}_{\beta; \lambda}$ using the functional analytic tools in the 
literature~\citep{Baker73,JordanLiRu21}. These tools are particularly useful to derive perturbation bounds, and in particular to 
derive Lemma~\ref{lemma:uniform-close-what-r-true-r} and Lemma~\ref{lemma:bound-lipschitzness-of-f}. Recall $k(x, x') = h(\norm{x-x'}_q^q)$.

\begin{definition}
We define the population and empirical covariance operator $\Sigma_\beta: \H \mapsto \H$, $\what{\Sigma}_\beta: \H \to \H$ 
as the bounded linear operators that satisfy for any $f \in \H$
\begin{equation*}
\begin{split}
\Sigma_\beta f &= \E_\Q[k(\beta^{1/q} \odot X, \cdot )f(\beta^{1/q} \odot X)], \\
\what{\Sigma}_\beta f &= \E_{\Q_n}[k(\beta^{1/q} \odot X, \cdot )f(\beta^{1/q} \odot X)].
\end{split}
\end{equation*}
We define the population and empirical covariance function $h_\beta \in \H$, $\what{h}_\beta \in \H$ by
\begin{equation*}
	h_\beta = \E_\Q[Y k(\beta^{1/q} \odot X, \cdot)]~~\text{and}~~\what{h}_\beta = \E_{\Q_n}[Y k(\beta^{1/q} \odot X, \cdot)].
\end{equation*}
\end{definition}

\begin{lemma}[Representation of the solution $f_{\beta; \lambda}$ and $\what{f}_{\beta; \lambda}$ in terms of the covariance operators and covariance functions]
We have the following representation
\begin{equation*}
	f_{\beta; \lambda} = (\Sigma_\beta + \lambda I)^{-1} h_\beta~~\text{and}~~\what{f}_{\beta; \lambda} = (\what{\Sigma}_\beta + \lambda I)^{-1} \what{h}_\beta. 
\end{equation*}
\end{lemma}

\subsubsection{Preliminaries II}
We introduce a high probability concentration result that underlies the proof of Lemma~\ref{lemma:bound-lipschitzness-of-f}. 
The proof of Lemma~\ref{lemma:high-probability-concentration-of-h-sigma} is given in 
Section~\ref{sec:proof-of-lemma-high-probability-concentration-of-h-sigma}. 
\begin{lemma}[Concentration of the covariance operators and covariance functions]
\label{lemma:high-probability-concentration-of-h-sigma}
There exists some constant $C > 0$ such that the following holds: for any $\eps > 0$, with probability at least 
$1-e^{-n\eps^2}$, the following bound holds: 
\begin{equation*}
	\sup_{\beta \in \mathcal{Z}} \normbig{h_\beta - \what{h}_\beta}_\H \le C  \cdot \left(\frac{1}{\sqrt{n}} + \eps\right)
	~~\text{and}~~
	\sup_{\beta \in \mathcal{Z}} \opnormbig{\Sigma_\beta - \what{\Sigma}_\beta} \le C  \cdot \left(\frac{1}{\sqrt{n}} + \eps\right),
\end{equation*}
where the constant $C > 0$ depends only on $M, M_X, M_Y, \mu, p$ and not on $n, \lambda$.
\end{lemma}

\begin{proof}
The proof is based on the technique that underlies the proof of Lemma E.3 in~\cite{JordanLiRu21}, 
albeit there is slight change in the evaluation of the metric entropy integral. 
For completeness, we provide the proof in Section~\ref{sec:proof-of-lemma-high-probability-concentration-of-h-sigma}. 
\end{proof}

\subsubsection{Proof of Lemma~\ref{lemma:uniform-close-what-r-true-r}}
Our starting point is the following identity: 
\begin{equation}
\begin{split}
&\left(\E_{\Q_n}[(\what{z}_\beta(X, Y) - z_\beta(X, Y))^2] \right)^{1/2} = \normbig{\what{\Sigma}_\beta^{1/2}(f_{\beta; \lambda} - \what{f}_{\beta; \lambda})}_\H.
\end{split}
\end{equation}
Indeed, note that (i) $ \what{z}_\beta(x, y) - z_\beta(x, y) = (f_{\beta; \lambda} - \what{f}_{\beta; \lambda}) (\beta^{1/q} \odot x)$ and 
(ii) $\E_{\Q_n}[g^2(\beta^{1/q}\odot X)]^{1/2} = \normsmall{\what{\Sigma}_\beta^{1/2}g}_\H$ holds for all $g\in \H$.
Below we shall prove the deterministic bound: 
\begin{equation}
\label{eqn:main-deterministic-bound-two}
	\normbig{\what{\Sigma}_\beta^{1/2}(f_{\beta; \lambda} - \what{f}_{\beta; \lambda})}_\H \le 
	\frac{1}{\lambda} M_Y \cdot \left(\opnormbig{\Sigma_\beta - \what{\Sigma}_\beta} + \normbig{h_\beta - \what{h}_\beta}_{\H}\right).
\end{equation}
Now Lemma~\ref{lemma:uniform-close-what-r-true-r} follows immediately from 
Lemma~\ref{lemma:high-probability-concentration-of-h-sigma}.

It remains to prove Eq.~\eqref{eqn:main-deterministic-bound-two}. The proof is based on simple algebra. 
The starting point is the following error decomposition:  
$\what{\Sigma}_\beta^{1/2}(f_{\beta; \lambda} - \what{f}_{\beta; \lambda}) 
		= \err_1 + \err_2$
where 
\begin{equation*}
\err_1 = \what{\Sigma}_\beta^{1/2} \big(({\Sigma}_\beta + \lambda I)^{-1} - (\what{\Sigma}_\beta + \lambda I)^{-1}\big) \what{h}_\beta,~~
\err_2 = \what{\Sigma}_\beta^{1/2}  (\Sigma_\beta + \lambda I)^{-1} (h_\beta - \what{h}_\beta).
\end{equation*}
As a result, $\normbig{\what{\Sigma}_\beta^{1/2}(f_{\beta; \lambda} - \what{f}_{\beta; \lambda})}_\H \le \norm{\err_1}_\H + \norm{\err_2}_\H$. 
Below we bound $\norm{\err_1}_\H$ and $\norm{\err_2}_\H$. 
\paragraph{Bound on $\norm{\err_1}_\H$}
Simple algebraic manipulation yields the following identity: 
	\begin{equation*}
		\err_1 = \left(\what{\Sigma}_\beta^{1/2}(\what{\Sigma}_\beta + \lambda I)^{-1/2}\right) \cdot
			\left(I - (\what{\Sigma}_\beta + \lambda I)^{1/2} ({\Sigma}_\beta + \lambda I)^{-1} (\what{\Sigma}_\beta + \lambda I)^{1/2}\right) 
				\cdot \left((\what{\Sigma}_\beta + \lambda I)^{-1/2} \what{h}_\beta\right).  
	\end{equation*}
	Now we bound the above three terms on the right-hand side. 
	\begin{itemize}
	\item $\Sigma_\beta$ is a positive operator. Hence, $\opnormbig{\what{\Sigma}_\beta^{1/2}(\what{\Sigma}_\beta + \lambda I)^{-1/2}} \le 1$. 
	\item We use the following fundamental fact in functional analysis. For any linear operator $A: \H \to \H$, denoting $A^*$
		to its adjoint operator, then $I - A^*A$ has the same spectrum as $I - AA^*$. Applying this to the operator 
		$A = (\Sigma_\beta + \lambda I)^{1/2} (\what{\Sigma}_\beta + \lambda I)^{-1/2}$, we obtain 
		\begin{equation}
		\begin{split}
			&\opnormbig{I - (\what{\Sigma}_\beta + \lambda I)^{1/2} ({\Sigma}_\beta + \lambda I)^{-1} (\what{\Sigma}_\beta + \lambda I)^{1/2}} \\
			&= \opnormbig{I - ({\Sigma}_\beta + \lambda I)^{-1/2} (\what{\Sigma}_\beta + \lambda I) (\Sigma_\beta + \lambda I)^{-1/2}} \\
			&= \opnormbig{({\Sigma}_\beta + \lambda I)^{-1/2} (\Sigma_\beta - \what{\Sigma}_\beta) (\Sigma_\beta + \lambda I)^{-1/2}}
				\le \frac{1}{\lambda} \opnormbig{\Sigma_\beta - \what{\Sigma}_\beta}. 
		\end{split}
		\end{equation}
	\item We have the bound $\normbig{(\what{\Sigma}_\beta + \lambda I)^{-1/2} \what{h}_\beta}_{\H} \le M_Y$. Indeed, the bound 
		appears in the proof of Proposition 10 in~\cite{JordanLiRu21}.
	\end{itemize}
	As a summary, we have proven that $\norm{\err_1}_\H \le \frac{1}{\lambda} M_Y \cdot \opnormbig{\Sigma_\beta - \what{\Sigma}_\beta}$.
\paragraph{Bound on $\norm{\err_2}_\H$} 
Recall $\err_2 = \what{\Sigma}_\beta^{1/2} \cdot (\Sigma_\beta + \lambda I)^{-1} \cdot (h_\beta - \what{h}_\beta)$. Note the following bounds. 
	\begin{itemize}
	\item $\opnorms{\what{\Sigma}_\beta^{1/2}} \le |h(0)|^{1/2}$. This is true since by Lemma~\ref{lemma:relation-of-different-metrics}, any function $g \in \H$ satisfies
		\begin{equation*}
			\normsmall{\what{\Sigma}_\beta^{1/2}g}_{\H} = \E_{\Q_n}[g^2(X)]^{1/2} \le \norm{g}_\infty \le |h(0)|^{1/2}\norm{g}_{\H}.
		\end{equation*}
	\item $\opnormbig{({\Sigma}_\beta + \lambda I)^{-1}} \le \frac{1}{\lambda}$ since ${\Sigma}_\beta$ is a positive operator. 
	\end{itemize}
As a result, this proves that $\norm{\err_2}_{\H} \le \frac{1}{\lambda} \cdot |h(0)|^{1/2} \cdot \normbig{h_\beta - \what{h}_\beta}_{\H}$. 

\paragraph{Summary}
As a result, we have shown Eq.~\eqref{eqn:main-deterministic-bound-two}. Lemma~\ref{lemma:uniform-close-what-r-true-r} follows
from Lemma~\ref{lemma:high-probability-concentration-of-h-sigma}.

\subsubsection{Proof of Lemma~\ref{lemma:bound-lipschitzness-of-f}}
\begin{itemize}
\item The boundedness part is simply a consequence of the optimality of $f_{\beta; \lambda}$. Indeed, the KRR 
objective takes the value of $\half \E_Q[Y^2]$ when $f = 0$. This gives the basic inequality: 
\begin{equation*}
\half \E_Q[Y^2] \ge   \FKRR_\lambda(\beta)  \ge  \frac{\lambda}{2}\norm{f_{\beta; \lambda}}_{\H}^2.
\end{equation*}
Note that $\norm{f_{\beta; \lambda}}_\infty \le |h(0)|^{1/2} \norm{f_{\beta; \lambda}}_\H$ by Lemma~\ref{lemma:relation-of-different-metrics}. 
Hence, $\norm{f_{\beta; \lambda}}_\infty \le \frac{1}{\lambda} |h(0)|^{1/2}M_Y$. 
\item The H\"{o}lder's continuity follows from the basic algebra. By Lemma~\ref{lemma:relation-of-different-metrics}, 
it suffices to prove for some constant $C > 0$ depending only on $M, M_X, M_Y, \mu, p$ and not on $n, \lambda$.
\begin{equation*}
	\normsmall{f_{\beta; \lambda} - f_{\beta'; \lambda}}_\H \le \frac{C}{\lambda^{3/2}} \norm{\beta -\beta'}_1^{1/q}. 
\end{equation*}
Now we bound $\normsmall{f_{\beta; \lambda} - f_{\beta'; \lambda}}_\H$. We start from the representation due to Lemma~\ref{lemma:bound-lipschitzness-of-f}: 
\begin{equation*}
	f_{\beta; \lambda} = (\Sigma_\beta + \lambda I)^{-1} h_\beta~~\text{and}~~f_{\beta'; \lambda} = (\Sigma_{\beta'} + \lambda I)^{-1} h_{\beta'}
\end{equation*}
Simple algebraic manipulation gives 
\begin{equation*}
	\begin{split}
		f_{\beta; \lambda}  -f_{\beta'; \lambda} 
			&= (\Sigma_{\beta'} + \lambda I)^{-1} (\Sigma_{\beta'} - \Sigma_{\beta})(\Sigma_\beta + \lambda I)^{-1}  h_{\beta}
				+ (\Sigma_{\beta'} + \lambda I)^{-1} (h_\beta - h_{\beta'}).
	\end{split}
\end{equation*}
Note the following bounds: (i) $\normbig{({\Sigma}_\beta + \lambda I)^{-1/2} h_\beta}_{\H} \le M_Y$---this basic bound appears 
in the proof of Proposition 10 in~\cite{JordanLiRu21} (ii) $\opnorms{\Sigma_\beta + \lambda I}^{-1} \le \frac{1}{\lambda}$. 
As a consequence, we immediately obtain the following bound: 
\begin{equation*}
	\norm{f_{\beta; \lambda}  -f_{\beta'; \lambda}}_\H \le \frac{1}{\lambda^{3/2}} \cdot \left(M_Y \cdot \opnorms{\Sigma_\beta - \Sigma_{\beta'}}
		+ \norm{h_\beta - h_\beta'}_\H\right)
\end{equation*}
Now it remains to bound $\opnorms{\Sigma_\beta - \Sigma_{\beta'}}$ and $\norms{h_\beta - h_\beta'}_\H$. We do some simple 
algebra. Let $(X,Y), (X', Y')$ be independent copies drawn from $\Q$. 
Following the proof of Proposition 11 in~\cite{JordanLiRu21}, we have the following identity: 
\begin{equation}
\label{eqn:h-norm}
\norm{h_\beta - h_\beta'}_\H^2 = \E[(k_{\beta, \beta}(X, X')+ k_{\beta', \beta'}(X, X') - 2k_{\beta, \beta'}(X,X'))YY'],
\end{equation}
where $k_{\beta_1, \beta_2}(x,x') = h(\norms{\beta_1^{1/q} \odot x - \beta_2^{1/q} \odot x'}_q^q)$. Additionally, we have the inequality: 
\begin{equation}
\label{eqn:Sigma-norm}
\opnorms{\Sigma_\beta - \Sigma_\beta'}^2 \le \E[(k_{\beta, \beta}^2(X, X')+ k_{\beta', \beta'}^2(X, X') - 2k_{\beta, \beta'}^2(X,X'))YY'].
\end{equation}
With some diligent algebraic manipulations and basic analytic tools (e.g., Taylor's intermediate theorem), 
the above two equations~\eqref{eqn:h-norm} and~\eqref{eqn:Sigma-norm} immediately yield 
\begin{equation*}
	\norm{h_\beta - h_\beta'}_\H \le C \norm{\beta-\beta'}_1^{1/2q}~~\text{and}~~\opnorms{\Sigma_\beta - \Sigma_{\beta'}} \le C \norm{\beta-\beta'}_1^{1/2q}
\end{equation*}
where the constant $C$ depends only on $M, M_X, M_Y, \mu, p$ and not on $n, \lambda$.
\end{itemize}

\subsubsection{Proof of Lemma~\ref{lemma:high-probability-concentration-of-h-sigma}}
\label{sec:proof-of-lemma-high-probability-concentration-of-h-sigma}
The proof here largely follows the proof of Lemma E.3 in~\cite{JordanLiRu21}. Only 
slight modification is needed. 
Below we only prove the concentration for the covariance function (i.e., $\what{h}_\beta \approx h_\beta$
uniformly over $\beta\in \mathcal{Z}$), as the proof for the covariance operator 
(i.e., $\what{\Sigma}_\beta \approx \Sigma_\beta$) is totally analogous. Below we use $C, C_1, C_2, \ldots, $ to 
denote constants that depend only on $M, M_X, M_Y, \mu, p$ and not on $n, \lambda$. 

\paragraph{Step 1: Symmetrization and Reduction} Let $\eps$ denote independent Radamacher random variables. 
Let $\what{h}_\beta(\eps) = \E_{\Q_n}[\eps k(\beta^{1/q}\odot X, \cdot) Y]$.
The standard symmetrization argument gives the following 
bound that holds for 
any convex and increasing mapping $\Phi: \R_+ \mapsto \R_+$: 
\begin{equation*}
	\E\bigg[\Phi \Big(\sup_{\beta \in \mathcal{Z}}\normbig{h_\beta - \what{h}_\beta}_\H\Big)\bigg]
		\le 
	\E\bigg[\Phi \Big(2 \cdot \sup_{\beta \in \mathcal{Z}}\normbig{\what{h}_\beta(\eps)}_\H\Big)\bigg]
\end{equation*}
A classical reduction argument due to Panchenko (Lemma I.1 in~\cite{JordanLiRu21}) implies that
it suffices to derive an exponential tail bound on $\sup_{\beta \in \mathcal{Z}}\normbig{\what{h}_\beta(\eps)}_\H$.

\paragraph{Step 2: Evaluation and Simplification} The reproducing property of $k$ shows that 
\begin{equation*}
	\normbig{\what{h}_\beta(\eps)}_\H^2 = \E_{\Q_n} \big[\eps \eps' h(\norms{X-X'}_{q, \beta}^q) YY'\big],
\end{equation*}
where $(X, Y), (X', Y')$ above are independently sampled from $\Q_n$. 
Write $W_\beta = \normbig{\what{h}_\beta(\eps)}_\H^2$.

\paragraph{Step 3: Centering}
Write $\wbar{W}_\beta = W_\beta - \E[W_\beta]$. 
Let $\wbar{W} = \sup_{\beta \in \mathcal{Z}} \wbar{W}_\beta$ be the supremum of the centered
process $\wbar{W}_\beta$. Since $|\E[W_\beta]| \le C/n$ for all $\beta$ where $C= |h(0)|^{1/2}M_Y$, 
we obtain the bound $\sup_\beta W_\beta \le \wbar{W} + C/n$. Below we control $\wbar{W}$ by 
first control individual $\wbar{W}_\beta$.

\paragraph{Step 4: Control $\wbar{W}_\beta$}
We use Hanson Wright's inequality to establish the sub-exponential property of the difference 
$\wbar{W}_\beta - \wbar{W}_{\beta'}$. Clearly $\eps Y$ is sub-gaussian. Now, let $\{(\eps^{(i)}, X^{(i)}, Y^{(i)})\}_{i=1}^n$
be the i.i.d data. Introduce the matrix $A_\beta \in \R^{n \times n}$ where  its $(i, j)$-th entry is defined by 
\begin{equation*}
(A_\beta)_{i, j} = h(\norms{X^{(i)} - X^{(j)}}_{q, \beta}^q)- \E[h(\norms{X^{(i)} - X^{(j)}}_{q, \beta}^q)].
\end{equation*}
Let $\Delta_{\beta, \beta'}$ be the matrix with $\Delta_{\beta, \beta'} = A_\beta - A_{\beta'}$. Hanson-Wright's inequality 
yields that 
\begin{equation*}
	\P\left(|\wbar{W}_\beta - \wbar{W}_{\beta'}| \ge t\mid X\right) \le 2\exp\left(- C_1 \cdot \min\left\{-n^2t/\opnorm{\Delta_{\beta, \beta'}}, 
		n^4 t^2/\matrixnorm{\Delta_{\beta, \beta'}}_F^2\right\}\right).
\end{equation*}
Note that $\max_{i, j} |\Delta_{\beta, \beta'}| \le C_2\cdot \norm{\beta- \beta'}_\infty$ in virtue of the Lipschitzness of the kernel $k$. 
Hence, we obtain $\matrixnormbig{\Delta_{\beta, \beta'}}_F \le  n\cdot C_2\norm{\beta- \beta'}_\infty$
and $\opnorm{\Delta_{\beta, \beta'}} \le  n\cdot C_2\norm{\beta- \beta'}_\infty$. Consequently, we obtain that 
the increment $|\wbar{W}_\beta - \wbar{W}_{\beta'}|$ is sub-exponential with 
\begin{equation*}
	\P\left(|\wbar{W}_\beta - \wbar{W}_{\beta'}| \ge t\right) \le 2\exp\left(- C_3 \cdot \min\left\{-n t/ \norm{\beta-\beta'}_\infty, 
		n^2 t^2/\norm{\beta-\beta'}_\infty^2 \right\}\right).
\end{equation*}

\newcommand{\diam}{{\rm diam}}
\newcommand{\Quantile}{Q}

\paragraph{Step 5: Control $\wbar{W}$ via Chaining}
The previous step shows that the process $\beta\mapsto \wbar{W}_\beta$ has increments that are sub-exponential with 
parameter on the scale of $ \norm{\beta-\beta'}_\infty/n$.  The chaining argument for the sub-exponential process 
(see Theorem 8 in~\cite{JordanLiRu21}) implies that the following bound holds with probability at least $1-e^{-nt}$: 
\begin{equation}
\label{eqn:chaining-result}
\begin{split}
	\wbar{W} \le C_4 \cdot \left(\frac{1}{n}\cdot \int_0^{\infty} \log N(\mathcal{Z}, \norm{\cdot}_\infty, \eps) d\eps + t\right).
\end{split}
\end{equation}
where $N(\mathcal{Z}, \norm{\cdot}_\infty, \eps)$ is the covering number of the set $\mathcal{Z}$ using the $\norm{\cdot}_\infty$-ball of radius $\eps$. 
As $\log N(\mathcal{Z}, \norm{\cdot}, \eps) = 0$ for $\eps > M$ and $\log N(\mathcal{Z}, \norm{\cdot}, \eps) \le p \log (3M/\eps)$
for $\eps \le M$, we obtain
\begin{equation}
\label{eqn:metric-entropy-bound}
	\int_0^{\infty} \log N(\mathcal{Z}, \norm{\cdot}_\infty, \eps) d\eps
	\le 3Mp \cdot \int_0^{1/3} \log (1/\eps) d\eps \le C_5.
\end{equation}
Consequently, we obtain that $\wbar{W} \le C_5 \cdot (1/n+t)$ holds with probability at least $1-e^{-nt}$.

\paragraph{Step 6: Finalizing Argument}
Combine the results in Step 3 and Step 5. The following bound holds with probability at least $1- e^{-nt}$ for any $t > 0$: 
\begin{equation*}
	\sup_{\beta \in \mathcal{B}_M} \normbig{\hat{h}_\beta(\eps)}_{\H} =  \sup_\beta W_\beta^{1/2} \le C_7 \cdot \left(\frac{1}{\sqrt{n}}+ \sqrt{t}\right).
\end{equation*}
As discussed in Step 1, Panchenko's argument translates it to the desired high probability bound, i.e., 
$\sup_{\beta \in \mathcal{B}_M}\normbig{\what{h}_\beta - h_\beta}_{\H} \le C_7 \cdot \left(\frac{1}{\sqrt{n}}+ \eps\right)$ with probability 
at least $1-e^{-n\eps^2}$.

\section{Statistical Consequence}

\subsection{Metric learning: Proof of Proposition~\ref{proposition:emp_algo}}
The main building block of the proof is Lemma~\ref{lemma:metric-learning-each-subset}, whose proof 
is given in Section~\ref{sec:proof-lemma-metric-learning-each-subset}. Although independently stated, Lemma~\ref{lemma:metric-learning-each-subset}
is in fact a consequence of Theorem~\ref{thm:FML-sparse-one-round}. 

\begin{lemma}
\label{lemma:metric-learning-each-subset}
Assume Assumptions~\ref{assumption:mu-assumtion-kernels} and~\ref{assumption:distribution-P}. 
Fix $\beta(0) \in \mathcal{Z}$ of full support: $\supp(\beta(0)) = [p]$. There exist constants
$c, C, \eps_0 > 0$ such that the following holds. For any subset $A \subseteq [p]$, 
consider the following empirical minimization problem
\begin{equation*}
	\minimize_{\beta} \FML(\beta; \Q^{A}_n)~~~\text{subject to}~~\beta \ge 0, ~\norm{\beta}_\infty \le M.
\end{equation*}
Let $\what{\beta}$ be the stationary point found by the projected gradient flow initialized at $\beta(0)$. Then 
the following property holds with probability at least  $1-e^{-cn\eps^2}$ if $\eps \ge C/\sqrt{n}$: 
\begin{itemize}
\item If $\E[Y|X] = \E[Y|X_A]$, then $\FML(\what{\beta}; \Q^A_n) > - \eps$ 
\item Otherwise, then $\supp(\what{\beta}) \subseteq S$, $\supp(\what{\beta}) \backslash A \neq \emptyset$, 
	$\FML(\what{\beta}; \Q^A_n) < -\eps$ if $\eps_0 \ge \eps$,
\end{itemize}
where the constants $c, C, \eps_0$ are independent of the sample size $n$. 
\end{lemma}

Now we are ready to prove Proposition~\ref{proposition:emp_algo}. We take the same constants $c, \eps_0$
as described in Lemma~\ref{lemma:metric-learning-each-subset}. We use $\Omega_n(A)$ to denote the event 
on which the property stated in Lemma~\ref{lemma:metric-learning-each-subset} holds. We take $\Omega_n
\equiv \cap_{A: A\subseteq S} \Omega_n(A)$ which we simply call it the ``good event''. On the event $\Omega_n$,
with a simple induction argument, one can show that Algorithm~\ref{alg:metric-learning} must return a set $\what{S}$ 
which simultaneously satisfies $\what{S} \subseteq S$ and $\E[Y|X] = \E[Y|X_{\what{S}}]$. Clearly, this proves Proposition~\ref{proposition:emp_algo}
as the good event $\Omega_n$ happens with probability at least $1-2^{|S|}e^{-cn}$ by the union bound. 

\subsubsection{Proof of Lemma~\ref{lemma:metric-learning-each-subset}}
\label{sec:proof-lemma-metric-learning-each-subset}
The proof is similar to that of Theorem~\ref{thm:FML-sparse-one-round}. 
Consider the following ``good'' event $\Omega_n^A(\eps)$: 
\begin{equation*}
\begin{array}{c}
	|\FML(\beta; \Q_n^A) - \FML(\beta; \Q^A)| \le \eps,~~~\\
	\norm{\grad \FML(\beta; \Q_n^A) - \grad \FML(\beta; \Q^A)}_\infty \le \eps
\end{array}~~\text{for all $\beta \in \mathcal{Z}$}.
\end{equation*}
Notice that the event $\Omega_n^A(\eps)$ is defined in a similar way to that of $\Omega_n(\eps)$ 
(cf. Eq.~\eqref{eqn:omega-n-metric-learning} in the main text). A similar proof to 
Lemma~\ref{lemma:uniformly-close-population-finite-samples} implies that the event $\Omega_n^A(\eps)$ happens with 
high probability at least $1-e^{-cn\eps^2}$ for any $\eps \ge C/\sqrt{n}$.

We are now ready to prove Lemma~\ref{lemma:metric-learning-each-subset}. 
\begin{itemize}
\item Assume $\E[Y|X] = \E[Y|X_A]$. According to Proposition 3 and Proposition 4 in \cite{LiuRu20}, $X \perp Y$ under 
	the reweighting distribution $\Q^A$, and $\FML(\beta; \Q^A) \equiv 0$. As a consequence, we have $\FML(\beta; \Q_n^A) > -\eps$
	for all $\beta \in \mathcal{Z}$ on the event $\Omega_n^A(\eps)$. 
\item Assume otherwise. Then $X^A \not \perp Y$ under $\Q^A$. Furthermore, if we denote 
	$S^A$ to be the signal set under $\Q^A$ according to Definition~\ref{definition:signal-set}, then $S^A \subseteq S$ and 
	$S^A \backslash A \neq \emptyset$ according to Proposition 3 in \cite{LiuRu20}. As a result, we have $|\FML(\beta(0); \Q^A)| > 0$ 
	since $\FML$ is an independence measure and $\beta(0)$ is of full support. Let $\eps_0 = |\FML(\beta(0); \Q^A)|/4$.  
	An almost identical proof of Theorem~\ref{thm:FML-sparse-one-round} gives that $\supp(\what{\beta}) \subseteq S$, 
	$\supp(\what{\beta}) \backslash A \neq \emptyset$ and $\FML(\what{\beta}; \Q^A_n) < -\eps$ on $\Omega_n^A(\eps)$ for all $\eps \le \eps_0/2$.
\end{itemize}

\subsection{Kernel ridge regression: Proof of Proposition~\ref{proposition:emp_algo_krr}}
The main building block of the proof is Lemma~\ref{lemma:kernel-ridge-regression-each-subset}, whose proof 
is given in Section~\ref{sec:proof-lemma-kernel-ridge-regression-each-subset}. Although independently stated, 
Lemma~\ref{lemma:kernel-ridge-regression-each-subset} is in fact a consequence of Theorem~\ref{thm:FKRR-sparse-one-round}. 

\begin{lemma}
\label{lemma:kernel-ridge-regression-each-subset}
Assume Assumptions~\ref{assumption:mu-assumtion-kernels}-\ref{assumption:insufficient-set-T}. Let $q=1$.
There exist constants $c, C, \eps_0,\lambda_0 > 0$ such that the following holds. For any subset 
$A \subseteq [p]$, consider the empirical minimization
\begin{equation*}
	\minimize_{\beta} \FKRR(\beta; \Q_n)~~~\text{subject to}~~\beta \ge 0,~\beta_A = M \mathbf{1}_A, ~\norm{\beta}_\infty \le M.
\end{equation*}
Let $\what{\beta}$ be the stationary point found by the projected gradient flow initialized at $\beta(0) = \beta^{(0; A)}$ 
(which is defined in the statement of Proposition~\ref{proposition:emp_algo_krr}). For any $\lambda \le \lambda_0$, 
$\eps \ge C/(\sqrt{n}\lambda^3)$, there exists an 
event $\Omega_n \equiv \Omega_n(\eps)$ whose definition is independent of the subset $A$ and 
which happens with probability at least  $1-e^{-cn(\eps^2 \wedge \lambda^2) \lambda^6}$ such that the following happens on $\Omega_n$:
\begin{itemize}
\item If $\E[Y|X] = \E[Y|X_A]$, then $\FKRR_\lambda(\beta^{*; A}; \Q_n) - \FKRR_\lambda(\what{\beta}; \Q_n) < \eps$. 
\item Otherwise, then $\supp(\what{\beta}) \subseteq S$, $\supp(\what{\beta}) \backslash A \neq \emptyset$, 
	$\FKRR_\lambda(\beta^{*, A}; \Q_n^A) - \FKRR_\lambda(\what{\beta}; \Q^A_n) > \eps$ if $\eps_0 \ge \eps$,
\end{itemize}
where the constants $c, C, \eps_0$ are independent of the sample size $n$. 
\end{lemma}

Now we are ready to prove Proposition~\ref{proposition:emp_algo_krr}. We take the same constants $c, \eps_0$
as described in Lemma~\ref{lemma:kernel-ridge-regression-each-subset}. Fix the parameter $\eps_n > 0$ from the statement of 
Proposition~\ref{proposition:emp_algo_krr}. We use $\Omega_n$ to denote the event 
described in Lemma~\ref{lemma:kernel-ridge-regression-each-subset}. On the event $\Omega_n$,
with a simple induction argument, one can easily show that Algorithm~\ref{alg:kernel-feature-selection-empirical} must return a set $\what{S}$ 
which simultaneously satisfies $\what{S} \subseteq S$ and $\E[Y|X] = \E[Y|X_{\what{S}}]$. This proves Proposition~\ref{proposition:emp_algo_krr}
as the good event $\Omega_n$ happens with probability at least $1-e^{-cn(\eps_n^2 \wedge \lambda_n^2) \lambda_n^6}$. 

\subsubsection{Proof of Lemma~\ref{lemma:kernel-ridge-regression-each-subset}}
\label{sec:proof-lemma-kernel-ridge-regression-each-subset}
The proof is similar to that of Theorem~\ref{thm:FKRR-sparse-one-round}. 
Recall the event $\Omega_n(\eps)$ (cf. Eq.~\eqref{eqn:omega-n-KRR}): 
\begin{equation*}
\begin{array}{c}
 	|\FKRR_\lambda(\beta; \Q_n) - \FKRR_\lambda(\beta; \Q)| \le \eps,~~~\\
 	\norm{\grad \FKRR_\lambda(\beta; \Q_n) - \grad \FKRR_\lambda(\beta; \Q)}_\infty \le \eps
\end{array}~~\text{for all $\beta \in \mathcal{Z}$}.
\end{equation*}
Lemma~\ref{lemma:uniformly-close-population-finite-samples-KRR} implies that the event $\Omega_n(\eps)$ happens with 
high probability at least $1-e^{-cn\eps^2\lambda^6}$ for any $\eps \ge C/(\sqrt{n}\lambda^3)$. 

Now is an opportune time to prove Lemma~\ref{lemma:kernel-ridge-regression-each-subset}. Below we fix $\eps > 0$. 
\begin{itemize}
\item  Assume $\E[Y|X] = \E[Y|X_A]$. According to Proposition~\ref{proposition:kernel-ridge-consistent-reason},
	there exists $\lambda_0 > 0$ such that for $\lambda \le \lambda_0$: $\FKRR_\lambda(\beta^{*; A}; \Q) < G + \eps/2$
	where $G = \half \Var_Q(\E_Q[Y|X])$. It's easy to show that $\FKRR_\lambda(\beta; \Q) > G$ for all $\beta \in \mathcal{Z}$ and $\lambda \ge 0$.
	Hence, this shows that $\FKRR_\lambda(\beta^{*; A}; \Q) - \FKRR_\lambda(\beta; \Q) < \eps/2$ for all $\beta \in \mathcal{Z}$. 
	By triangle inequality, this implies that $\FKRR_\lambda(\beta^{*; A}; \Q) - \FKRR_\lambda(\beta; \Q) < \eps$
	for all $\beta \in \mathcal{Z}$ on the event $\Omega_n(\eps/4)$. 
\item  Assume otherwise. Then there exists a feature $X_l$ such that 
	$\E[Y|X_A] \neq \E[Y|X_{A \cup \{l\}}]$.  Proposition~\ref{lemma:gradient-bound-conditional-main-effect} shows a lower bound 
	on the gradient: $\partial_{\beta_l} F(\beta; \Q) \ge c/\lambda$ at $\beta = \beta^{*, A} = \beta^{0, A}$ where $c > 0$ is a constant 
	independent of $n$ and $\lambda$. The rest of the proof then follows exactly the same logic as appeared in the proof of 
	Theorem~\ref{thm:FKRR-sparse-one-round}. That means that, for some $\eps_0 > 0$, we must have on the event $\Omega_n(\eps_0 \lambda)$, 
	$\supp(\what{\beta}) \subseteq S$, $\supp(\what{\beta}) \backslash A \neq \emptyset$,  
	$\FKRR_\lambda(\beta^{*, A}; \Q_n) - \FKRR_\lambda(\what{\beta}; \Q_n) > \eps_0$. 
\end{itemize}

\section{Other Results}
\label{sec:extension-to-weak-dependence}

\subsection{Extension: Relaxation of the independence assumption $X_S \perp X_{S^c}$}

\begin{definition}
 The maximal correlation between two groups of random variables $W_1, W_2$ is defined by  
 \begin{equation*}
 	\vartheta(W_1, W_2) = \sup_{g_1, g_2} \frac{\Cov(g_1(W_1), g_2(W_2))}{\sqrt{\Var(g_1(W_1)) \Var(g_2(W_2))}}
 \end{equation*}
 where the supremum is taken over all real-valued measurable functions such that 
 $\Var(g_1(W_1)) < \infty$ and $\Var(g_2(W_2)) < \infty$.
\end{definition}

\setcounter{theorem}{0}
\renewcommand\thetheorem{\arabic{theorem}A'}

Theorem~\ref{thm:FML-sparse-one-round-extension} is a viable extension of Theorem~\ref{thm:FML-sparse-one-round}. Basically, 
Theorem~\ref{thm:FML-sparse-one-round-extension} allows weak dependence between  the signal variables $X_S$ and the noise 
variables $X_{S^c}$ (while Theorem~\ref{thm:FML-sparse-one-round} requires exact independence between $X_S$ and $X_{S^c}$). 

\begin{theorem}
\label{thm:FML-sparse-one-round-extension}
Assume Assumptions~\ref{assumption:mu-assumtion-kernels} and~\ref{assumption:distribution-P}. 
Assume the set of signals is not empty: $S \neq \emptyset$. 
Consider the trajectory $t \mapsto \wtilde{\beta}(t)$ of the gradient flow with respect to the empirical metric learning 
objective $\FML(\cdot; \Q_n)$. Choose the initialization to be of full support: $\supp(\wtilde{\beta}(0)) = [p]$.  
Then there exists $\vartheta_0 > 0$ such that the following holds. As long as $\vartheta(X_S, X_{S^c}) \le \vartheta_0$, 
the following happens with probability at least $1-e^{-cn}$: 
\begin{equation*}
	\emptyset \neq \supp(\wtilde{\beta}(t)) \subseteq S~~\text{holds for all $t \ge \tau$},
\end{equation*}
where the constants $c, \tau > 0$ are independent of the sample size $n$. 
\end{theorem}

\setcounter{theorem}{1}
\renewcommand\thetheorem{\arabic{theorem}'}

\begin{proof}
The proof of Theorem~\ref{thm:FML-sparse-one-round-extension} follows the exact same route as that
of Theorem~\ref{thm:FML-sparse-one-round}. Indeed,  it suffices to show that the population objective $\FML(\cdot; \Q)$ is 
self-penalizing when there is the weak dependence $\vartheta(X_S, X_{S^c}) \le \vartheta_0$
for some $\vartheta_0 > 0$. This is implied by the following Theorem~\ref{theorem:metric-learning-extension}, which is a
generalization of the gradient bound of Theorem~\ref{theorem:metric-learning} to the weak dependence setting. 
\end{proof}

\begin{theorem}
\label{theorem:metric-learning-extension}
Assume Assumptions~\ref{assumption:mu-assumtion-kernels} and~\ref{assumption:distribution-P}. 
The following holds for all $\beta\ge 0$ and $j \not\in S$: 
\begin{equation}
\label{eqn:self-penalization-gradient-bound-extension}
\begin{split}
\partial_{\beta_j}  \FML(\beta; \Q) & \ge \underline{c}(\beta) \cdot |\FML(\beta; \Q)| - C \vartheta(X_S, X_{S^c}).
\end{split}
\end{equation}
where $\underline{c}(\beta) > 0$ is defined in Theorem~\ref{theorem:metric-learning-extension} and 
$C > 0$ depends only on $M_\mu, M_X$.  
\end{theorem}

\begin{proof}
The big picture of the proof of Theorem~\ref{theorem:metric-learning-extension} is identical to that of 
Theorem~\ref{theorem:metric-learning} (Section~\ref{sec:proof-theorem-metric-learning}). We only need 
to do some minor tweaks. There are only two lines of derivations in the proof of Theorem~\ref{theorem:metric-learning}
where we've used the independence assumption $X_S \perp X_{S^c}$, namely, equations~\eqref{eqn:integrand-LHS} and 
\eqref{eqn:lower-bound-for-the-first-term}. There we use the independence assumption $X_S \perp X_{S^c}$ 
to attain the identity $\E[g_1(X_S)g_2(X_{S^c})] = \E[g_1(X_S)] \E[g_2(X_{S^c})]$ that holds for any real-valued functions $g_1, g_2$. 
More generally, without the independence assumption $X_S \perp X_{S^c}$, we instead attain a bound than an identity that 
holds for all functions $g_1, g_2$: 
\begin{equation*}
	\E[g_1(X_S)g_2(X_{S^c})] \in \left[\E[g_1(X_S)] \E[g_2(X_{S^c})] \pm \vartheta(X_S, X_{S^c}) \sqrt{\Var(g_1(X_S)) \Var(g_2(X_{S^c})}\right],
\end{equation*}
where the notation $[a \pm b]$ means the interval $[a- b, a + b]$ for any $a \in \R, b \in \R_+$. Hence, 
the analogue of Eq.~\eqref{eqn:integrand-LHS} in the more general 
setting is simply the following:
\begin{equation}\tag{\ref{eqn:integrand-LHS}'}
\begin{split}
& \E_Q\left[YY' t e^{-t \norm{X-X'}_{q, \beta}^q} |X_j - X_j'|^q\right] \\
	&\ge \E_Q \left[YY' e^{-t \normsmall{X_S-X_S'}_{q, \beta_S}^q}\right] \cdot 
			\E_Q\left[t e^{-t \normsmall{X_{S^c} - X_{S^c}'}_{q, \beta_{S^c}}^q} |X_j - X_j'|^q\right] 
				- \vartheta_0(X_S, X_{S^c}) \cdot M_\mu (2M_X)^q. 
\end{split}
\end{equation}
and the analogue of Eq.~\eqref{eqn:lower-bound-for-the-first-term} in the more general 
setting is similarly
\begin{equation}\tag{\ref{eqn:lower-bound-for-the-first-term}'}
\begin{split}
	 \E_Q \left[YY' e^{-t \normsmall{X_S-X_S'}_{q, \beta_S}^q}\right] 
		&\ge  \E_Q \left[YY' e^{-t \normsmall{X-X'}_{q, \beta}^q}\right] - \vartheta_0(X_S, X_{S^c}). 
\end{split}
\end{equation}
As a consequence, we can now follow the proof of Theorem~\ref{theorem:metric-learning} and obtain 
the following general bound which holds without the need of independence assumption $X_S \perp X_{S^c}$: 
\begin{equation*}
	\partial_{\beta_j}  \FML(\beta; \Q) \ge \underline{c}(\beta) \cdot |\FML(\beta; \Q)| - C \cdot \vartheta(X_S, X_{S^c})~\text{where}~C = M_\mu^2 (2M_X)^q.
\end{equation*}
This completes the proof of Theorem~\ref{theorem:metric-learning-extension}. 
\end{proof}

Theorem~\ref{thm:FML-sparse-one-round-extension} allows a further direct generalization of Proposition~\ref{proposition:emp_algo} to the 
setting where the signal variables $X_S$ and noise variables $X_{S^C}$ are weakly dependent. We do not pursue further discussion on 
it for the space considerations.

\subsection{Basic derivations for the metric learning objective}
\label{section:basic-derivation}
This section gives the derivation of the metric learning objective as appeared in Eq.~\eqref{eqn:obj-metric-learning-emp}.
Let $X \in \R^p$ and $Y \in \{\pm 1\}$, and the loss $L(y, \hat{y}) = - y \what{y}$. Write $k(x, x') = h(\norm{x-x'}_q^q)$. 
Using the reproducing property of the RKHS, we have  
$f(X)= \langle f, k(X,\cdot) \rangle_{\H}$ for all $f \in \H$. Hence we have the representation
\begin{equation*}
\begin{split}
	F_n(\beta) &= \min_f J_n(\beta, f) \\
	~~\text{where}~J_n(\beta, f) &= -\langle \what{\E}[k(\beta^{1/q}\odot X, \cdot)Y], f\rangle_{\H} + \lambda_n \norm{f}_{\H}^2
\end{split}
\end{equation*}
Note that $f \mapsto J_n(\beta, f)$ is quadratic in $f$ satisfying $J_n(\beta, f) = 
\lambda_n (\normsmall{f-\bar{f}}_{\H}^2- \normsmall{\bar{f}}_{\H}^2)$
where $\bar{f} = \frac{1}{2\lambda_n}  \what{\E}[k(\beta^{1/q}\odot X, \cdot)Y]$ is the unique minimizer of $f \mapsto J_n(\beta, f)$.
As a result, we obtain that $F_n(\beta) = J_n(\beta, \bar{f}) = -\lambda_n\normsmall{\bar{f}}_{\H}^2$.
Using the reproducing property again, we obtain that 
    \begin{equation*}
    \begin{split}
       F_n(\beta) = -\lambda_n\normsmall{\bar{f}}_{\H}^2
            & = -\frac{1}{4\lambda_n} 
            \langle \E[k(\beta^{1/q}\odot X, \cdot)Y], \E[k(\beta^{1/q}\odot X', \cdot)Y']  \rangle_{\H} \\
            &=-\frac{1}{4\lambda_n} \E[YY'k(\beta^{1/q}\odot X, \beta^{1/q}\odot X')].
    \end{split}
    \end{equation*}  
Note that the last line 
uses the bilinearity of the inner product $\langle \cdot, \cdot \rangle_{\H}$ and the independence between $(X,Y)$ and $(X', Y')$.
This completes the derivation of Eq.~\eqref{eqn:obj-metric-learning-emp}.

\end{document}